\documentclass{article}
\usepackage{graphicx}
\usepackage{booktabs}
\usepackage{makecell}
\usepackage{adjustbox}
\usepackage{array}
\usepackage{multirow}
\usepackage{amsthm}
\usepackage{tabularx}
\usepackage{threeparttable}
\usepackage{subcaption}
\usepackage{arydshln}

\usepackage[preprint]{neurips_2026}

\usepackage[utf8]{inputenc} 
\usepackage[T1]{fontenc}    
\usepackage{hyperref}       
\usepackage{url}            
\usepackage{booktabs}       
\usepackage{amsfonts}       
\usepackage{amsmath}       
\usepackage{nicefrac}       
\usepackage{microtype}      
\usepackage{xcolor}         
\newtheorem{theorem}{Theorem}[section]
\newtheorem{lemma}[theorem]{Lemma}
\newtheorem{proposition}[theorem]{Proposition}

\usepackage{stmaryrd}
\usepackage[ruled, vlined]{algorithm2e}
\usepackage{wrapfig}
\usepackage{sidecap}
\usepackage{array}
\sidecaptionvpos{figure}{c}

\title{Let It Go or Learn to Self-Correct:\\Continuous Diffusion for Constrained Discrete Tasks}

\author{%
  Mariia Drozdova \\
  University of Geneva \\
  \texttt{mariia.drozdova@unige.ch} \\
  \And
  Stéphane Liem Nguyen \\
  University of Geneva \\
  \texttt{stephane.nguyen@unige.ch} \\
  \And
  François Fleuret \\
  University of Geneva \\
  \texttt{francois.fleuret@unige.ch} \\
}

\begin{document}

\maketitle

\begin{abstract}
Denoising Diffusion Probabilistic Models (DDPMs) generate samples by starting from noise and repeatedly denoising while keeping each update close to the current noisy state. This behavior is effective in many continuous domains, but its role is less clear for globally constrained discrete tasks, such as Sudoku,
graph connectivity, Latin squares, and N-queens. In such settings, early discrete errors can be difficult to undo. As a result, standard diffusion sampling may preserve early mistakes, even when the model’s clean predictions are informative. We compare standard samplers to sampling directly from the model’s clean prediction. Without retraining, this single change improves Sudoku validity from $31\%$ to $95\%$, with consistent gains across the other discrete tasks. We hypothesize that staying close to the current noisy state is harmful because the reverse trajectory can drift off the forward noising distribution the model was trained on. To reduce this train-test mismatch, we further introduce self-correction training, which exposes the model to its own predictions, improving robustness to errors that arise during inference. This substantially improves the performance of standard samplers. Our results suggest that continuous diffusion models can learn nontrivial global constraints, but discrete reasoning tasks require better alignment between training and inference: either through samplers that reduce commitment to early decisions, or through training that teaches the model to correct its own inference-time errors.
\end{abstract}

\newcommand{\alphaf}[1]{\alpha(#1)}
\newcommand{\betaf}[1]{\beta(#1)}
\newcommand{\diffusioncoeff}[1]{\delta(#1)}
\newcommand{\ttodataf}[1]{{#1-1}}
\newcommand{\ttonoisef}[1]{{#1+1}}
\newcommand{\sigmaf}[1]{\sigma(#1)}
\newcommand{\epsilonf}[1]{\epsilon(#1)}
\newcommand{\xf}[1]{x_{#1}}
\newcommand{\noise}{\epsilon}
\newcommand{\noisef}[1]{\epsilon_{#1}}

\newcommand{\timenoise}{T}
\newcommand{\timelast}{1}
\newcommand{\timedata}{0}

\newcommand{\dalphaf}[1]{\dot{\alpha}(#1)}
\newcommand{\dbetaf}[1]{\dot{\beta}(#1)}
\newcommand{\dsigmaf}[1]{\dot{\sigma}(#1)}

\newcommand{\xdata}{\xf{\timedata}}

\newcommand{\alphat}{\alphaf{t}}
\newcommand{\betat}{\betaf{t}}
\newcommand{\xt}{\xf{t}}
\newcommand{\epsilont}{\epsilonf{t}}
\newcommand{\dalphat}{\dalphaf{t}}
\newcommand{\dbetat}{\dbetaf{t}}
\newcommand{\xhatf}[1]{\hat{x}_{#1}}
\newcommand{\uhatf}[1]{\hat{u}_{#1}}
\newcommand{\shatf}[1]{\hat{s}_{#1}}
\newcommand{\xtildef}[1]{\tilde{x}_{#1}}
\newcommand{\epshatf}[1]{\hat{\epsilon}_{#1}}
\newcommand{\xhatdata}{\xhatf{\timedata}}
\newcommand{\xhatdataf}[2]{\xhatf{\timedata}(#1,#2)}

\newcommand{\tnow}{{t_i}}

\newcommand{\tnext}{s}
\newcommand{\xnow}{\xf{\tnow}}
\newcommand{\xnext}{\xf{\tnext}}
\newcommand{\istart}{{N-1}}
\newcommand{\isecond}{{N-2}}

\newcommand{\ilastiii}{3}
\newcommand{\ilastii}{2}
\newcommand{\ilast}{1}
\newcommand{\iend}{0}

\newcommand{\xstart}{\xf{t_{\istart}}}
\newcommand{\xlast}{\xf{t_{\ilast}}}
\newcommand{\xend}{\xf{t_{\iend}}}

\newcommand{\inow}{{n}}

\newcommand{\ibefore}{{\ttodataf{\inow}}}
\newcommand{\inext}{{\ttonoisef{\inow}}}
\newcommand{\xinow}{\xf{\inow}}
\newcommand{\xinext}{\xf{\inext}}
\newcommand{\xinextdata}{\xf{\ibefore}}
\newcommand{\xistart}{\xf{{\istart}}}           %
\newcommand{\xilastiii}{\xf{{\ilastiii}}}           %
\newcommand{\xilastii}{\xf{{\ilastii}}}           %
\newcommand{\xilast}{\xf{{\ilast}}}           %
\newcommand{\xiend}{\xf{{\iend}}}           %

\newcommand{\abarf}[1]{\bar{\alpha}({#1})}

\newcommand{\postvarf}[1]{\tilde{\beta}^{\mathrm{post}}_{#1}}

\newcommand{\ddpmstdf}[1]{\tau_{#1}}

\section{Introduction}

Diffusion models have become a dominant paradigm for generative modeling in continuous domains, achieving state-of-the-art performance on high-dimensional perceptual data such as images \cite{ho2020denoising, karras2022elucidating, song2020denoising, song2021scorebased, rombach2022high} and videos \cite{alonso2024diffusion,ho2022video, xing2024survey}.  Given this success, it is worth asking whether they can also handle globally constrained discrete tasks. Prior work has incorporated constraints into diffusion models through explicit mechanisms: projection \cite{christopher2024constrained, utkarsh2025physics}, guidance \cite{inoue2023layoutdm}, or search-based post-processing \cite{sun2023difusco}. Discrete diffusion methods built on D3PM \cite{austin2021structured} have been applied to combinatorial optimization \cite{sun2023difusco}, puzzle solving and satisfiability \cite{ye2024beyond}, and controlled text and molecule generation \cite{cardei2025constrained}.
On combinatorial problems, DIFUSCO \cite{sun2023difusco} evaluated both D3PM-based discrete diffusion and continuous diffusion \cite{chen2022analog}, reporting the latter as less effective. In this work, we investigate why standard continuous diffusion sampling can fail on these tasks and whether changes to sampling and training can mitigate these failures.

We study constrained discrete tasks where constraints are implicit in the data rather than provided explicitly to the model: Sudoku \cite{jolicoeur2025less,wang2025hierarchical,wewer2025spatial}, 
graph connectivity \cite{du2024learning}, Latin squares, and N-queens. These benchmarks are closely related to recent iterative-reasoning approaches: HRM/TRM use recursive refinement \cite{wang2025hierarchical,jolicoeur2025less}, SRM studies denoising-based spatial reasoning and sampling-order choices \cite{wewer2025spatial}, and IRED casts graph connectivity as iterative energy diffusion \cite{du2024learning}. Our goal is complementary: we use these tasks to diagnose standard continuous diffusion itself, asking when the denoiser learns useful global structure, why inference can fail to produce valid samples, and how it can be improved.

In DDPMs \cite{ho2020denoising}, each reverse step makes a small update conditioned on the current state, which is well suited to perceptual data but its effect on globally constrained discrete tasks is less clear.  In constrained discrete tasks, we observe a failure mode of standard DDPM
sampling: generated samples can look locally plausible (e.g. each cell resembles a valid symbol: one-hot encoding or valid MNIST image) but remain globally invalid (e.g. Sudoku constraints are not satisfied). This raises a question: does the denoiser fail to learn the constraints, or does the reverse update fail to use the denoiser's clean prediction effectively?

We first test this by modifying only the reverse update. Both DDPM and the modified sampler use the current state \(x_t\) as input to the denoiser, producing a clean proposal \(\hat x_0(x_t,t)\). The difference is what happens after this proposal is formed: DDPM carries a direct \(x_t\)-dependent residual into the reverse mean, keeping the next state tied to the current analog state. We remove this direct residual, yielding a limiting case of the generalized DDIM/DDPM sampling family~\cite{song2020denoising} that we call Tweedie reprojection. This inference-only change gives large gains on the same checkpoints, improving one-hot Sudoku validity from \(31\%\) to \(95\%\). Thus, the denoiser has learned useful constraint structure, but not always robustly enough to guide the DDPM trajectory once that trajectory begins to drift.

We interpret this as a form of training-inference mismatch
\cite{ning2023input, ningelucidating, deng2023markuptoimage, ren2024multi,
zhang2025anti, li2024error}. During training, the model sees forward-noised
valid samples. During sampling, however, the reverse process is driven by the
model's own imperfect predictions. Small denoising errors can move the
trajectory toward states that are locally plausible but globally invalid; the DDPM residual can then keep subsequent updates close to these states, causing errors to compound. This motivates us to also modify the training process with self-correction loss, where the model is exposed to noised versions
of its own intermediate predictions and trained to recover the original valid
target. This training modification improves standard samplers: for example,
DDPM's validity on one-hot Sudoku rises from \(31\%\) to \(87\%\). 

Our study focuses on controlled constrained-discrete benchmarks. Tweedie reprojection should be viewed as a diagnostic for decoded constraint satisfaction, not as a general sampler: in perceptually rich domains, the \(x_t\) residual may carry important detail and discarding it may harm generation quality. Self-correction only partially addresses the training-inference mismatch; more complete approaches remain future work.

In summary, our contributions are:
\begin{itemize}
\item We identify a failure mode of continuous diffusion on constrained discrete reasoning tasks, where inference can lead to locally plausible but globally inconsistent states.
\item We study the role of the direct \(x_t\)-dependent residual in DDPM sampling, and show that removing this residual can substantially improve constraint satisfaction without retraining.
\item We propose a self-correction training procedure exposing the denoiser to noisy versions of its own predictions. This reduces the training-inference mismatch and improves standard samplers (Euler, Euler-Maruyama, DDPM). \footnote{Code will be available at \url{https://github.com/MariiaDrozdova/continuous_diffusion_for_constrained_tasks}}
\end{itemize}

\section{Preliminaries}
\label{sec:prelim}

We briefly review the diffusion notation used throughout the paper. We use
\(t=\timedata\) for clean data and \(t=\timenoise\) for noise. In the DDPM
formulation, \(t\in\{\timedata,\timelast,\ldots,\timenoise\}\) is discrete and
\(\ttodataf{t}\) denotes one step toward data. For ODE/SDE samplers, \(t\) is
continuous in \([\timedata,\timenoise]\). We use the same symbols
\(\xf{t}\), \(\alphaf{t}\), and \(\betaf{t}\) in both cases.

\paragraph{Forward process.}
Let \(\xdata\in\mathbb{R}^d\) be a clean sample. We use the standard
variance-preserving diffusion path
\begin{equation}
    \xt
    =
    \alphaf{t}\xdata
    +
    \betaf{t}\noise,
    \qquad
    \noise\sim\mathcal{N}(0,I),
    \qquad
    \betaf{t}\doteq\sqrt{1-\alphaf{t}^{2}},
    \label{eq:forward_marginal}
\end{equation}
where $\xt\in\mathbb{R}^d$ is the resulting noisy sample, \(\alphaf{\timedata}=1\) and \(\alphaf{\timenoise}\approx 0\). Thus
\(q(\xt\mid\xdata)=
\mathcal{N}(\alphaf{t}\xdata,\betaf{t}^{2}I)\). The equivalent discrete Markov
chain realization is given in Appendix~\ref{app:sampling_0}.

\paragraph{DDPM reverse process.}
In the discrete formulation, the forward Markov chain gives a tractable Gaussian
posterior
\begin{equation}
    q(\xf{\ttodataf{t}}\mid \xt,\xdata)
    =
    \mathcal{N}\!\left(
        \mu_t(\xt,\xdata),
        \tilde{\beta}_t I
    \right),
    \label{eq:ddpm_posterior_short}
\end{equation}
with closed-form mean and variance given in Appendix~\ref{app:sampling_0}. Since
\(\xdata\) is unknown at sampling time, DDPM replaces it with the model's current
clean-sample estimate:
\begin{equation}
    p_\theta(\xf{\ttodataf{t}}\mid \xt)
    \doteq
    q\!\left(
        \xf{\ttodataf{t}}
        \mid
        \xt,
        \xhatf{\timedata}(\xt,t)
    \right).
    \label{eq:ddpm_step}
\end{equation}
In other words, each DDPM reverse step conditions on both the current noisy state
\(\xt\) and the model's current clean guess \(\xhatf{\timedata}(\xt,t)\).

\paragraph{Training objective and parameterization.}
Diffusion models can be written in several equivalent parameterizations,
including clean-sample prediction, noise prediction, score prediction, and
velocity prediction \cite{ho2020denoising,song2019generative,lipman2022flow,
lipman2024flow}. At the optimum, these quantities can be converted
between one another; the conversion formulas are given
in Appendix~\ref{app:sampling_0}. Following \cite{li2511back}, the prediction parameterization and the loss
target are separate design choices. In this work, we use \(x\)-prediction and
train directly against the clean sample:
\begin{equation}
    \mathcal{L}_{\mathrm{simple}}(\theta)
    =
    \mathbb{E}_{\xdata,t,\noise}
    \left[
        \left\|
            f_\theta(\xt,t)-\xdata
        \right\|^2
    \right],
    \qquad
    \xhatf{\timedata}(\xt,t)\doteq f_\theta(\xt,t).
    \label{eq:simple_loss}
\end{equation}
Under that loss, the optimal \(x\)-prediction is
\(\mathbb{E}[\xdata\mid \xt]\). By Tweedie's formula \cite{efron2011tweedie},
\begin{equation}
    \mathbb{E}[\xdata\mid \xt]
    =
    \frac{
        \xt+\betaf{t}^{2}\nabla_{\xt}\log q_t(\xt)
    }{
        \alphaf{t}
    },
    \label{eq:tweedie}
\end{equation}
where \(q_t\) is the marginal distribution of \(\xt\). We therefore interpret
\(\xhatf{\timedata}(\xt,t)\) as a plug-in Tweedie estimate of the clean sample.

\paragraph{Continuous-time samplers and noise scale.}
For continuous-time samplers, the same path can be described by a marginal vector
field \(u_t(x)\), yielding the probability-flow ODE
\( \mathrm{d}x
    =
 u_t(x)\,\mathrm{d}t.
\)

More generally, for any non-negative diffusion schedule \(\sigmaf{t}\), the reverse-time SDE
\begin{equation}
    \mathrm{d}x
    =
    \left[
        u_t(x)
        -
        \frac{\sigmaf{t}^{2}}{2}
        \nabla_x\log q_t(x)
    \right]\mathrm{d}t
    +
    \sigmaf{t}\,\mathrm{d}w_t
    \label{eq:sde_family}
\end{equation}
shares the same marginals $q_t$ \cite[Thm.~17]{holderrieth2025introduction}. Thus \(\sigmaf{t}\) changes the
stochasticity of the sampler without changing the target probability path
theoretically. The probability-flow ODE is recovered by setting
\(\sigmaf{t}=0\), while the variance-preserving reverse SDE of
\cite{song2021scorebased} corresponds to
\(\sigmaf{t}^{2}=-2\dalphaf{t}/\alphaf{t}\). In practice, learned denoisers and finite-step solvers introduce model and discretization error, so different choices of $\sigmaf{t}$ can lead to different empirical performance. We therefore treat $\sigmaf{t}$ as a sampler hyperparameter for Euler-Maruyama-based samplers. 

\section{Method}

\subsection{Inference under constraint violations}
\label{sec:constraint_violations}

Under the denoising objective \eqref{eq:simple_loss}, the model is only exposed to noisy versions $\xt$ of valid samples $\xdata$. The learned denoiser \(f_\theta(\xt,t)\) is then used as a plug-in estimate of the clean object in reverse updates of the form \(q(\xf{\ttodataf{t}} \mid \xt, \hat{x}_0)\).  However, at inference time on constrained tasks, imperfect score estimates can drive the sampling trajectory toward intermediate states that are unlikely under the forward noising trajectory of any valid solution. 

This creates a mismatch between the states encountered in training and those visited during inference (exposure bias) \cite{ning2023input, ningelucidating, deng2023markuptoimage, ren2024multi, zhang2025anti, li2024error}. In this regime, the standard DDPM reverse update $q(\xf{\ttodataf{t}} \mid \xt, \xhatf{\timedata})$ can preserve information from
the current state even when that state contains wrong discrete commitments,
leading to global constraint violations. We address this mismatch from two
directions: (i) a sampling rule that reduces direct dependence on the current
state after forming the clean proposal $\xhatf{\timedata}$
(Sec.~\ref{sec:tweedie_reprojection}), and (ii) a training objective that exposes
the model to its own predictions, improving robustness to the states encountered
during sampling (Sec.~\ref{sec:self_correction}).

\subsection{A stochastic anchor-free update}
\label{sec:tweedie_reprojection}

First, we investigate the simplified reverse update that reduces the influence of the current state $\xt$.

We denote \(\xhatf{\timedata}\equiv \xhatf{\timedata}(\xt,t)\), and set $B_t = \sqrt{\diffusioncoeff{t}}\,
(1-\alphaf{\ttodataf{t}}^{2}) / (1-\alphaf{t}^{2})$ where $\diffusioncoeff{t} = \frac{\alphaf{t}^{2}}{\alphaf{\ttodataf{t}}^{2}}$. 
The DDPM ancestral update can then be written as (See Appendices \ref{app:ddpm_derivation} and \ref{app:subsection:tweedie_ddpm})
\begin{equation}
\label{eq:ddpm_anchor_decomposition}
p_\theta^{\mathrm{DDPM}}(\xf{\ttodataf{t}}\mid \xt)
=
\mathcal N\!\Bigl(
\underbrace{\alphaf{\ttodataf{t}}\xhatf{\timedata}}_{\text{Tweedie center}}
+
\underbrace{
B_t\bigl(\xt-\alphaf{t}\xhatf{\timedata}\bigr)
}_{\text{$x_t$-dependent residual}},
\;
\underbrace{
(1-\alphaf{\ttodataf{t}}^{2})
}_{\text{marginal variance}}
\underbrace{
\frac{1-\diffusioncoeff{t}}{1-\alphaf{t}^{2}}
}_{\text{shrinking factor}}
I
\Bigr).
\end{equation}

We investigate the limiting case with no direct dependence on \(\xt\) in the reverse step. We refer to this update as \textsc{Tweedie
reprojection}:
\begin{equation}
\label{eq:tweedie_reprojection}
p_\theta^{\mathrm{Tw}}(\xf{\ttodataf{t}}\mid \xt)
=
\mathcal N\!\left(
\alphaf{\ttodataf{t}}\xhatf{\timedata},
(1-\alphaf{\ttodataf{t}}^{2})I
\right).
\end{equation}

This update removes the direct \(x_t\)-dependent residual from the reverse mean and re-noises the clean prediction using the forward marginal variance. Thus, \(x_t\) is only used through the denoiser prediction \(\hat x_0(x_t,t)\). Tweedie reprojection is an endpoint of the generalized DDIM sampler from \cite{song2020denoising} (see Appendix~\ref{sec:ddim_general}). We use it as an inference-time diagnostic to test whether constraint satisfaction has already been learned by the denoiser, but is not preserved by the standard sampler. 

\subsection{Self-correction training}
\label{sec:self_correction}

To improve constraint satisfaction of the final sample, we introduce \textsc{self-correction}, a training procedure that exposes the model to its own imperfect predictions.

\begin{wrapfigure}{r}{0.45\textwidth}
    \vspace{-10pt}
    \centering
    \includegraphics[width=0.5\textwidth]{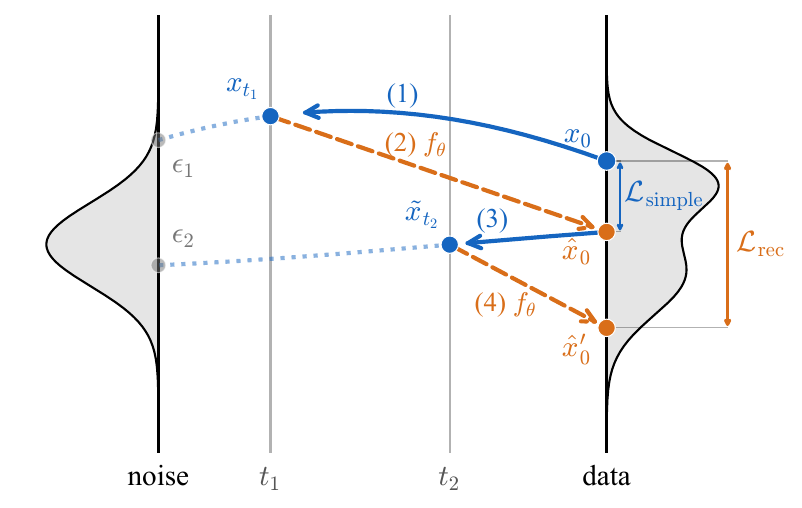}
    \caption{Self-correction training: Combining $\mathcal{L}_\mathrm{rec}$ and $\mathcal{L}_\mathrm{simple}$, training exposes the model to potentially encountered states during inference, which encourages it to correct constraint violations instead of only reinforcing locally consistent structure.}
\label{fig:self_correction}
\vspace{-4.1pt}
\end{wrapfigure}

Illustrated in Figure \ref{fig:self_correction}, given a clean sample $\xdata$, we first sample
a noise level $t_1 \sim \mathcal{U}[\timedata, \timenoise]$ and construct $\xf{t_1}$ as usual. We then compute an intermediate
prediction $\xhatf{\timedata} = f_\theta(\xf{t_1},  t_1)$.

This prediction can be imperfect and may violate constraints, especially when $t_1$ is far from data. We next noise $\xhatf{\timedata}$ (with gradients disabled) at a second level $t_2 \sim \mathcal{U}[\timedata, t_1]$ to obtain $\xtildef{t_2} = \alphaf{t_2} \xhatf{\timedata} + \betaf{t_2} \epsilon_2$. The model is trained to recover the original $\xdata$ from $\xtildef{t_2}$ using the loss $\mathcal{L}_{\mathrm{rec}} = \mathbb{E}[\|f_\theta(\xtildef{t_2}, t_2) - \xdata\|^2]$. This procedure generalizes naturally to multiple prediction steps, though longer unrolls were unstable when trained from scratch. A warm-started multi-step variant trained successfully but did not outperform one-step self-correction; see Appendix~\ref{app:loss_ablations}. The full training objective combines this recovery loss with the usual denoising objective $\mathcal{L}_{\mathrm{simple}}$ from \eqref{eq:simple_loss}:
\begin{equation}
 \mathcal{L}_{\textsc{sc}}  = \mathcal{L}_{\mathrm{rec}} + \lambda_{\mathrm{simple}} \mathcal{L}_{\mathrm{simple}},\quad\lambda_{\mathrm{simple}} \geq 0.
\end{equation}
Here, $\mathcal{L}_{\mathrm{simple}}$ acts as an anchor to the original diffusion objective, preserving single-step denoising on forward-noised data while $\mathcal{L}_{\mathrm{rec}}$ trains correction from self-induced states. Algorithm \ref{alg:self_correction} summarizes the training procedure.

For conditional generation, known values are pinned to their ground-truth after each step. Details on the pinning procedure for training and inference are provided in Appendices \ref{app:self_correction_algo} and \ref{app:sampling_process}, respectively.

\subsection{Local distributional interpretation of self-correction}
\label{sec:sc_distributional_view}

{For any fixed continuous \(x_0\)-prediction \(\bar x_0\in\mathbb R^d\), not necessarily a discrete or valid object, define
\[
q_t(\cdot;\bar x_0)
=
\mathcal N\!\left(
    \alphaf{t}\bar x_0,
    \betaf{t}^2 I
\right).
\]

\begin{proposition}[Local self-correction matching]
\label{prop:local_self_correction_matching}
Fix a valid target \(x_0\) and an arbitrary continuous prediction \(\bar x_0\). Let \(t_0>0\) denote the starting time of the local reverse window, with subsequent reverse times satisfying \(t<t_0\). Suppose that, over a local reverse window, the denoiser returns \(\bar x_0\), and the reverse process uses the DDPM posterior kernels induced by this prediction. If
\(Y_{t_0}\sim q_{t_0}(\cdot;\bar x_0)\), then
\(
Y_t\sim q_t(\cdot;\bar x_0)
\)
at every later reverse time \(t\) in the window. Consequently, for independent \(\epsilon,\epsilon'\sim\mathcal N(0,I)\), with
\[
X_t^{\mathrm{SC}}
=
\alphaf{t}\bar x_0+\betaf{t}\epsilon',
\qquad
X_t^{\mathrm{std}}
=
\alphaf{t}x_0+\betaf{t}\epsilon,
\]
for \(\betaf{t}>0\), we have
\[
D_{\mathrm{KL}}\!\left(
    \mathcal L(X_t^{\mathrm{SC}})
    \,\middle\|\,
    \mathcal L(Y_t)
\right)=0,
\qquad
D_{\mathrm{KL}}\!\left(
    \mathcal L(X_t^{\mathrm{std}})
    \,\middle\|\,
    \mathcal L(Y_t)
\right)
=
\frac{\alphaf{t}^{2}}{2\betaf{t}^{2}}
\lVert x_0-\bar x_0\rVert_2^2.
\]
\end{proposition}

The proposition gives a local interpretation of self-correction. It supposes that, over a short reverse-time window, the denoiser repeatedly predicts the same continuous proposal \(\bar x_0\). In this regime, self-correction generates exactly the same local proposal-centered distribution as the idealized reverse process, whereas standard diffusion training remains centered on the original clean sample \(x_0\). Consequently, self-correction explicitly trains the denoiser on the proposal-centered states that it may revisit during inference, while still retaining \(x_0\) as the training target.
This fixed-proposal result is a local idealization; Appendix \ref{app:local_proposal_explained} extends it to the more realistic case where successive proposals vary but remain within a small neighborhood of a common proposal.
}

\section{Experiments}

We train and evaluate on the following benchmarks for constrained data generation and in-painting/completion with a focus on discrete-space reasoning tasks: Sudoku, Sudoku-Extreme, graph connectivity (GC), Latin squares, and \(N\)-queens (see Appendix~\ref{app:additional_tasks_info} for details). For Sudoku, we use completed boards from Sudoku-Extreme and randomly generate conditioning masks, evaluating on 21-clue puzzles and the Medium and Hard clue-count settings of SRM~\cite{wewer2025spatial}. For Sudoku-Extreme~\cite{jolicoeur2025less,wang2025hierarchical}, we instead use the original masks guaranteeing a unique solution for each puzzle. Following IRED~\cite{du2024learning}, we train GC on \(N=12\) and evaluate on \(N=12\) and \(N=18\). We additionally evaluate Latin squares (\(N=7\)) and \(N\)-queens (\(N=14\)) under random in-painting and unconditional generation.

\newcommand{\NA}{\multicolumn{1}{c}{-}}
\newcommand{\best}[1]{\textbf{#1}}
\newcommand{\pmv}[2]{#1$\pm$#2}

\begin{table}[t]
\centering
\caption{\textbf{Success rates across samplers and tasks.} Self-correction training uses $\lambda_\mathrm{simple}=0.1$. Entries report mean $\pm$ standard deviation over three independently trained seeds, each evaluated with an independent test seed. Bold values indicate the best result for each configuration separately.}
\label{tab:main_results}
\setlength{\tabcolsep}{3.8pt}
\renewcommand{\arraystretch}{1.15}
\begin{adjustbox}{max width=\textwidth}
\begin{tabular}{lccccccccccc}
\toprule
\multirow{2}{*}{\makecell[c]{Sampler}}
&
\multicolumn{3}{c}{Sudoku}
&
\multicolumn{2}{c}{Sudoku-Extreme}
&
\multicolumn{2}{c}{N-Queens}
&
\multicolumn{2}{c}{Latin}
&
\multicolumn{2}{c}{GC}
\\
\cmidrule(lr){2-4}
\cmidrule(lr){5-6}
\cmidrule(lr){7-8}
\cmidrule(lr){9-10}
\cmidrule(lr){11-12}
&
\makecell{21 clues}
&
\makecell{Medium}
&
\makecell{Hard}
&
\makecell{pass@1}
&
\makecell{pass@10}
&
\makecell{random}
&
\makecell{gen}
&
\makecell{random}
&
\makecell{gen}
&
\makecell{$N=12$}
&
\makecell{$N=18$}
\\
\midrule

\multicolumn{12}{c}{\textbf{Baseline} \textnormal{(no self-correction loss)}} \\
\addlinespace[1pt]

DDPM
& \pmv{.31}{.02}
& \pmv{.83}{.01}
& \pmv{.38}{.00}
& \pmv{.05}{.00}
& \pmv{.19}{.01}
& \pmv{.53}{.00}
& \pmv{.06}{.01}
& \pmv{.71}{.03}
& \pmv{.87}{.03}
& \pmv{.80}{.01}
& \pmv{.67}{.01}
\\

Euler
& \pmv{.20}{.01}
& \pmv{.78}{.00}
& \pmv{.23}{.00}
& \pmv{.04}{.00}
& \pmv{.13}{.01}
& \pmv{.45}{.02}
& \pmv{.04}{.01}
& \pmv{.58}{.02}
& \pmv{.75}{.03}
& \pmv{.74}{.01}
& \pmv{.58}{.02}
\\

EM
& \pmv{.21}{.01}
& \pmv{.76}{.01}
& \pmv{.28}{.00}
& \pmv{.04}{.00}
& \pmv{.13}{.01}
& \pmv{.45}{.02}
& \pmv{.04}{.00}
& \pmv{.59}{.02}
& \pmv{.79}{.07}
& \pmv{.73}{.00}
& \pmv{.56}{.02}
\\

\addlinespace[2pt]
\cdashline{1-12}
\addlinespace[1pt]

EM decay
& \pmv{.81}{.03}
& \pmv{.96}{.01}
& \pmv{.80}{.02}
& \pmv{.18}{.02}
& \pmv{.58}{.03}
& \pmv{.69}{.01}
& \pmv{.15}{.03}
& \pmv{.96}{.02}
& \pmv{.99}{.01}
& \pmv{.94}{.01}
& \pmv{.89}{.03}
\\

\makecell[l]{Tweedie\\reprojection\\(ours)}
& \best{\pmv{.95}{.01}}
& \best{\pmv{.99}{.00}}
& \best{\pmv{.93}{.01}}
& \best{\pmv{.26}{.01}}
& \best{\pmv{.64}{.02}}
& \best{\pmv{.83}{.00}}
& \best{\pmv{.25}{.01}}
& \best{\pmv{.99}{.01}}
& \best{\pmv{1.00}{.00}}
& \best{\pmv{.97}{.00}}
& \best{\pmv{.95}{.00}}
\\

\midrule

\multicolumn{12}{c}{\textbf{Self-correction loss} \textnormal{(with $\lambda_\mathrm{simple}=0.1$)}} \\
\addlinespace[1pt]

DDPM
& \pmv{.87}{.02}
& \pmv{.96}{.00}
& \pmv{.83}{.02}
& \pmv{.22}{.01}
& \pmv{.64}{.04}
& \pmv{.74}{.04}
& \pmv{.26}{.04}
& \pmv{.98}{.01}
& \pmv{.92}{.04}
& \pmv{.88}{.03}
& \pmv{.77}{.05}
\\

Euler
& \pmv{.79}{.02}
& \pmv{.94}{.00}
& \pmv{.67}{.01}
& \pmv{.22}{.02}
& \pmv{.61}{.03}
& \pmv{.68}{.02}
& \pmv{.14}{.04}
& \pmv{.94}{.02}
& \pmv{.64}{.07}
& \pmv{.84}{.03}
& \pmv{.68}{.04}
\\

EM
& \pmv{.82}{.02}
& \pmv{.93}{.01}
& \pmv{.73}{.03}
& \pmv{.21}{.02}
& \pmv{.61}{.03}
& \pmv{.66}{.02}
& \pmv{.22}{.04}
& \pmv{.97}{.02}
& \pmv{.78}{.07}
& \pmv{.81}{.08}
& \pmv{.68}{.04}
\\

\addlinespace[2pt]
\cdashline{1-12}
\addlinespace[1pt]

EM decay
& \best{\pmv{.98}{.00}}
& \best{\pmv{.99}{.00}}
& \best{\pmv{.98}{.00}}
& \best{\pmv{.53}{.04}}
& \best{\pmv{.90}{.02}}
& \pmv{.81}{.03}
& \best{\pmv{.47}{.04}}
& \best{\pmv{1.00}{.00}}
& \best{\pmv{1.00}{.00}}
& \pmv{.88}{.02}
& \pmv{.80}{.04}
\\

\makecell[l]{Tweedie\\reprojection\\(ours)}
& \best{\pmv{.98}{.01}}
& \best{\pmv{.99}{.00}}
& \pmv{.94}{.01}
& \pmv{.49}{.01}
& \pmv{.84}{.02}
& \best{\pmv{.89}{.01}}
& \pmv{.35}{.06}
& \best{\pmv{1.00}{.00}}
& \pmv{.98}{.02}
& \best{\pmv{.99}{.01}}
& \best{\pmv{.99}{.01}}
\\

\midrule

\multicolumn{12}{c}{
\textbf{Discrete diffusion (D3PM)}
\textnormal{reference}
} \\
\addlinespace[1pt]

Ancestral
& \pmv{.57}{.05}
& \pmv{.88}{.02}
& \pmv{.52}{.04}
& \pmv{.03}{.01}
& \pmv{.10}{.03}
& \pmv{.85}{.01}
& \pmv{.28}{.01}
& \pmv{.88}{.03}
& \pmv{.65}{.05}
& \pmv{.46}{.00}
& \pmv{.29}{.02}
\\

Confidence
& \best{\pmv{1.00}{.00}}
& \best{\pmv{1.00}{.00}}
& \best{\pmv{.99}{.01}}
& \best{\pmv{.24}{.02}}
& \best{\pmv{.44}{.05}}
& \best{\pmv{.97}{.00}}
& \best{\pmv{.48}{.01}}
& \best{\pmv{1.00}{.00}}
& \best{\pmv{.99}{.00}}
& \best{\pmv{.99}{.02}}
& \best{\pmv{1.00}{.00}}
\\

Remask
& {\pmv{.99}{.00}}
& \best{\pmv{1.00}{.00}}
& \pmv{.98}{.00}
& \pmv{.23}{.02}
& \pmv{.42}{.03}
& \best{\pmv{.97}{.00}}
& \best{\pmv{.48}{.01}}
& \best{\pmv{1.00}{.00}}
& \pmv{.96}{.01}
& \best{\pmv{.99}{.02}}
& \best{\pmv{1.00}{.00}}
\\

\bottomrule
\end{tabular}
\end{adjustbox}
\vspace{-2em}
\end{table}

As we use continuous diffusion models, we lift discrete configurations into a continuous space \cite{chen2022analog} by encoding each token as a one-hot vector, yielding $\xdata \in \mathbb{R}^{m\times d}$ (see Appendix \ref{app:one_hot}). For all tasks except MNIST Sudoku \cite{wewer2025spatial} and Graph Connectivity (GC) \cite{du2024learning}, our network $f_\theta$ is a Transformer \cite{vaswani2017attention} with shared 
hyperparameters, taking as input a continuous sample $\xt \in \mathbb{R}^{m \times d}$ and diffusion time $t$ (see Appendix \ref{app:main_architecture} for architecture details).
For Graph Connectivity, we adopt the Neural Logic Machine-based \cite{dong2018neural} architecture of IRED \cite[Table 11]{du2024learning} with their hyperparameters. As for MNIST Sudoku, we use the pre-trained SRM model \cite{wewer2025spatial} without any additional training.

Results are compared across different samplers: DDPM, Euler, Euler-Maruyama (EM), EM decay (EM with decaying \(\sigma_t\); \eqref{eq:sde_family}), and our Tweedie reprojection sampler. For EM, we tune a constant noise level \(\sigma\) separately for each configuration. We first select a checkpoint step, based on the validity rate on the validation set under deterministic Euler sampling. At this step, we partition the validation set into five folds and select a single \(\sigma\), by maximizing the pass@1 validity rate averaged over all folds. The fixed \((\text{step},\sigma)\) setting is then evaluated once on the untouched test set for each seed, and we report mean \(\pm\) std across the three seeds. For EM decay, we jointly tune \((\sigma,t_{\mathrm{decay}})\) using the same procedure. The sampler uses the selected noise level before \(t_{\mathrm{decay}}\) and linearly anneals it to zero toward the data endpoint. See Appendix \ref{app:hyperparam_hardware} for additional details on training and inference hyperparameters.

\paragraph{Results on discrete constrained tasks.}
Table~\ref{tab:main_results} compares the continuous samplers within each training regime. The D3PM rows use separately trained discrete-diffusion models and are included as an external reference. We first consider the baseline models, trained without self-correction. The standard samplers, DDPM and EM, give broadly similar performance across tasks, with DDPM requiring no sampler hyperparameter tuning and EM using the noise level chosen according to the protocol described above. We then compare against two ways of changing the stochastic reverse process: EM decay, which allows larger noise early in sampling and anneals it toward the data endpoint, and Tweedie reprojection, which changes the reverse-step center by removing the $x_t$ residual. Both modifications improve over the standard samplers under baseline, while Tweedie reprojection gives the largest and most consistent gains. For example, on 21-clue Sudoku, success improves from about $31\%$ with DDPM to $95\%$ with Tweedie reprojection. Similar improvements appear across other datasets. This suggests that, for one-hot encoded constrained discrete tasks, a model trained with the standard denoising objective can already make useful clean predictions.

The second block of Table~\ref{tab:main_results} shows the effect of self-correction training. Keeping the same sampler comparison, self-correction substantially improves DDPM and EM, narrowing the gap between standard sampling and Tweedie reprojection. For instance, DDPM rises from about $31\%$ to $87\%$ on 21-clue Sudoku and from about $19\%$ to $64\%$ on Sudoku-Extreme pass@10. This suggests that self-correction partially fixes the reverse-process failure by training the model on sampler-induced states rather than only on forward-noised valid samples. Finally, combining self-correction with the modified samplers gives the strongest overall results: Tweedie reprojection and EM decay are both highly competitive, with each being best or tied for best on different tasks. Overall, the table supports two conclusions: removing the $x_t$ residual is already highly beneficial at inference time for constrained discrete tasks represented as one-hot vectors, and self-correction further improves robustness by training the model to recover from errors encountered during sampling.

Table~\ref{tab:main_results} focuses on exact validity, a strict metric under which a sample fails if any constraint is violated. To provide a more graded view of the generated solutions, we report two additional diagnostics in Appendix~\ref{app:other_metrics}. Table~\ref{tab:commitment} measures how close the final continuous state is to its nearest one-hot configuration, while Table~\ref{tab:constraint-violations} reports the fraction of predicted cells involved in constraint violation. Together, these metrics help distinguish failures of discrete commitment from samples that are locally plausible but still violate a small number of global constraints.

D3PM's confidence-based and remasking samplers perform strongly on standard Sudoku, conditioned \(N\)-queens, and Latin squares. On Sudoku-Extreme, our best self-corrected continuous configurations achieve higher validity than the D3PM samplers evaluated here. For broader context, TRM~\cite{jolicoeur2025less} reports \(87\%\) with a special MLP variant on Sudoku-Extreme (\(74.7\%\) with an attention-based backbone), while IRED~\cite{du2024learning} reports \(99.1\%\) and \(93.8\%\) on graph connectivity at \(N=12\) and \(N=18\). Although not compute-matched, our self-corrected Tweedie sampler reaches \(84\%\) pass@10 on Sudoku-Extreme and nearly perfect accuracy on graph connectivity, while EM decay reaches \(90\%\) pass@10 on Sudoku-Extreme. Together, these results suggest that continuous diffusion remains promising for constrained discrete tasks once the training-inference mismatch is addressed.

\paragraph{MNIST-Sudoku.}
As an inference-only sanity check beyond one-hot grids, we evaluate Tweedie reprojection on the MNIST-Sudoku Hard split of SRM~\cite{wewer2025spatial}, where symbols are represented as MNIST digit images~\cite{lecun1998mnist} but success is still exact Sudoku validity. We use the released SRM diffusion-baseline checkpoint and change only the inference sampler, without retraining or self-correction. The original rectified flow baseline achieves $\mathbf{0.8\%}$ accuracy, while Tweedie reprojection raises accuracy to $\mathbf{65.7\%}$ ($N{=}1000$, 95\% CI $\pm 3$ pp), exceeding the best SRM sampling-order strategy at $51.6\%$.

The SRM experiment isolates the effect of inference alone, since the underlying model is kept fixed. To additionally test whether our conclusions depend on one-hot representations or argmax decoding, we train our own models using alternative continuous representations and decoders (Appendix~\ref{app:other_representations}). We replace one-hot vectors with analog-bit codes and threshold decoding, fixed random embeddings and nearest-neighbour decoding, and mini MNIST-Sudoku with a learned CNN decoder. Across four representations in the single-seed experiments, baseline DDPM validity of \(0.31,0.19,0.26,0.01\) increases to \(0.94,0.89,0.89,0.42\) with Tweedie reprojection and to \(0.87,0.79,0.76,0.39\) with self-correction, showing that the findings are not specific to one-hot representations or argmax decoding.

In addition to the self-correction ablations in Appendix~\ref{app:loss_ablations}, we report rectified-flow and consistency-model experiments in Appendices~\ref{app:rectified_flow_sudoku} and~\ref{app:consistency_models}.

\begin{figure}[t]
\centering

\makebox[\textwidth][c]{%
\begin{subfigure}[t]{0.45\textwidth}
    \centering
    \includegraphics[
        width=\linewidth,
        trim={0.3cm 0.2cm 0.3cm 0.84cm},
        clip
    ]{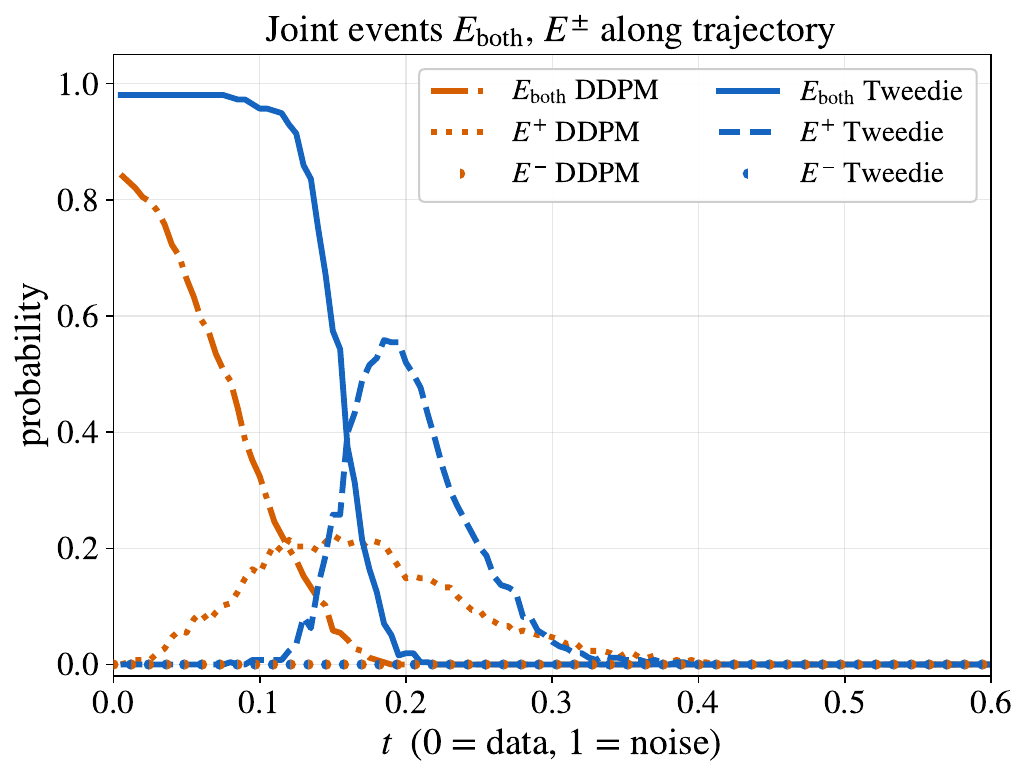}
    \caption{}
    \label{fig:analysis_eplus_eminus}
\end{subfigure}
\hspace{0.005\textwidth}
\begin{subfigure}[t]{0.45\textwidth}
    \centering
    \includegraphics[
        width=\linewidth,
        trim={0.2cm 0.2cm 0.25cm 0.84cm},
        clip
    ]{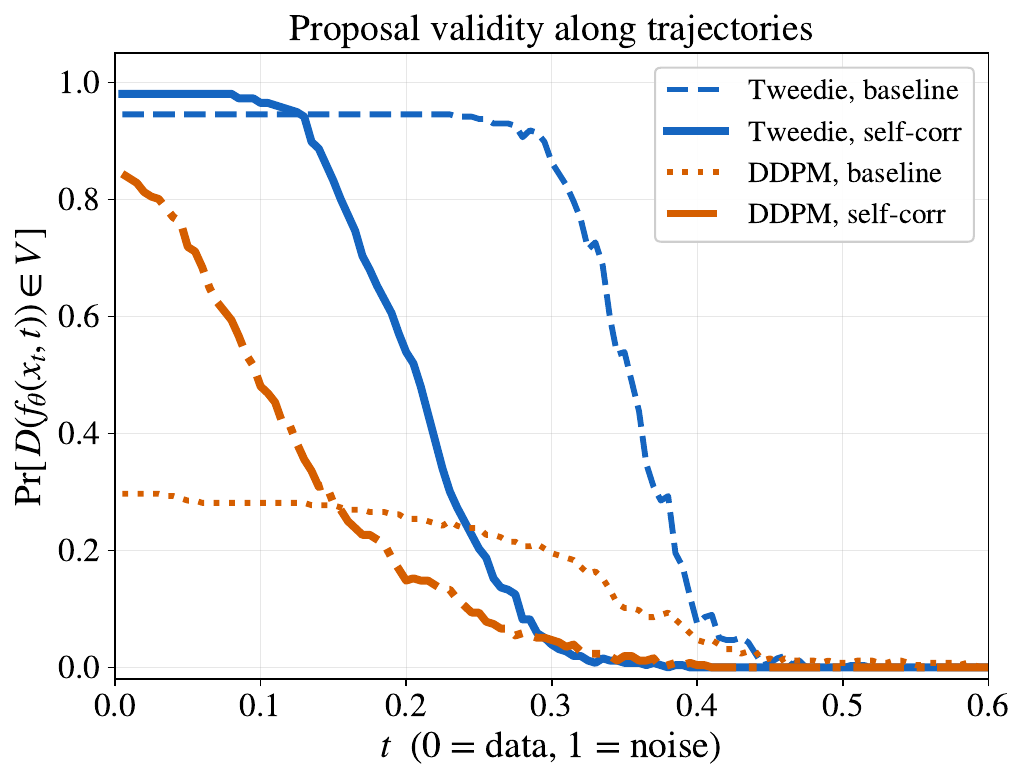}
    \caption{}
    \label{fig:analysis_proposals}
\end{subfigure}%
}
\caption{\textbf{DDPM-Tweedie gap on Sudoku.} \textbf{(a)} Decoded-center comparison at matched sampling times: \(E_{\mathrm{both}}\) means both centers are valid, \(E_+\) means only the Tweedie center is valid, and \(E_-\) means only the DDPM center is valid. \textbf{(b)} Validity of denoiser proposals \(D(f_\theta(x_t,t))\) along Tweedie and DDPM trajectories.}
\label{fig:analysis_mechanism}
\end{figure}

\section{Analysis}
\label{sec:analysis}

\paragraph{Locally plausible but globally invalid states.}
The results above show that the same denoiser can behave very differently under different samplers. The issue is not that the model never learns valid local symbols. The final continuous output can be close to a one-hot representation without its decoded grid being valid (Table~\ref{tab:commitment}). Similarly, in MNIST-Sudoku, individual cells can resemble recognizable digits while the full board violates Sudoku constraints.

This creates a specific training-inference mismatch. Standard denoising trains on forward-noised valid objects, \( x_t=\alpha(t)x_0+\beta(t)\epsilon, x_0\in\mathcal V. \) During inference, the model can instead visit noisy versions of its own imperfect proposals. These states may be locally plausible and globally invalid. Although Gaussian noise has full support, the structured invalid states produced by the sampler might receive little training mass under ordinary forward noising. The denoiser is therefore not directly trained to correct precisely the states that the reverse process may create.

\paragraph{DDPM can preserve globally invalid decoded states.}

Equation~\eqref{eq:ddpm_anchor_decomposition} shows that the DDPM reverse center contains an \(x_t\)-dependent residual in addition to the denoiser's clean prediction. This residual is not intrinsically harmful: it keeps the reverse update close to the current analog state, which is appropriate when the current state lies on a reliable trajectory and contains details that should persist. The problem in our setting is that \(x_t\) can already encode small mistakes that violate discrete constraints, and the denoiser has not necessarily been trained to correct such sampler-induced states. The residual can then keep the update near a locally plausible but globally invalid decoded configuration.

Figure~\ref{fig:analysis_mechanism}\subref{fig:analysis_eplus_eminus} compares the Tweedie and DDPM centers computed from the same current state \(x_t\) and denoiser prediction \(\hat x_0(x_t,t)\). We perform this comparison separately along trajectories generated by each sampler. We write \(E_{\mathrm{both}}\) for the event that both centers decode to valid configurations at the same reverse time, \(E_+\) for the event that only the Tweedie center is valid, and \(E_-\) for the event that only the DDPM center is valid. Thus the residual is harmless on \(E_{\mathrm{both}}\), helpful on \(E_-\), and harmful on \(E_+\). Across trajectories, \(E_+\) is much larger than \(E_-\). In this regime, the DDPM residual more often turns a valid clean proposal into an invalid centered update than it rescues an invalid proposal.
{
Tweedie reprojection is limited by proposal quality: it can exploit an informative clean proposal, but cannot compensate when the denoiser's proposals remain poor. Proposal quality is not the only source of failure, however: on Sudoku-Extreme, a valid proposal appeared at least once in \(62\%\) of failed Tweedie trajectories, but the final sample was still invalid (see Appendix \ref{app:tweedie_failures}).
}
\paragraph{Self-correction targets the missing training states.}
The same mechanism suggests a training-side fix. Instead of exposing the model only to noisy valid objects as standard diffusion does, self-correction also exposes the model to noisy versions of its own intermediate predictions while keeping the original valid object as the target. Thus the model is trained to map model-induced, possibly invalid states back to a valid solution. Figure~\ref{fig:analysis_mechanism}\subref{fig:analysis_proposals} shows this effect. Along DDPM trajectories, the baseline denoiser proposes valid grids much less often than along Tweedie trajectories. Self-correction substantially improves proposal validity on DDPM-induced states. This explains why self-correction narrows the DDPM-Tweedie gap: it does not remove the DDPM anchor, but it makes the denoiser more reliable on the states encountered during inference, improving the trajectory as well.

Additional ablations support this interpretation (see Appendix~\ref{app:loss_ablations}). Input perturbation~\cite{ning2023input} and self-conditioning~\cite{chen2022analog} are weaker than self-correction, suggesting that the important ingredient is not generic robustness to extra noise or an additional memory channel, but exposure to model-induced states. Random symbol corruptions (see Appendix~\ref{app:loss-variants}) also help DDPM, but remain weaker than self-correction, indicating that arbitrary invalid grids are less well matched to the inference sampling distribution than the model's own intermediate predictions.

\paragraph{The sampler gap is not only a noise-scale effect.}
On the sampling side, EM decay shows that additional stochasticity can also help: it injects more noise early in sampling and anneals this noise near the data endpoint. This can break some bad intermediate commitments while still allowing the sample to settle near the end. However, Appendix Figure~\ref{fig:analysis_noise} shows that changing the sampling variance alone does not close the DDPM-Tweedie gap. Scaling the DDPM noise changes how broadly the kernel samples, but not where it is centered. Tweedie reprojection instead changes the center by sampling around the clean proposal rather than continuing the same state-dependent update.

\section{Related work}

\paragraph{Diffusion models and samplers.}
Score-based and diffusion models \cite{sohl2015deep, ho2020denoising, song2021scorebased, song2020denoising, karras2022elucidating, lipman2024flow, holderrieth2025introduction, lai2025principles} have achieved strong results in high-dimensional generation. 
Beyond standard DDPM sampling, alternative inference schemes include DDIM \cite{song2020denoising}, standard ODE and SDE integrators (e.g. Euler, Heun, Euler-Maruyama) \cite{song2021scorebased}, stochastic predictor-corrector methods \cite{song2021scorebased}, and samplers with explicit Langevin-like noise injection steps to maintain correct marginals \cite{karras2022elucidating}. Some methods also self-condition their models on noisy and estimated clean samples to improve sample quality \cite{chen2022analog, watson2023novo}. 
All these methods differ in their trade-off between stability and stochastic exploration, and in how much each update is anchored to the previous noisy sample versus relying on a fresh denoised prediction.

\paragraph{Inference-time control and constrained generation.}
A large body of work incorporates constraints or objectives during or after sampling. Methods include hard conditioning by masking or inpainting \cite{janner2022diffuser, inoue2023layoutdm, mayet2025tdpaint}, Tweedie-based posterior sampling and guidance \cite{chung2023parallel,chung2023diffusion}, projection onto constraint sets \cite{christopher2024constrained, cardei2025constrained, utkarsh2025physics}, inference-time optimization \cite{utkarsh2025physics, christopher2024constrained, cardei2025constrained, li2025hardflow}, guidance \cite{janner2022diffuser, inoue2023layoutdm}, and search-based post-processing \cite{sun2023difusco}.
These approaches typically assume access to explicit constraints or a constraint violation signal. In contrast, we consider the setting where constraints must be learned implicitly from data \cite{du2024learning}, and study failure modes arising purely from inference dynamics.

\paragraph{Diffusion for structured and combinatorial tasks.}
Diffusion models have been applied to structured domains such as Sudoku, graphs, and combinatorial optimization \cite{sun2023difusco, du2024learning, wewer2025spatial, avdeyev2023dirichlet, ye2024beyond, kim2025train, li2025generation, pereira_2026_20037323}.
Discrete diffusion methods \cite{austin2021structured} and structured variants often outperform continuous diffusion \cite{chen2022analog} in such settings \cite{sun2023difusco}, highlighting the challenges of applying continuous diffusion to discrete-structured tasks. Continuous diffusion can be used by embedding discrete data in a continuous space \cite{hoogeboom2021argmax, chen2022analog, avdeyev2023dirichlet}, and these representations can be restricted to a bounded support such as the probability simplex \cite{avdeyev2023dirichlet}.
Our work explores standard continuous diffusion for highly structured generation and completion tasks.

\paragraph{Training-inference mismatch and exposure bias.}
Prior work has identified training-inference mismatch (exposure bias) as a source of error accumulation in sequential models \cite{bengio2015scheduled, ross2011reduction, huang2025self, bachmann2024pitfalls, chen2024diffusion, du2024learning}.
In diffusion models, similar effects arise because the model is trained on noisy ground-truth samples but receives its own predictions as inputs at inference time \cite{ning2023input, ningelucidating, deng2023markuptoimage, ren2024multi, zhang2025anti, li2024error}, and error has been shown to necessarily accumulate along the sampling trajectory of imperfect diffusion models under mild assumptions~\cite{li2024error}.
Recent approaches mitigate exposure bias by modifying the training process: perturbing inputs \cite{ning2023input}, minimizing cumulative errors as regularization \cite{li2024error}, or exposing the model to its own errors via truncated \cite{deng2023markuptoimage} or analytically simulated rollouts \cite{ren2024multi}. 
Other methods mitigate diffusion exposure bias by modifying training objectives, adding correction modules, or rescaling predictions at inference time~\cite{ren2024multi,zhang2025anti,ningelucidating}. Whereas the diffusion-specific approaches cited above address exposure bias primarily in visual generation, we study its effect on decoded globally constrained discrete tasks.

\paragraph{Non-diffusion solvers for constrained tasks.} OptNet \cite{amos2017optnet} and SATNet \cite{wang2019satnet} present differentiable optimization-based solver layers. More recent work explores recursive refinement models for reasoning tasks \cite{wang2025hierarchical, jolicoeur2025less}. Such models mitigate error compounding by training over rollouts of their refinement.
Our work is complementary: we do not aim to compete with these solvers, but to characterize a fundamental limitation of continuous diffusion-based inference in such settings.

\section{Limitations}
\label{sec:limitations}

Our experiments focus on constrained discrete tasks. Most tasks use one-hot encodings with argmax decoding, where validity is symbolic. MNIST-Sudoku, analog-bit and random-embedding experiments show that the same mechanism appears beyond exact one-hot vectors, but the evaluation is still symbolic Sudoku validity. We therefore do not claim that Tweedie reprojection is appropriate for perceptually rich domains, where success also requires modeling a diverse continuous distribution over texture, geometry, color, and fine details.

Tweedie reprojection changes the nature of the reverse process. It uses
\(x_t\) to form the proposal \(\hat x_0=f_\theta(x_t,t)\), but then discards the residual \(x_t-\alpha(t)\hat x_0\), so it no longer enforces the same proximity between noisy states as DDPM posterior updates. This can help when the state contains wrong symbolic commitments,  but in domains where meaningful variation exists within a decoded mode, or where fine continuous details are not fully captured by \(\hat x_0\), repeated reprojection may lose information that a state-preserving sampler would retain.

Finally, self-correction only partially addresses the training-inference mismatch. It exposes the model to one-step model-induced states, improving training coverage, but longer sampling trajectories can still visit states not well represented during training.  A more complete solution may require new training objectives or noise processes that better cover both forward-noised data and the sampler-induced states, which we leave for future work.

\section{Conclusion}
\label{sec:conclusion}

We studied continuous diffusion models for constrained discrete tasks represented
in continuous space. Across Sudoku, 
graph connectivity, Latin squares,
and \(N\)-queens, the same trained denoiser can behave very differently under
different reverse processes. This reveals a training-inference mismatch:
standard denoising trains on forward-noised valid objects, while inference
sampling can create locally plausible but globally invalid states that are not
well covered by the training distribution. Tweedie reprojection exposes this mismatch from the inference side by removing the direct
\(x_t\)-dependent residual from the reverse update. Improved validity indicates that the denoiser's clean proposals contain useful global structure
that the standard trajectory may fail to exploit. Self-correction addresses the
same problem from the training side by exposing the model to its own intermediate
predictions and training recovery to the original valid target.

Overall, our results suggest that continuous diffusion models can learn global
constraints, but constrained discrete reasoning requires better alignment between
the states used for training and the states produced during sampling. 


{
\small

\bibliographystyle{plain}   
\bibliography{literature}   
}

\appendix
\newpage
\section{Diffusion and parameterizations details}\label{app:sampling_0}

This appendix expands the diffusion notation, parameterization conversions, and
continuous-time derivations summarized in Sec.~\ref{sec:prelim}.

\subsection{Forward Markov chain}

The forward marginal used in the main text is
\[
    \xf{t}
    =
    \alphaf{t}\xdata
    +
    \betaf{t}\noise,
    \qquad
    \noise\sim\mathcal{N}(0,I),
    \qquad
    \betaf{t}=\sqrt{1-\alphaf{t}^{2}}.
\]
It induces
\begin{equation}
    q(\xf{t}\mid \xdata)
    =
    \mathcal{N}
    \left(
        \alphaf{t}\xdata,
        (1-\alphaf{t}^{2})I
    \right).
    \label{eq:forward_closed}
\end{equation}
Equivalently, the same marginals can be generated by the discrete Markov chain
\begin{equation}
    \xf{{t}}
    =
    \sqrt{\diffusioncoeff{t}}\,
    \xf{\ttodataf{t}}
    +
    \sqrt{1-\diffusioncoeff{t}}\,
    \noisef{t},
    \qquad
    \noisef{t}\sim\mathcal{N}(0,I),
    \label{eq:fwd_markov}
\end{equation}
with transition kernel
\[
    q(\xf{t}\mid \xf{\ttodataf{t}})
    =
    \mathcal{N}
    \left(
        \sqrt{\diffusioncoeff{t}}\xf{\ttodataf{t}},
        (1-\diffusioncoeff{t})I
    \right).
\]
The one-step coefficient is chosen so that Eq.~\eqref{eq:fwd_markov} reproduces
the marginals in Eq.~\eqref{eq:forward_closed}:
\begin{equation}
    \diffusioncoeff{t}
    =
    \frac{\alphaf{t}^{2}}{\alphaf{\ttodataf{t}}^{2}},
    \qquad
    \alphaf{t}^{2}
    =
    \prod_{s=1}^{t}\diffusioncoeff{s}.
    \label{eq:delta_def}
\end{equation}
\subsection{DDPM posterior}

Since the forward process is Gaussian, the posterior conditional on the clean
sample is tractable \cite{ho2020denoising}:
\begin{equation}
    q(\xf{\ttodataf{t}}\mid \xt,\xdata)
    =
    \mathcal{N}
    \left(
        \mu_t(\xt,\xdata),
        \tilde{\beta}_t I
    \right),
    \label{eq:posterior}
\end{equation}
with mean and variance given by:
\begin{align}
    \mu_t(\xt,\xdata)
    &=
    \alphaf{\ttodataf{t}}\xdata
    \frac{1-\diffusioncoeff{t}}{1-\alphaf{t}^{2}}
    +
    \frac{
        \sqrt{\diffusioncoeff{t}}
        \left(1-\alphaf{\ttodataf{t}}^{2}\right)
    }{
        1-\alphaf{t}^{2}
    }
    \xt,
    \label{eq:posterior_mean}
    \\
    \tilde{\beta}_t
    &=
    \left(1-\alphaf{\ttodataf{t}}^{2}\right)
    \frac{1-\diffusioncoeff{t}}{1-\alphaf{t}^{2}}.
    \label{eq:posterior_var}
\end{align}
Thus DDPM sampling replaces the unknown \(\xdata\) with the model prediction
\(\xhatf{\timedata}(\xt,t)\) and samples from
\[
    p_\theta(\xf{\ttodataf{t}}\mid \xt)
    \doteq
    q
    \left(
        \xf{\ttodataf{t}}
        \mid
        \xt,
        \xhatf{\timedata}(\xt,t)
    \right).
\]

The posterior mean can also be written in an anchored form:
\begin{equation}
    \mu_t(\xt,\xdata)
    =
    \alphaf{\ttodataf{t}}\xdata
    +
    B_t
    \left(
        \xt-\alphaf{t}\xdata
    \right),
    \qquad
    B_t
    \doteq
    \frac{
        \sqrt{\diffusioncoeff{t}}
        \left(1-\alphaf{\ttodataf{t}}^{2}\right)
    }{
        1-\alphaf{t}^{2}
    }.
    \label{eq:posterior_anchor_form}
\end{equation}
This decomposition separates the clean-sample component from the residual
inherited from the current noisy state. In the main paper, this residual is the
source of the anchoring effect that distinguishes DDPM-style updates from pure
Tweedie reprojection.

\subsection{Prediction parameterizations}
\label{app:parameterizations}

A diffusion model can be parameterized to predict the clean sample \(\xdata\),
the noise \(\noise\) \cite{ho2020denoising}, the marginal score
\(s_t(\xt)=\nabla_{\xt}\log q_t(\xt)\) \cite{song2019generative}, or the
marginal velocity \(u_t(\xt)\) in the flow-matching formulation
\cite{lipman2022flow,lipman2024flow}. The prediction parameterization and the
loss target are separate choices \cite{li2511back}. At the population optimum,
these quantities are equivalent up to deterministic conversions induced by the
forward process \cite{lai2025principles}.

Let
\[
    x^*_{\timedata}(\xt,t)
    \doteq
    \mathbb{E}[\xdata\mid \xt].
\]
Then the corresponding optimal noise, score, and velocity predictions are
\begin{equation}
\begin{gathered}
    \epsilon^*(\xt,t)
    =
    \frac{
        \xt-\alphaf{t}x^*_{\timedata}(\xt,t)
    }{
        \betaf{t}
    },
    \qquad
    s^*(\xt,t)
    \doteq
    \nabla_{\xt}\log q_t(\xt)
    =
    -\frac{\epsilon^*(\xt,t)}{\betaf{t}},
    \\
    u^*(\xt,t)
    =
    \frac{\dbetaf{t}}{\betaf{t}}\xt
    +
    \left(
        \dalphaf{t}
        -
        \frac{\alphaf{t}\dbetaf{t}}{\betaf{t}}
    \right)
    x^*_{\timedata}(\xt,t).
\end{gathered}
\label{eq:reparam}
\end{equation}
Tweedie's formula gives
\[
    x^*_{\timedata}(\xt,t)
    =
    \mathbb{E}[\xdata\mid \xt]
    =
    \frac{
        \xt+\betaf{t}^{2}\nabla_{\xt}\log q_t(\xt)
    }{
        \alphaf{t}
    }.
\]
Therefore an \(x\)-prediction model can be used in score-, noise-, or
velocity-based samplers by applying Eq.~\eqref{eq:reparam}. In our experiments,
we use \(x\)-prediction and the \(x\)-prediction loss in
Eq.~\eqref{eq:simple_loss}, following the practical distinction between
prediction type and loss target highlighted by \cite{li2511back}.

\subsection{Continuous-time derivation}
\label{sec:ct}

For continuous-time samplers, we view the variance-preserving forward process as
a probability path \(\{q_t\}_{t\in[\timedata,\timenoise]}\):
\begin{equation}
    \xf{t}
    =
    \alphaf{t}\xdata
    +
    \betaf{t}\noise,
    \qquad
    q_t(x\mid \xdata)
    =
    \mathcal{N}
    \left(
        \alphaf{t}\xdata,
        \betaf{t}^{2}I
    \right),
    \qquad
    \noise\sim\mathcal{N}(0,I).
    \label{eq:ct_path}
\end{equation}
Here \(t\) is continuous, \(\alphaf{\timedata}=1\),
\(\alphaf{\timenoise}\approx 0\), and
\(\betaf{t}=\sqrt{1-\alphaf{t}^{2}}\).

\paragraph{Conditional and marginal vector fields.}
Differentiating Eq.~\eqref{eq:ct_path} with respect to \(t\), for a fixed pair
\((\xdata,\noise)\), gives the conditional velocity field
\begin{equation}
    u_t(x\mid \xdata)
    =
    \dalphaf{t}\xdata
    +
    \dbetaf{t}\noise
    =
    \frac{\dbetaf{t}}{\betaf{t}}x
    +
    \left(
        \dalphaf{t}
        -
        \alphaf{t}
        \frac{\dbetaf{t}}{\betaf{t}}
    \right)
    \xdata,
    \label{eq:cond_vf}
\end{equation}
where we used
\[
    \noise
    =
    \frac{x-\alphaf{t}\xdata}{\betaf{t}}.
\]
Since \(u_t(x\mid\xdata)\) is linear in \(\xdata\), the marginal vector field is
obtained by replacing \(\xdata\) with
\(\mathbb{E}[\xdata\mid \xt=x]\). Using Tweedie's formula gives
\begin{equation}
    u_t(x)
    =
    \frac{\dalphaf{t}}{\alphaf{t}}x
    +
    \left(
        \betaf{t}^{2}
        \frac{\dalphaf{t}}{\alphaf{t}}
        -
        \betaf{t}\dbetaf{t}
    \right)
    \nabla_x\log q_t(x).
    \label{eq:marg_vf}
\end{equation}

\paragraph{Probability-flow ODE and SDE family.}
The deterministic dynamics with marginals \(q_t\) are given by the
probability-flow ODE
\begin{equation}
    \mathrm{d}x
    =
    u_t(x)\,\mathrm{d}t.
\label{eq:pf_ode}
\end{equation}
More generally, for any non-negative diffusion schedule \(\sigma(t)\), the reverse-time SDE driven by a standard Wiener process \(w_t\),
\[
    \mathrm{d}x
    =
    \left[
        u_t(x)
        -
        \frac{\sigmaf{t}^{2}}{2}
        \nabla_x\log q_t(x)
    \right]\mathrm{d}t
    +
    \sigmaf{t}\,\mathrm{d}w_t
\]
has the same marginals \(q_t\) in the exact-score and exact-solver limit
\cite[Thm.~17]{holderrieth2025introduction}. The probability-flow ODE is the
special case \(\sigmaf{t}=0\). The variance-preserving reverse SDE of
\cite{song2021scorebased} corresponds to
\[
    \sigmaf{t}^{2}
    =
    -2\frac{\dalphaf{t}}{\alphaf{t}}.
\]

\paragraph{Numerical samplers.}
The continuous-time formulation allows sampling by numerically solving either
the probability-flow ODE or the SDE family. We use ODE solvers such as Euler and
Heun, and SDE solvers such as Euler--Maruyama
\cite{song2021scorebased,karras2022elucidating}. In all cases, we plug the
network prediction \(\xhatf{\timedata}(\xt,t)\) into the required score or
velocity expression using the conversions in Eq.~\eqref{eq:reparam}. DDPM and
DDIM can also be derived from this continuous-time perspective
\cite{song2020denoising,song2021scorebased}.

\section{Samplers}\label{app:sampling}

We explain various sampling update rules below and then describe the overall sampling process for both generation and completion/in-painting with the pinning procedure.

\subsection{DDIM generalized formula: DDIM, DDPM and Tweedie reprojection}\label{sec:ddim_general}

DDIM \cite{song2020denoising} introduces a family of forward processes with the same marginal distributions \eqref{eq:forward_marginal} as DDPM \cite{ho2020denoising}. Their generalized sampling update takes the form:
\begin{align}
q_\kappa(\xf{\ttodataf{t}} \mid \xt, \xhatf{\timedata}) &= \mathcal{N}\!\bigl(\alphaf{\ttodataf{t}} \xhatf{\timedata} + \sqrt{1 - \alphaf{t-1}^2 - \kappa_t^2} \cdot \frac{\xt - \alphaf{t} \xhatf{\timedata}}{\sqrt{1-\alphaf{t}^2}}, \kappa_t^2 I\bigr)\label{eq:ddim_general}\\
&=\mathcal{N}\!\bigl(\alphaf{\ttodataf{t}} \xhatf{\timedata} + \sqrt{1 - \alphaf{t-1}^2 - \kappa_t^2} \cdot \epshatf{t}, \kappa_t^2 I\bigr)\label{eq:ddim_simplified_general}
\end{align}
Different choices of $\kappa_t^2$ recover known samplers: $\kappa_t^2 = \tilde{\beta}_t = (1-\alphaf{\ttodataf{t}}^{2}) \cdot \frac{1-\diffusioncoeff{t}}{1-\alphaf{t}^{2}}$ yields DDPM \cite{ho2020denoising}, while $\kappa_t^2 = 0$ gives the deterministic DDIM update \cite{song2020denoising}.

Our Tweedie reprojection corresponds to $\kappa_t^2 = 1 - \alphaf{t-1}^2$. This endpoint relation does not imply exact marginal preservation after replacing the true clean sample \(x_0\) by the learned prediction \(\hat x_0(x_t,t)\), even if that prediction equals the exact conditional mean.

\paragraph{Linear-path sampling for SRM.}
For the linear path \(x_t=(1-t)x_0+t\epsilon\), with normalized time \(t\in[0,1]\), write \(v_\theta(x_t,t)\) for the corresponding velocity prediction. The clean and noise estimates are
\[
\hat x_0=x_t-tv_\theta(x_t,t),
\qquad
\hat\epsilon=x_t+(1-t)v_\theta(x_t,t).
\]
The deterministic generalized DDIM update to \(s<t\) is therefore
\[
x_s=(1-s)\hat x_0+s\hat\epsilon
    =x_t+(s-t)v_\theta(x_t,t),
\]
which is an Euler step on the linear path. Tweedie reprojection instead uses \(x_s=(1-s)\hat x_0+s\epsilon'\), with fresh \(\epsilon'\sim\mathcal N(0,I)\).

\subsection{Time discretizations}

Whereas we use uniform / linear time discretization (equidistant timesteps) throughout our work, others can also be used in practice \cite{karras2022elucidating}.

We will denote a discretization of $N \in \mathbb{N}^+$ time points with:
\begin{equation}
\{t_{\istart}, t_{\isecond}, \dots, t_{\iend}\}, \qquad t_{\istart} = \timenoise, t_{\iend} = \timedata.
\end{equation}
where we also use $h_i \doteq \tnow - t_{i-1}$ as positive step sizes since timesteps are decreasing.

To compare the generalized sampling update \eqref{eq:ddim_general} (including DDIM, DDPM, and Tweedie reprojection) with other samplers, we can replace it as follows:
\begin{equation}
\xf{\tnext} = \alphaf{\tnext} \xhatf{\timedata} + \sqrt{1 - \alphaf{\tnext}^2 - \kappa_{\tnow}^2} \cdot \frac{\xnow - \alphaf{\tnow} \xhatf{\timedata}}{\sqrt{1-\alphaf{\tnow}^2}} + \kappa_{\tnow} \cdot \epsilonf{\tnow}
\label{eq:ddim_general_general_timesteps}
\end{equation}
with $\epsilonf{\tnow} \sim \mathcal{N}(0, I)$ and $i=\istart,\dots,\ilast$.

\subsection{Numerical ODE and SDE solvers: Euler, Heun, and Euler-Maruyama}\label{sec:numerical_solvers}

The Euler and Heun (aka improved Euler) methods are first- and second-order ODE solvers, popular for their low Number of Function Evaluations (NFE) and ease of implementation.

They become deterministic samplers once used to solve the PF-ODE \eqref{eq:pf_ode}. Starting from $\xf{\timenoise} \sim \mathcal{N}(0, I)$, Euler and Heun iterate from $i=\istart$ to $i=\ilast$, using the estimate of the marginal vector field, retrieved from the model's $x$-prediction.

\paragraph{Euler update rule}
\begin{equation}
\xnext = \xnow - h_i \cdot \uhatf{\tnow}(\xnow),\label{eq:euler}
\end{equation}

\paragraph{Euler-Maruyama (EM) update rule}

Recall that for \textit{any non-negative} diffusion schedule $\sigma_t$, the stochastic dynamics
\begin{equation}
    \mathrm{d}x
    = \Bigl[u_t(x) - \tfrac{\sigma_t^{2}}{2} \nabla_x \log q_t(x)\Bigr]\mathrm{d}t
    + \sigma_t \mathrm{d}w.
\end{equation}
share the same marginals $q_t$ \cite[Thm.~17]{holderrieth2025introduction}, reducing to the deterministic case \eqref{eq:pf_ode} for $\sigma_t = 0$.

Euler-Maruyama (EM) is a numerical SDE solver that can be used as a stochastic sampler. Starting from $\xf{\timenoise} \sim \mathcal{N}(0, I)$, EM iterates from $i=\istart$ to $i=\ilast$, using the estimate of the marginal vector field and score, retrieved from the model's $x$-prediction.
\begin{equation}
\xnext = \xnow + h_i \cdot \hat{f}(\xnow, \tnow) + g(\tnow) \sqrt{h_i} \cdot \epsilonf{\tnow},\qquad \epsilonf{\tnow} \sim \mathcal{N}(0, I)
\end{equation}
where we denote $\hat{f}(x, t) \doteq - \uhatf{t}(x) + \tfrac{\sigma_t^{2}}{2} \shatf{t}(x)$ as the \emph{drift} coefficient and $g(t) \doteq \sigma_t$ as the \emph{diffusion} coefficient.

In our work, we refer to EM as using a fixed $g(\tnow) = \sigma$, and EM decay  as using a \emph{decreasing} diffusion schedule $\sigma_t$.

\paragraph{Euler--Maruyama noise scale.}
The SDE family \eqref{eq:sde_family} motivates treating \(\sigmaf{t}\) as a sampler hyperparameter.
Theoretically, changing \(\sigmaf{t}\) changes the stochastic dynamics but not the marginal path \(q_t\), assuming exact scores and exact integration. In practice, finite-step solvers and learned denoisers introduce discretization and model error, so the choice of \(\sigmaf{t}\) affects empirical performance. We therefore tune the EM noise scale by five-fold cross-validation. For EM decay, we also tune the time at which \(\sigmaf{t}\) begins linearly annealing to zero near the data endpoint.

\subsection{Sampling process and pinning procedure}\label{app:sampling_process}

The algorithm \ref{alg:sampling_process} summarizes the overall sampling process for both unconditional and conditional generation (completion or in-painting), with pinning procedures for the latter. Recall that we go from $\xf{\timenoise}$ to $\xdata$ when sampling.

\begin{algorithm}[H]
\caption{(Un-)conditional generation. 
Inputs are the time discretization, optional conditioning $c$, its corresponding mask $m$, and soft-to-hard conditioning threshold $t^*$}
\label{alg:sampling_process}
\DontPrintSemicolon

\SetKwProg{Proc}{Procedure}{:}{}
\SetKwFunction{FSample}{sample}
\SetKwFunction{FPin}{pin}
\SetKwFunction{FSoftPin}{soft\_pinning}
\SetKwFunction{FHardPin}{hard\_pinning}

\Proc{\FSample{$\{t_\istart, t_\isecond, \dots, t_\iend\}$, $c$, $m$, $t^*$}}
{
    Sample initial tensor $\xf{\timenoise} \sim \mathcal{N}(0, I)$ \;
    Initialize buffer to store trajectory $\{\xf{t_i}\}_{i}$\;
    \For{$i=\istart$ \KwTo $\ilast$}{
        $s = t_{i-1}$ \;
        Compute step size $h_i = \tnow - s$ \;
        Sample $\xnext = \texttt{update\_rule($\xnow, \tnow, h_i, \dots$)}$ \; 
        \If{$m \neq \emptyset$}{
            
            $\texttt{pin($\xnext, \tnext, t^*, c, m$)}$ \; 
        }
        Append $\xnext$ to buffer \;
    }
}

\Proc{\FPin{$\xnext$, $\tnext$, $t^*$, $c$, $m$}}
{   
    \If{$\tnext \geq t^*$}{
        $\texttt{soft\_pinning($\xnext, c, m$)}$ \; 
    }
    \Else{
        $\texttt{hard\_pinning($\xnext, c, m$)}$ \; 
    }
}

\Proc{\FSoftPin{$\xnext$, $c$, $m$}}
{   
    Sample $c_\tnext \sim q(\xnext \mid c)$ \;
    $\xnext[m] = c_\tnext[m]$ \; 
}

\Proc{\FHardPin{$\xnext$, $c$, $m$}}
{
    $\xnext[m] = c[m]$ \; 
}
\end{algorithm}

\section{\textsc{self-correction} algorithm}
\label{app:self_correction_algo}

We summarize the \textsc{self-correction} training procedure in Algorithm \ref{alg:self_correction}. We denote $\mathrm{sg}(.)$ as the stop-gradient operator, and the rest are described in previous sections or are self-explanatory.

\begin{algorithm}[H]
\caption{\textsc{self-correction} training algorithm for (un-)conditional generation for a full batch. Inputs are the regularization coefficient $\lambda_\mathrm{simple}$, optional conditioning $c$ and its corresponding mask $m$}
\label{alg:self_correction}
\DontPrintSemicolon

\SetKwProg{Proc}{Procedure}{:}{}
\SetKwFunction{FTrain}{training\_procedure}
\SetKwFunction{FUpdateDiffusionModel}{update\_diffusion\_model}

\Proc{\FTrain{}}
{
    \For{steps}{
        \texttt{update\_diffusion\_model()} \;
    }
}
\Proc{\FUpdateDiffusionModel{$\lambda_\mathrm{simple}$, $c$, $m$}}
{
    Sample $\xdata \sim q_\timedata(x) $ \; 
    Sample $t_1 \sim \mathcal{U}[\timedata, \timenoise]$ and $t_2 \sim \mathcal{U}[\timedata, t_1]$\;
    Sample $\xf{t_1} \sim q(\xf{t_1} \mid \xdata)$\;
    \If{$m \neq \emptyset$}{
        $\texttt{hard\_pinning($\xf{t_1}, c, m$)}$ \;
    }
    Compute $\xhatf{\timedata} = f_\theta(\xf{t_1}, t_1)$ \;
    \If{$m \neq \emptyset$}{
        $\texttt{hard\_pinning($\xhatf{\timedata},  c, m$)}$ \;
    }
    Sample $\xtildef{t_2} \sim q(\xf{t_2} \mid \text{sg}(\xhatf{\timedata})$) \;
    \If{$m \neq \emptyset$}{
        $\texttt{hard\_pinning($\xtildef{t_2}, c, m$)}$ \;
    }
    Compute $\xhatf{\timedata}'=f_\theta(\xtildef{t_2}, t_2)$ \; 
    \If{$m \neq \emptyset$}{
        $\texttt{hard\_pinning($\xhatf{\timedata}', c, m$)}$ \;
    }
    Compute simple loss $\mathcal{L}_{\mathrm{simple}}(\theta)
    =
    \mathrm{MSE}(\xhatf{\timedata}, \xdata)$ \;
    Compute recovery loss $\mathcal{L}_{\mathrm{rec}}(\theta) = \mathrm{MSE}(\xhatf{\timedata}', \xdata)$ \;
    Compute final loss $\mathcal{L}_{\mathrm{SC}}(\theta)  = \mathcal{L}_{\mathrm{rec}}(\theta) + \lambda_{\mathrm{simple}} \mathcal{L}_{\mathrm{simple}}(\theta)$ \;
    Update model $f_\theta$ \;
}
\end{algorithm}

This procedure generalizes naturally to multiple prediction steps. Longer unrolls were unstable when trained from scratch; the warm-started multi-step experiment in Appendix~\ref{app:loss_ablations} trained successfully but did not outperform one-step self-correction.

\paragraph{Computational overhead of self-correction training.}
  Standard denoising training performs one gradient forward pass and one
  backward pass per optimizer step. Self-correction training adds (i)~one
  additional \emph{gradient-free} forward pass, which produces the model
  prediction that is re-noised to the second time point, and (ii)~one gradient
  forward pass on a sub-batch of $\lfloor \lambda_\mathrm{simple} B \rfloor$
  samples for $\mathcal{L}_\mathrm{simple}$; a single backward pass is taken on
  the combined objective. Counting a backward pass as twice the cost of a
  forward pass, this predicts a per-step training cost of
  $(4 + 3\lambda_\mathrm{simple})/3 \approx 1.43\times$ the baseline for
  $\lambda_\mathrm{simple}=0.1$. Both variants are trained for the
  same number of optimizer steps ($2\times10^6$), and inference cost is
  unchanged, as self-correction modifies only the training objective and not
  the sampler.

\section{Additional tasks details}\label{app:additional_tasks_info}

\newcolumntype{C}[1]{>{\centering\arraybackslash}p{#1}}

\begin{table}[t]
\centering
\caption{\textbf{Dataset summary.} Throughout, clues denote partial observations the model 
conditions on: randomly positioned revealed cells with count $k \sim 
\mathcal{U}\llbracket a, b \rrbracket$ where $\llbracket a, b \rrbracket = \{a, \dots, b\}$ (Sudoku, Latin square, $N$-queens); or the full adjacency matrix (GC, always given at 
train and inference). $^\dagger$ Out-of-distribution.}
\label{tab:dataset_summary}
\setlength{\tabcolsep}{4pt}
\renewcommand{\arraystretch}{1.18}

\begin{threeparttable}
\begin{adjustbox}{max width=\textwidth}
\begin{tabular}{l C{3.1cm} C{3.35cm} C{3.35cm} C{3.2cm} C{2.2cm}}
\toprule
 & Sudoku
 & Sudoku-Extreme
 & GC
 & Latin square
 & $N$-queens
\\
\midrule

Grid size
& $9 \times 9$
& $9 \times 9$
& $N \times N$
& $7 \times 7$
& $14 \times 14$
\\

Cells $m$
& $81$
& $81$
& $N^2$
& $49$
& $14$
\\

Vocab $d$
& $9$
& $9$
& $2$
& $7$
& $14$
\\

\midrule

Train clues
& $\llbracket 0,80 \rrbracket$
& $\llbracket 17,35 \rrbracket$
& \makecell[c]{Adj. mat.\\$N=12$}
& $\llbracket 0,N^2-1 \rrbracket$
& $\llbracket 0,N-1 \rrbracket$
\\

\midrule

Eval clues
& \makecell[c]{
$21$\\
$\llbracket 27,53 \rrbracket$ (Med.)\\
$\llbracket 0,26 \rrbracket$ (Hard)\\
}
& $\llbracket 17,36 \rrbracket$
& \hfill \makecell[l]{
Adj. mat., $N=12$\\
Adj. mat., $N=18^\dagger$
} \hfill
& \makecell[c]{
random: $14$\\
generation: $0$
}
& \makecell[c]{
random: $7$\\
generation: $0$
}
\\

\midrule

Validity
& \makecell[c]{Sudoku\\rules}
& \makecell[c]{Sudoku\\rules}
& \makecell[c]{Connectivity\\matches\\ground truth}
& \makecell[c]{One symbol\\per row/col}
& \makecell[c]{No two\\queens attack}
\\

\bottomrule
\end{tabular}
\end{adjustbox}
\end{threeparttable}
\end{table}

We train and evaluate on the following benchmarks for constrained data generation and in-painting/completion with a focus on discrete-space reasoning tasks: Sudoku-Extreme \cite{wang2025hierarchical, jolicoeur2025less}, MNIST Sudoku \cite{wewer2025spatial}, graph-connectivity (GC, \cite{du2024learning})  and two datasets we introduce: Latin square and N-queens. Throughout, clues denote randomly positioned revealed cells.

\paragraph{Sudoku.}
We consider two separate Sudoku training regimes, both based on the Sudoku-Extreme dataset~\cite{wang2025hierarchical,jolicoeur2025less}, which provides partial \(9\times9\) grids together with their completions. For the \emph{Sudoku} setting, we train on completed boards from the training split and generate conditioning masks on the fly, with the number of clues sampled uniformly from \(\{0,\ldots,80\}\). We evaluate this model with 21 clues, as well as the Medium and Hard clue-count settings of SRM~\cite{wewer2025spatial}, corresponding to clue counts sampled from \(\{27,\ldots,53\}\) and \(\{0,\ldots,26\}\), respectively. For the \emph{Sudoku-Extreme} setting, we train a separate model using the dataset-provided partial grids from the Sudoku-Extreme training split and evaluate it on the corresponding test split.

For evaluation, we sample \(2000\) puzzles per seed from the corresponding test split, using three seeds. For the Sudoku clue-count settings, masks are generated on the fly according to the evaluation regime, whereas for Sudoku-Extreme we use the dataset-provided partial grids. Since the Sudoku-Extreme test set is large, we evaluate on a uniformly sampled subset and report binomial confidence intervals; the standard error is at most \(1.12\) percentage points, corresponding to a 95\% confidence interval of approximately \(\pm2.19\) percentage points.

\paragraph{Graph connectivity.} Graph connectivity \cite{du2024learning} is a dataset for graphs with $N$ nodes, consisting of both $N\times N$ binary adjacency and $N\times N$ binary connectivity matrices. The latter describes the existence of a path between nodes. Conditioned on adjacency matrices, we train and evaluate models to generate connectivity matrices. We define a connectivity matrix as valid if it matches the true connectivity matrix. We do not use any random number of clues, and following prior work \cite{du2024learning}, we only train with at most $N=12$ nodes, and evaluate on $N=12$ and $N=18$. For training and evaluation, we use their train and test dataset generators (seeds are different), which continually provide samples, meaning that there is no fixed dataset size.

\paragraph{MNIST Sudoku.} We use the MNIST-Sudoku dataset of \cite{wewer2025spatial}, where each \(9\times9\) puzzle is rendered as a \(252\times252\) grayscale image. Every cell contains a \(28\times28\) MNIST \cite{lecun1998mnist} digit image sampled from a random instance of the corresponding digit class, requiring the model to jointly perform digit recognition and Sudoku constraint reasoning directly from raw pixels. Following the original split convention, we reserve the last 1{,}000 puzzles for evaluation. Inputs are scaled to \([-1,1]\) and treated as continuous-valued images without tokenization; the diffusion model predicts the full \(252\times252\) board. This experiment is separate from the mini MNIST-Sudoku models trained with \(8\times8\) digits in Appendix~\ref{app:other_representations}.

\paragraph{Latin square.} Latin squares of size $N \times N$ are arrays with values taken among $N$ different symbols. Arrays are valid if each symbol occurs exactly once per row and once per column. In our case, we use $N=7$ and randomized backtracking to create $210000$ valid Latin squares, where $80\%$ is for the training set. We train with a number of clues uniformly drawn from $\{0, \dots, N^2-1\}$ and evaluate in random in-painting and unconditional generation. 

\paragraph{\(N\)-Queens.}
An \(N\)-Queens board contains \(N\) queens, with no two sharing a row, column, or diagonal. We represent each board as a vector in \(\{1,\ldots,N\}^N\), whose \(i\)-th entry gives the queen's column in row \(i\). We use \(N=14\) and construct the dataset by sampling uniformly from the complete set of \(365{,}596\) valid boards. The training split contains \(168{,}000\) distinct boards. During training, the number of revealed queens is sampled uniformly from \(\{0,\ldots,N-1\}\). We evaluate both completion from randomly revealed queens and unconditional generation.

\subsection{Continuous representations from discrete configurations}\label{app:one_hot}

From discrete configurations of $m$ cells, we convert each token into their one-hot representation \cite{chen2022analog}, giving $m \times d$ tensors where the last dimension is the one-hot dimension. Please refer to Table \ref{tab:dataset_summary} for their values.

\section{Model architectures and hyperparameters}

\subsection{Architecture for Sudoku, N-Queens, and Latin}\label{app:main_architecture}

For all tasks except MNIST Sudoku \cite{wewer2025spatial} and Graph Connectivity (GC) \cite{du2024learning} (see Appendix \ref{app:mnist_sudoku_gc_architecture}), our network $f_\theta$ is a Transformer \cite{vaswani2017attention}, taking as input a continuous sample $\xt$ and diffusion time $t$:

\begin{itemize}
    \item The sample $\xt$ is linearly embedded in $\mathbb{R}^{d_\mathrm{emb}}$ and added to its learnable positional embeddings. For Sudoku, we use three additive learnable positional embeddings (row, column, and block) instead of one.
    \item The diffusion time $t$ is embedded using
Fourier features $[\sin(2\pi\omega_k \frac{t}{T}), \cos(2\pi\omega_k \frac{t}{T})]_{k=1}^{d_\mathrm{time}/2}$ with log-spaced frequencies $\omega_k$, followed by a two-layer MLP with SiLU activations mapping the embedding into $\mathbb{R}^{d_\mathrm{emb}}$.
    \item We then add them together and feed them through a sequence of $L$ Transformer blocks, followed by a layer normalization and linear layer.
\end{itemize}
We use pre-norm Transformer blocks \cite{radford_gpt2} with \emph{full} self-attention \cite{vaswani2017attention} and per-position MLP sublayers, each wrapped with residual connections. Within each Transformer block, we apply dropout after each attention and MLP sublayer, as well as attention weight dropout inside the multi-head attention.

Except for the number of continuous tokens (i.e., number of cells $m$) and vocabulary size (i.e., one-hot dimension $d$) (see Table \ref{tab:dataset_summary}), all hyper-parameters are shared across tasks described in this subsection. Please refer to Table \ref{tab:architecture_details} for their values and to Algorithm \ref{alg:self_correction} for the training procedure.

\begin{table}[h]
  \caption{Architecture details}
  \label{tab:architecture_details}
  \centering
  \begin{tabular}{lll}
    \toprule
    Name    & Short-hand & Value \\
    \midrule
    Embedding dimension         & ${d_\mathrm{emb}}$    & $128$         \\
    Time embedding dimension    & $d_\mathrm{time}$     & $64$          \\
    Time embedding MLP          & --                    & $[d_\mathrm{time}, {d_\mathrm{emb}}, {d_\mathrm{emb}}]$          \\
    Depth                       & $L$                   & $4$           \\
    Attention heads             & $n_\mathrm{heads}$    & $8$           \\
    Dropout probability         & $p_\mathrm{drop}$     & $0.01$        \\
    MLP hidden dimension& $d_\mathrm{mlp}$      & $4 \times {d_\mathrm{emb}} = 512$, GeLU        \\
    \bottomrule
  \end{tabular}
\end{table}

\subsection{Architectures for Graph Connectivity and MNIST Sudoku}\label{app:mnist_sudoku_gc_architecture}

For Graph Connectivity, we adopt the Neural Logic Machine-based \cite{dong2018neural} architecture of IRED \cite[Table 11]{du2024learning} and use the same hyperparameters as them.

As for MNIST Sudoku, we use the pre-trained SRM model \cite{wewer2025spatial} without any additional training, and therefore have no architecture hyperparameters to report.

\section{Hyper-parameters and hardware}\label{app:hyperparam_hardware}

We show in the following the list of hyper-parameters and hardware used. We use $t^*$ to denote the soft-to-hard conditioning threshold in conditional generation tasks, and the rest are described in previous sections or are self-explanatory.

\begin{table}[h]
  \caption{\textbf{Training and evaluation hyperparameters.} All main VP-continuous diffusion tasks share optimizer type, learning rate, and step budget; batch size is the only task-specific training knob,
  and sampler step budgets differ only for Sudoku-Extreme and SRM. }
  \label{tab:hparams}
  \centering
  \begin{tabular}{ll}
    \toprule
    Hyperparameter & Value \\
    \midrule
    \multicolumn{2}{l}{\textbf{Training}} \\
    \addlinespace[2pt]
    \multicolumn{2}{c}{\textit{Training loop}} \\
    \addlinespace[1pt]
    Optimizer steps                          & $2 \times 10^{6}$ \\
    Batch size (default)                     & $256$ \\
    Batch size (GC)                    & $64$ \\
    Mixed precision                          & AMP, float16 \\
    Hardware                                 & NVIDIA RTX 3090 \\
    \addlinespace[2pt]
    \multicolumn{2}{c}{\textit{Optimization}} \\
    \addlinespace[1pt]
    Optimizer                                & Adam \\
    Learning rate                            & $1 \times 10^{-3}$ \\
    \addlinespace[2pt]
    \multicolumn{2}{c}{\textit{Diffusion}} \\
    \addlinespace[1pt]
    Model parameterization                   & $x$-prediction \\
    Loss target                              & $x$-loss \\
    Noise schedules                          & VP-cosine \\
    \addlinespace[2pt]
    \multicolumn{2}{c}{\textit{Self-correction loss}} \\
    \addlinespace[1pt]
    $\lambda_\mathrm{simple}$ (default)      & $0.1$ \\
    \midrule
    \multicolumn{2}{l}{\textbf{Evaluation}} \\
    \addlinespace[2pt]
    \multicolumn{2}{c}{\textit{Diffusion sampling}} \\
    \addlinespace[1pt]
    Sampler steps (default)                  & $200$ \\
    Sampler steps (Sudoku-Extreme)           & $1000$ \\
    Sampler steps (SRM)                      & $500$ \\
    \bottomrule
  \end{tabular}
\end{table}

\subsection{Discrete-diffusion reference}
\label{app:d3pm}

We additionally train absorbing-state discrete-diffusion models~\cite{austin2021structured} as external references for Table~\ref{tab:main_results}. For Sudoku, \(N\)-Queens, and Latin squares, we use a four-layer Transformer of width \(128\), approximately matching the size of our continuous model. For graph connectivity, we use the same NLM backbone, replacing the reachability input by a categorical \(0/1/\texttt{[MASK]}\) representation. Training randomly masks target tokens and minimizes clean-token cross-entropy at masked positions; conditioning clues (or the graph adjacency matrix) remain fixed. 

We evaluate three inference procedures: ancestral unmasking, confidence-based unmasking, and remasking, which adds four refinement sweeps that re-predict the lowest-confidence \(20\%\) of non-clue tokens. We use \(256\) sampling steps for Sudoku, \(N\)-Queens, and Latin squares, \(1{,}000\) for Sudoku-Extreme, and \(64\) for graph connectivity. Checkpoints are selected on held-out validation data and evaluated on test data. 

\subsection{Experiments compute resources}\label{app:compute}

Training was done on a single RTX 3090 with 24 GiB of VRAM with a single worker.
Each training run took approximately 17 hours on a Sudoku task, 44 hours on a GC task, 14 hours on a Latin task, 10 hours on a N-Queens task.

Standard paper evaluations complete within approximately 5–15 minutes per checkpoint across all samplers and regimes. Cross-validation experiments for graph connectivity are more computationally intensive, usually requiring around 12–25 minutes for 5-fold evaluation with EM and EM decay samplers. The most expensive setting is Sudoku Extreme pass@10 evaluation, which can take approximately 30–50 minutes per sampler due to the large number of samples and denoising steps.

\subsection{Sensitivity of EM decay to sampler hyperparameters}
\label{app:em_decay_sensitivity}

We examine the dependence of EM decay on its initial noise scale
\(\sigma\) and the reverse-progress value \(\tau_{\mathrm{start}}\)
at which the noise begins to decay, with \(\tau=0\) at noise and
\(\tau=1\) at data. The following values are validation results
averaged across training seeds and validation folds.

On Sudoku-Extreme, the baseline model achieves \(0.177\) validity
with \((\sigma,\tau_{\mathrm{start}})=(14,0.8)\), but only \(0.010\)
when \(\sigma\) is increased to \(20\) at the same decay time.
For the self-corrected model, validity is \(0.531\) at
\((20,0.7)\), but falls to \(0.138\) when decay begins at \(0.8\)
instead. Thus, increasing the noise scale or delaying its decay
does not consistently improve validity. EM decay requires careful
validation-based selection of these parameters, whereas Tweedie
reprojection has neither of these two sampler parameters.

\begin{table}[t]
\centering
\caption{\textbf{Best settings in the averaged EM-decay validation sweeps.}
Scores are averaged across training seeds and validation folds.
These are validation diagnostics, not test results.}
\label{tab:em_decay_validation_summary}
\small
\begin{tabular}{llccc}
\toprule
Regime & Training & \(\sigma\)
& \(\tau_{\mathrm{start}}\) & Validation validity \\
\midrule
Sudoku-Extreme & Baseline        & 14 & 0.8 & 0.177 \\
Sudoku-Extreme & Self-correction & 20 & 0.7 & 0.531 \\
Sudoku 21 clues & Baseline        & 7 & 0.7 & 0.843 \\
Sudoku 21 clues & Self-correction & 7 & 0.7 & 0.982 \\
\bottomrule
\end{tabular}
\end{table}

\section{Existing assets and licenses}
We use existing benchmark datasets, checkpoints, and reference implementations only for research evaluation. 
Table~\ref{tab:asset_licenses} summarizes the license information available to us. 
When a dataset license is not specified separately, we report the license of the associated code or release.

\begin{table}[h]
\centering
\small
\caption{Existing assets used in this work.}
\begin{tabular}{ll}
\toprule
Asset & License / terms \\
\midrule
HRM~\cite{wang2025hierarchical}  / Sudoku-Extreme / Maze-Hard
& Apache-2.0 code release; datasets/checkpoints via HRM \\
IRED~\cite{du2024learning}  / Graph Connectivity
& MIT code release \\
SRM~\cite{wewer2025spatial} / MNIST-Sudoku 
& MIT code release; datasets/checkpoints via SRM \\
\bottomrule
\end{tabular}
\label{tab:asset_licenses}
\end{table}

\section{Maze experiments}

Maze-Hard \cite{jolicoeur2025less, wang2025hierarchical} is a dataset consisting of $30\times30$ mazes. Each completed maze is an array filled with values corresponding to a wall, a corridor, the start, the goal or a shortest-path cell. During training and inference, walls, start, and goal are fixed and the model predicts only the path. We define a grid as \textit{valid} if the predicted path connects the start and goal cell; \textit{length} (see Table \ref{tab:maze_results}) means that the path is valid and has the shortest possible length.

Maze results are shown separately in Table~\ref{tab:maze_results}: we use the same architecture as our other tasks (see  \ref{app:main_architecture}), models are trained with $\lambda_{\mathrm{simple}}=0.5$, batch size $64$, and are evaluated on conditional path completion. The same qualitative pattern holds, with modified samplers outperforming DDPM.

\begin{table}[h]
\centering
\caption{Single-run Maze results, reporting path validity (\textit{valid}) and shortest-path recovery (\textit{length}).}
\label{tab:maze_results}
\setlength{\tabcolsep}{3.8pt}
\renewcommand{\arraystretch}{1.15}
\begin{tabular}{lcccc}
\toprule
\multirow{2}{*}{Sampler}
& \multicolumn{2}{c}{pass@1}
& \multicolumn{2}{c}{pass@10} \\
\cmidrule(lr){2-3}\cmidrule(lr){4-5}
& valid & length & valid & length \\
\midrule
\multicolumn{5}{c}{\textbf{Baseline} \textnormal{(no self-correction loss)}} \\
\addlinespace[1pt]
DDPM & .131 & .026 & .448 & .103 \\
\addlinespace[2pt]
\cdashline{1-5}
\addlinespace[1pt]
EM decay & .315 & .108 & .658 & .359 \\
Tweedie reprojection (ours) & .408 & .205 & .621 & .371 \\
\midrule
\multicolumn{5}{c}{\textbf{Self-correction loss} \textnormal{(with $\lambda_\mathrm{simple}=0.5$)}} \\
\addlinespace[1pt]
DDPM & .071 & .032 & .436 & .239 \\
\addlinespace[2pt]
\cdashline{1-5}
\addlinespace[1pt]
EM decay & .842 & .558 & \best{.991} & .900 \\
Tweedie reprojection (ours) & \best{.905} & \best{.661} & .990 & \best{.911} \\
\bottomrule
\end{tabular}
\end{table}

\begin{figure}[t]
  \centering
  \includegraphics[width=0.95\linewidth]{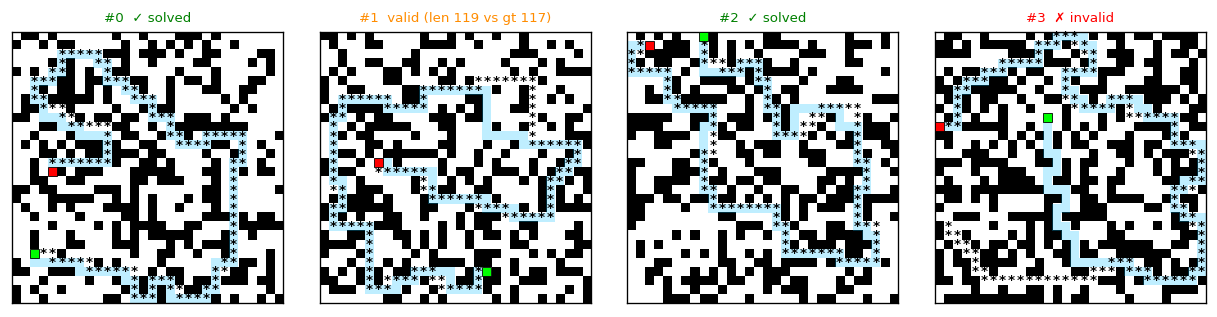}
  \caption{Qualitative samples from our best maze model on the held-out
  \textsc{Maze-Hard} test set. The model is trained with self-correction loss
  ($\lambda_{\text{simple}}=0.5$), and sampled with Tweedie Reprojection at $T=200$. Each panel shows a 30$\times$30 maze with the
  shortest path in \textcolor[HTML]{00B0F0}{transparent cyan} and the predicted
  path as black~$\boldsymbol{\ast}$ markers.}
  \label{fig:maze-samples}
\end{figure}

\section{DDPM derivations}
\label{app:ddpm_derivation}

Detailed derivation of the DDPM formulas from the main part. This appendix is fully discrete. To avoid confusion with the continuous-time notation in Sec.~\ref{sec:ct}, we use \(n\) for the discrete grid index. The main text formulas are recovered by setting \(t=\ttonoisef{n}\), so that \(\xinext=\xt\) and \(\xinow=\xf{\ttodataf{t}}\).

\subsection{Setup}

We use a time grid
\[
\timenoise=t_{\istart}>t_{\isecond}>\cdots> t_{\iend}=\timedata,
\]
and write $\xinow \doteq x_{t_n}$ to simplify notation.
In this section, our convention is:
\[
\boxed{\xistart\ \text{is noise}, \qquad \xiend\ \text{is data}.}
\]
So decreasing $n$ (decreasing $t_n$) means {moving from noise to data}.
Conversely, moving from $n$ to $\ttonoisef{n}$ means {adding noise}.

We assume a prescribed conditional marginal path:
\begin{equation}
\boxed{
q(\xinow \mid \xiend)
=\mathcal{N}\!\Big(\alphaf{\inow}\,\xiend,\; (1-\alphaf{\inow}^2)\,I\Big),
\qquad \alphaf{\istart}=0,\ \alphaf{\iend}=1.
}
\label{eq:app_ddpm_marginals}
\end{equation}

\subsection{Noise-adding Markov step} 

We can write the noise-adding Gaussian Markov kernel as
\begin{equation}
\boxed{
q(\xinext\mid \xinow)
=\mathcal{N}\!\Big(\sqrt{\diffusioncoeff{\inext}}\,\xinow,\ (1-\diffusioncoeff{\inext})\,I\Big),
\qquad n=\iend,\dots,\isecond.
}
\label{eq:app_ddpm_markov_noise_add}
\end{equation}
equivalently the reparameterization
\begin{equation}
\xinext=\sqrt{\diffusioncoeff{\inext}}\,\xinow+\sqrt{1-\diffusioncoeff{\inext}}\,\epsilon_{\inext},
\qquad \epsilon_{\inext}\sim\mathcal{N}(0,I)\ \text{i.i.d.}
\label{eq:app_ddpm_one_step}
\end{equation}

\subsection{Choosing coefficient so the chain matches the marginals} 

We take the conditional expectation of \eqref{eq:app_ddpm_one_step} given $\xiend$:
\[
\mathbb{E}[\xinext \mid \xiend]
=\sqrt{\diffusioncoeff{\inext}}\,\mathbb{E}[\xinow\mid \xiend].
\]
Using the marginal mean from \eqref{eq:app_ddpm_marginals}, $\mathbb{E}[\xf{k}\mid \xiend]=\alphaf{k} \xiend$, we obtain
\[
\alphaf{\ttonoisef{n}}\xiend = \sqrt{\diffusioncoeff{\inext}}\,\alphaf{\inow} \xiend
\quad\Longrightarrow\quad
\boxed{\sqrt{\diffusioncoeff{\inext}}=\frac{\alphaf{\ttonoisef{n}}}{\alphaf{\inow}}}
\]
and hence
\begin{equation}
\boxed{
\diffusioncoeff{\inext}=\frac{\alphaf{\ttonoisef{n}}^2}{\alphaf{\inow}^2}
=\frac{\abarf{\ttonoisef{n}}}{\abarf{\inow}},
\qquad \abarf{\inow}\doteq\alphaf{\inow}^2.
}
\label{eq:app_ddpm_delta_def}
\end{equation}

\subsection{Unrolling and variance identity}

Iterating \eqref{eq:app_ddpm_one_step} gives

\begin{align*}
\xilast
&=
\sqrt{\diffusioncoeff{\ilast}}\,\xiend
+
\sqrt{1-\diffusioncoeff{\ilast}}\,\epsilon_{\ilast}
\\[4pt]
\xilastii
&=
\sqrt{\diffusioncoeff{\ilastii}}\,\xilast
+
\sqrt{1-\diffusioncoeff{\ilastii}}\,\epsilon_{\ilastii}
\\
&=
\sqrt{\diffusioncoeff{\ilastii}\diffusioncoeff{\ilast}}\,\xiend
+
\sqrt{\diffusioncoeff{\ilastii}(1-\diffusioncoeff{\ilast})}\,\epsilon_{\ilast}
+
\sqrt{1-\diffusioncoeff{\ilastii}}\,\epsilon_{\ilastii}
\\[4pt]
\xilastiii
&=
\sqrt{\diffusioncoeff{\ilastiii}}\,\xilastii
+
\sqrt{1-\diffusioncoeff{\ilastiii}}\,\epsilon_{\ilastiii}
\\
&=
\sqrt{\diffusioncoeff{\ilastiii}\diffusioncoeff{\ilastii}\diffusioncoeff{\ilast}}\,\xiend
+
\sqrt{\diffusioncoeff{\ilastiii}\diffusioncoeff{\ilastii}(1-\diffusioncoeff{\ilast})}\,\epsilon_{\ilast}
\\
&\quad
+
\sqrt{\diffusioncoeff{\ilastiii}(1-\diffusioncoeff{\ilastii})}\,\epsilon_{\ilastii}
+
\sqrt{1-\diffusioncoeff{\ilastiii}}\,\epsilon_{\ilastiii}
\\[6pt]
\xinow
&= \left(
\prod_{j=\ilast}^{\inow}\sqrt{\diffusioncoeff{j}}
\right)\xiend
+
\sum_{k=\ilast}^{\inow}
\left(
\sqrt{1-\diffusioncoeff{k}}
\prod_{j=\ttonoisef{k}}^{\inow}\sqrt{\diffusioncoeff{j}}
\right)\epsilon_k
\end{align*}

\begin{align}
\xinow
&=\frac{\alphaf{\inow}}{\alphaf{\iend}}\xiend
+
\sum_{k=\ilast}^{\inow}
\left(
\sqrt{1-\diffusioncoeff{k}}
\prod_{j=\ttonoisef{k}}^{\inow}\sqrt{\diffusioncoeff{j}}
\right)\epsilon_k
\label{eq:app_ddpm_unroll}
\end{align}
Since $\alphaf{\iend}=1$, the signal term is $\alphaf{\inow} \xiend$.

Define the accumulated noise term
\[
\eta_n\doteq
\sum_{k=\ilast}^{\inow}
\left(
\sqrt{1-\diffusioncoeff{k}}\prod_{j=\ttonoisef{k}}^{\inow}\sqrt{\diffusioncoeff{j}}
\right)\epsilon_k.
\]
Then $\mathbb{E}[\eta_n\mid \xiend]=0$ and, because the $\epsilon_k$ are independent,
\begin{align}
\mathbb{V}\left[\eta_n\mid \xiend\right]
&=
\sum_{k=\ilast}^{\inow}
\left(
(1-\diffusioncoeff{k})\prod_{j=\ttonoisef{k}}^{\inow}\diffusioncoeff{j}
\right)I.
\label{eq:app_ddpm_var_sum}
\end{align}
This sum telescopes. Using $\diffusioncoeff{k}=\abarf{k}/\abarf{\ttodataf{k}}$ from \eqref{eq:app_ddpm_delta_def},
one checks the identity
\[
(1-\diffusioncoeff{k})\prod_{j=\ttonoisef{k}}^{\inow}\diffusioncoeff{j}
=
\frac{\abarf{n}}{\abarf{k}}-\frac{\abarf{n}}{\abarf{\ttodataf{k}}}
\]
so summing from $k=\ilast$ to $\inow$ gives
\[
\sum_{k=\ilast}^{\inow}
(1-\diffusioncoeff{k})\prod_{j=\ttonoisef{k}}^{\inow}\diffusioncoeff{j}
=
\frac{\abarf{\inow}}{\abarf{\inow}}-\frac{\abarf{\inow}}{\abarf{\iend}}
=
1-\abarf{\inow},
\]
because $\abarf{\iend}=\alphaf{\iend}^2=1$.
Therefore
\begin{equation}
\boxed{
\mathbb{V}(\eta_n\mid \xiend)=(1-\alphaf{\inow}^2)I,
}
\label{eq:app_ddpm_eta_var}
\end{equation}
and hence
\[
\xinow=\alphaf{\inow} \xiend+\eta_n
\quad\Longrightarrow\quad
q(\xinow\mid \xiend)=\mathcal{N}(\alphaf{\inow} \xiend,(1-\alphaf{\inow}^2)I),
\]
which matches the prescribed marginals \eqref{eq:app_ddpm_marginals}.

\subsection{Posterior}

Now we derive the {denoising} conditional used for DDPM-style reverse simulation:
\[
q(\xinow \mid \xinext, \xiend).
\]
Because the noise-adding chain is Markov in the direction $\xinow\to \xinext$,
the joint factorization is
\[
q(\xinext,\xinow\mid \xiend)
=
q(\xinext\mid \xinow, \xiend)\,q(\xinow\mid \xiend)
=
q(\xinow\mid \xinext, \xiend)\,q(\xinext\mid \xiend).
\]
Thus, as a function of $\xinow$,
\begin{equation}
\boxed{
q(\xinow\mid \xinext,\xiend)\ \propto\ q(\xinext\mid \xinow)\,q(\xinow\mid \xiend).
}
\label{eq:app_ddpm_posterior_propto}
\end{equation}
The normalization constant $q(\xinext\mid \xiend)$ does not depend on $\xinow$.

From \eqref{eq:app_ddpm_markov_noise_add},
\begin{equation}
q(\xinext\mid \xinow)
\propto
\exp\!\left(
-\frac{\|\xinext-\sqrt{\diffusioncoeff{\inext}}\,\xinow\|^2}{2(1-\diffusioncoeff{\inext})}
\right).
\label{eq:app_ddpm_term1}
\end{equation}
From the marginal \eqref{eq:app_ddpm_marginals},
\begin{equation}
q(\xinow\mid \xiend)
\propto
\exp\!\left(
-\frac{\|\xinow-\alphaf{\inow} \xiend\|^2}{2(1-\alphaf{\inow}^2)}
\right).
\label{eq:app_ddpm_term2}
\end{equation}

Multiplying \eqref{eq:app_ddpm_term1} and \eqref{eq:app_ddpm_term2} and writing the exponent in quadratic form gives
\[
q(\xinow \mid \xinext, \xiend)
\propto
\exp\!\left(
-\frac{1}{2}\left(A \|\xinow\|^2 - 2 \langle B,\xinow\rangle\right)
\right),
\]
with
\begin{align*}
A
&=
\frac{1}{1-\alphaf{\inow}^2}
+
\frac{\diffusioncoeff{\inext}}{1-\diffusioncoeff{\inext}}, \\[4pt]
B
&=
\frac{\alphaf{\inow}}{1-\alphaf{\inow}^2}\,\xiend
+
\frac{\sqrt{\diffusioncoeff{\inext}}}{1-\diffusioncoeff{\inext}}\,\xinext.
\end{align*}

Hence the posterior variance and mean are
\[
\postvarf{\inext} = A^{-1},
\qquad
\tilde\mu_{\inow\mid\inext} = \postvarf{\inext}\,B .
\]

\begin{equation}
\boxed{
q(\xinow\mid \xinext,\xiend)=\mathcal{N}(\tilde\mu_{\inow\mid\inext},\ \postvarf{\inext} I).
}
\label{eq:app_ddpm_posterior_gauss}
\end{equation}

Then
\begin{align}
\postvarf{\inext}
&=
\left(
\frac{\diffusioncoeff{\inext}}{1-\diffusioncoeff{\inext}}+\frac{1}{1-\alphaf{\inow}^2}
\right)^{-1}
=
\frac{(1-\diffusioncoeff{\inext})(1-\alphaf{\inow}^2)}{1-\diffusioncoeff{\inext}\alphaf{\inow}^2}.
\label{eq:app_ddpm_beta_tilde_mid}
\end{align}
Using $\alphaf{\ttonoisef{n}}^2=\diffusioncoeff{\inext}\alphaf{\inow}^2$ (equivalent to \eqref{eq:app_ddpm_delta_def}), we have
\[
1-\diffusioncoeff{\inext}\alphaf{\inow}^2
=
1-\alphaf{\inext}^2.
\]
Therefore,
\begin{equation}
\boxed{
\postvarf{\inext}
=
\frac{1-\alphaf{\inow}^2}{1-\alphaf{\inext}^2}\,(1-\diffusioncoeff{\inext}).
}
\label{eq:app_ddpm_beta_tilde_final}
\end{equation}

The posterior mean is
\begin{align*}
\tilde\mu_{\inow\mid\inext}
&=
\frac{(1-\alphaf{\inow}^2)(1-\diffusioncoeff{\inext})}{1-\alphaf{\inext}^2}
\left(
\frac{\alphaf{\inow}}{1-\alphaf{\inow}^2}\,\xiend
+
\frac{\sqrt{\diffusioncoeff{\inext}}}{1-\diffusioncoeff{\inext}}\,\xinext
\right) \\[6pt]
&=
\frac{\alphaf{\inow}(1-\diffusioncoeff{\inext})}{1-\alphaf{\inext}^2}\,\xiend
+
\frac{\sqrt{\diffusioncoeff{\inext}}(1-\alphaf{\inow}^2)}{1-\alphaf{\inext}^2}\,\xinext.
\end{align*}

Thus
\[
\boxed{
\tilde\mu_{\inow\mid\inext}
=
\frac{\alphaf{\inow}(1-\diffusioncoeff{\inext})}{1-\alphaf{\inext}^2}\,\xiend
+
\frac{\sqrt{\diffusioncoeff{\inext}}(1-\alphaf{\inow}^2)}{1-\alphaf{\inext}^2}\,\xinext.
}
\]

Next, substitute the estimator of \(\xiend\) obtained from the marginal at step \(\inext\):
\[
\xiend
=
\frac{\xinext-\sqrt{1-\alphaf{\inext}^2}\,\epsilon_{\inext}}{\alphaf{\inext}}.
\]

Using \(\alphaf{\inow}/\alphaf{\inext}=1/\sqrt{\diffusioncoeff{\inext}}\), we obtain
\begin{align*}
\tilde\mu_{\inow\mid\inext}
&=
\frac{1}{\sqrt{\diffusioncoeff{\inext}}}\,\xinext
-
\frac{1-\diffusioncoeff{\inext}}{\sqrt{\diffusioncoeff{\inext}}\sqrt{1-\alphaf{\inext}^2}}\,
\epsilon_{\inext}.
\end{align*}

\[
\boxed{
\tilde\mu_{\inow\mid\inext}
=
\frac{\xinext}{\sqrt{\diffusioncoeff{\inext}}}
-
\frac{1-\diffusioncoeff{\inext}}{\sqrt{\diffusioncoeff{\inext}}\sqrt{1-\alphaf{\inext}^2}}\,
\epsilon_{\inext}
}
\]

For the equidistant time steps $N=T + 1$, set \(t=\ttonoisef{n}\). Then
\[
\xinext=\xt,\qquad
\xinow=\xf{\ttodataf{t}},\qquad
\diffusioncoeff{\inext}=\diffusioncoeff{t},\qquad
\alphaf{\inow}=\alphaf{\ttodataf{t}}.
\]
This recovers the main-text DDPM posterior
\[
q(\xf{\ttodataf{t}}\mid \xt,\xdata)
=
\mathcal N\!\left(
\alphaf{\ttodataf{t}} \xdata \cdot \frac{1-\diffusioncoeff{t}}{1-\alphaf{t}^2}
+
\frac{\sqrt{\diffusioncoeff{t}}(1-\alphaf{\ttodataf{t}}^2)}{1-\alphaf{t}^2}\xt,
\;
\postvarf{t}I
\right),
\]
where
\[
\postvarf{t}
=
(1-\alphaf{\ttodataf{t}}^{2})
\frac{1-\diffusioncoeff{t}}{1-\alphaf{t}^{2}}.
\]

\section{Toy example: DDPM anchoring vs.\ Tweedie reprojection}
\label{app:tweedie_reprojection_toy}

This toy example is only meant to illustrate the difference between a DDPM-style
anchored update and Tweedie reprojection. It is not intended as evidence that
the failures in our discrete tasks are caused by the same mechanism.

In this example the data support is the union of two intervals (see Figure \ref{fig:no_mans_land}),
\[
\mathcal V_{\mathrm{toy}}=[A,B]\cup[-B,-A],
\qquad
A=2,\quad B=3,\quad M=(A+B)/2=2.5 .
\]
We run reverse trajectories from Gaussian noise and count a sample as valid if
the final point lies in $\mathcal V_{\mathrm{toy}}$. We use an ``oracle'' denoiser that predicts \(M\) when \(y>0\) and \(-M\) otherwise.

\begin{wrapfigure}{r}{0.42\textwidth}
    \vspace{-1.0em}
    \centering
    \includegraphics[width=0.40\textwidth]{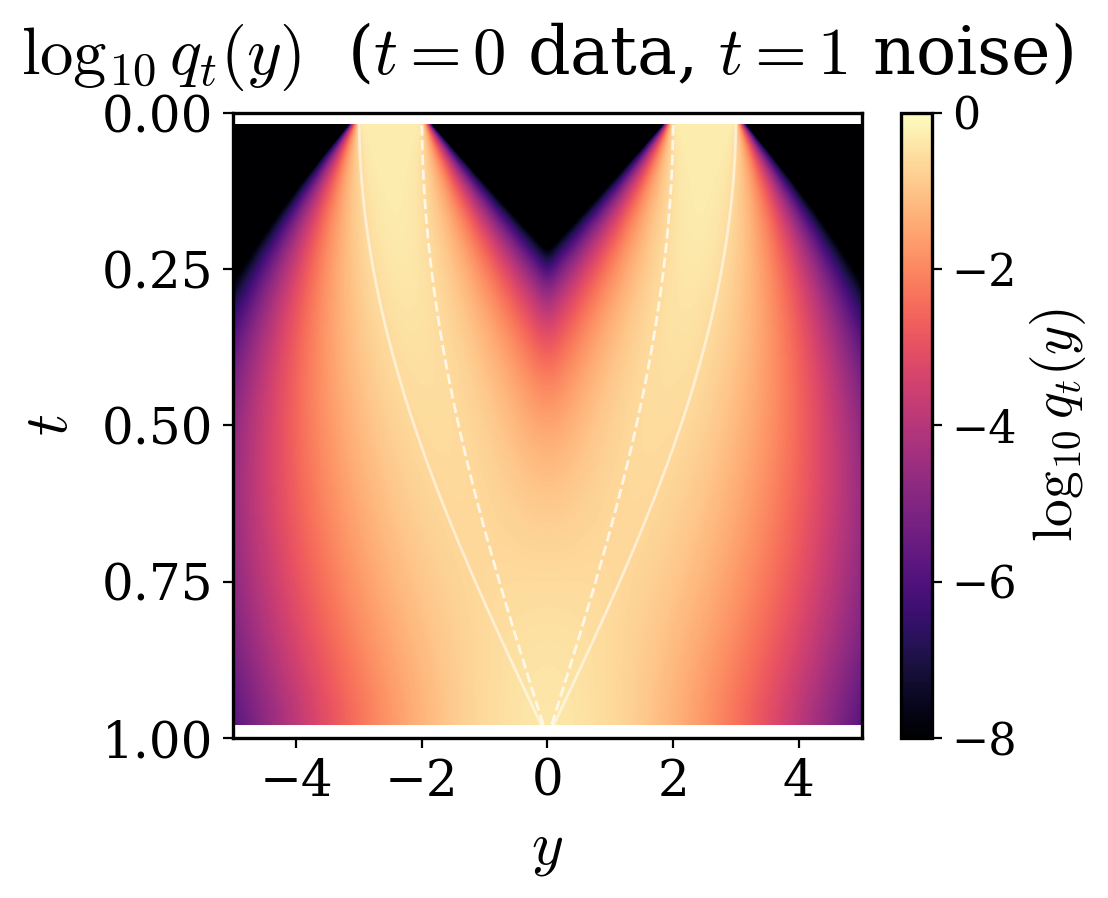}
    \caption{Forward marginal $\log_{10} q_t(y)$ for the toy two-uniform data law
    $\tfrac12 U[A,B]+\tfrac12 U[-B,-A]$. The center is a low-density
    no-man's-land.}
    \label{fig:no_mans_land}
    \vspace{1.0em}

\end{wrapfigure}

Tweedie reprojection updates the mean directly from the
denoiser prediction,
\[
X_s = \alpha_s \hat x_0,
\qquad
\hat x_0=f_\theta(X_t,t),
\]
so the current state \(X_t\) influences the next state only through
\(\hat x_0\). A DDPM-style update also keeps a residual anchor to the current
mean,
\[
X_s = \alpha_s \hat x_0
+
b_t\bigl(X_t-\alpha_t \hat x_0\bigr).
\]
Thus, even when the denoiser predicts a point near the valid set, DDPM can retain
part of the current off-manifold location.

To visualize this effect, we use simple corrupted denoisers whose sign is unreliable (smooth, input-noised, and adversarial sign variants). In this toy,
both signs \(\pm M\) are valid modes. Therefore, a sampler that reprojects
directly through \(\hat x_0\) can remain valid even when the predicted sign is
wrong. In contrast, as shown in Figure \ref{fig:wrong_angle_traj} an anchored DDPM update can stay trapped near the current
state when the denoiser direction is decorrelated from, or anti-aligned with,
\(X_t\). 

\begin{figure}[ht]
    \centering
    \includegraphics[width=\textwidth]{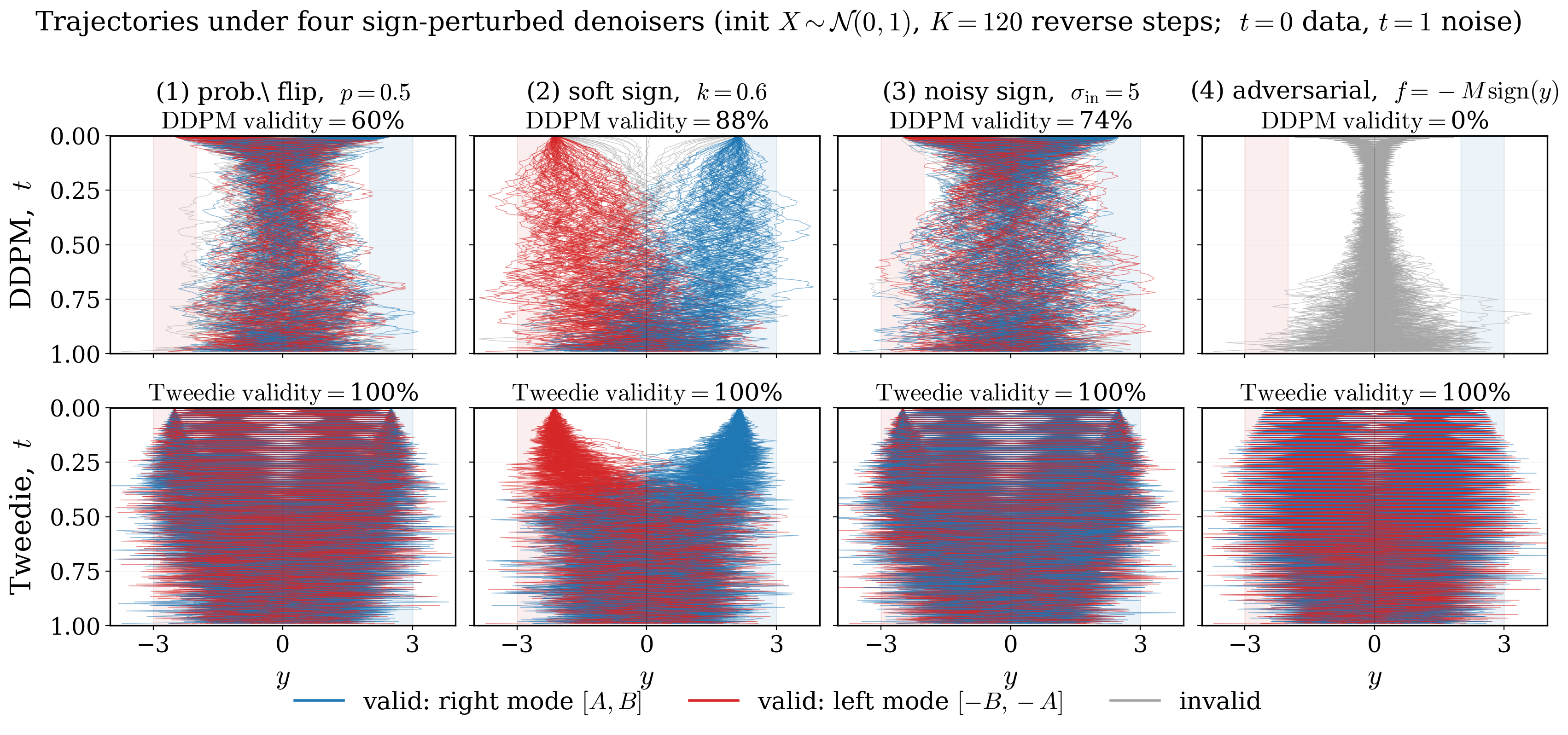}
\caption{Reverse trajectories under sign-perturbed toy denoisers. Tweedie
reprojects through \(\hat x_0=f_\theta(x_t,t)\), whereas DDPM retains an explicit
anchor to \(x_t\). This anchor can keep trajectories in the low-density region
when the denoiser direction is unreliable.}
    \label{fig:wrong_angle_traj}
\end{figure}

To summarize, Tweedie reprojection does not
carry an explicit residual path from \(X_t\) to \(X_s\); DDPM does. Hence, when
\(\hat x_0\) is already valid or close to valid, Tweedie can exploit that
prediction directly, whereas DDPM may still inherit part of the current
off-manifold state.

\section{Mechanistic Analysis of DDPM-Tweedie Sampler Gap}

This appendix gives a mechanistic explanation of the DDPM-Tweedie gap. Decoded validity depends on cellwise argmax decisions: Gaussian perturbations preserve a decoded state when the sampler center has sufficient margin relative to the noise scale. We use this margin view to compare the centers of Tweedie reprojection and DDPM ancestral sampling. Tweedie recenters the next state around the denoiser’s prediction, while DDPM also retains a residual component from the current noisy state. We show that this residual turns out to be harmful for discrete constrained problems.

\subsection{Setup and decoded stability}

We use the convention \(\timedata\) for data and \(\timenoise\) for noise. Let
\[
\alphaf{t}
=
\cos\left(\frac{\pi t}{2 T}\right),
\qquad
\betaf{t}
=
\sin\left(\frac{\pi t}{2 T}\right),
\qquad
\alphaf{t}^2+\betaf{t}^2=1.
\]
We write
\[
\xhatf{0,t}=f_\theta(\xt,t).
\]
For one-hot tasks, let
\[
D:\mathbb R^{m\times d}\to\{1,\dots,d\}^m
\]
be the cellwise argmax decoder, where \(m\) is the number of active/free cells
and \(d\) is the number of symbols per cell. For Sudoku with 21 clues,
\(m=60\); for a 28-clue subset, \(m=53\). Let
\[
\mathcal V\subseteq\{1,\dots,d\}^m
\]
be the set of valid decoded configurations and
\[
\mathcal I=\{1,\dots,d\}^m\setminus\mathcal V
\]
be the set of invalid configurations.

For \(x\in\mathbb R^{m\times d}\), define
\[
\operatorname{gap}(x)
=
\min_i
\left[
x_{i,D_i(x)}
-
\max_{a\neq D_i(x)}x_{i,a}
\right].
\]

\subsection{One-hot representations corrupted with Gaussian noise}

\begin{figure}[t]
\centering
\includegraphics[width=\textwidth]{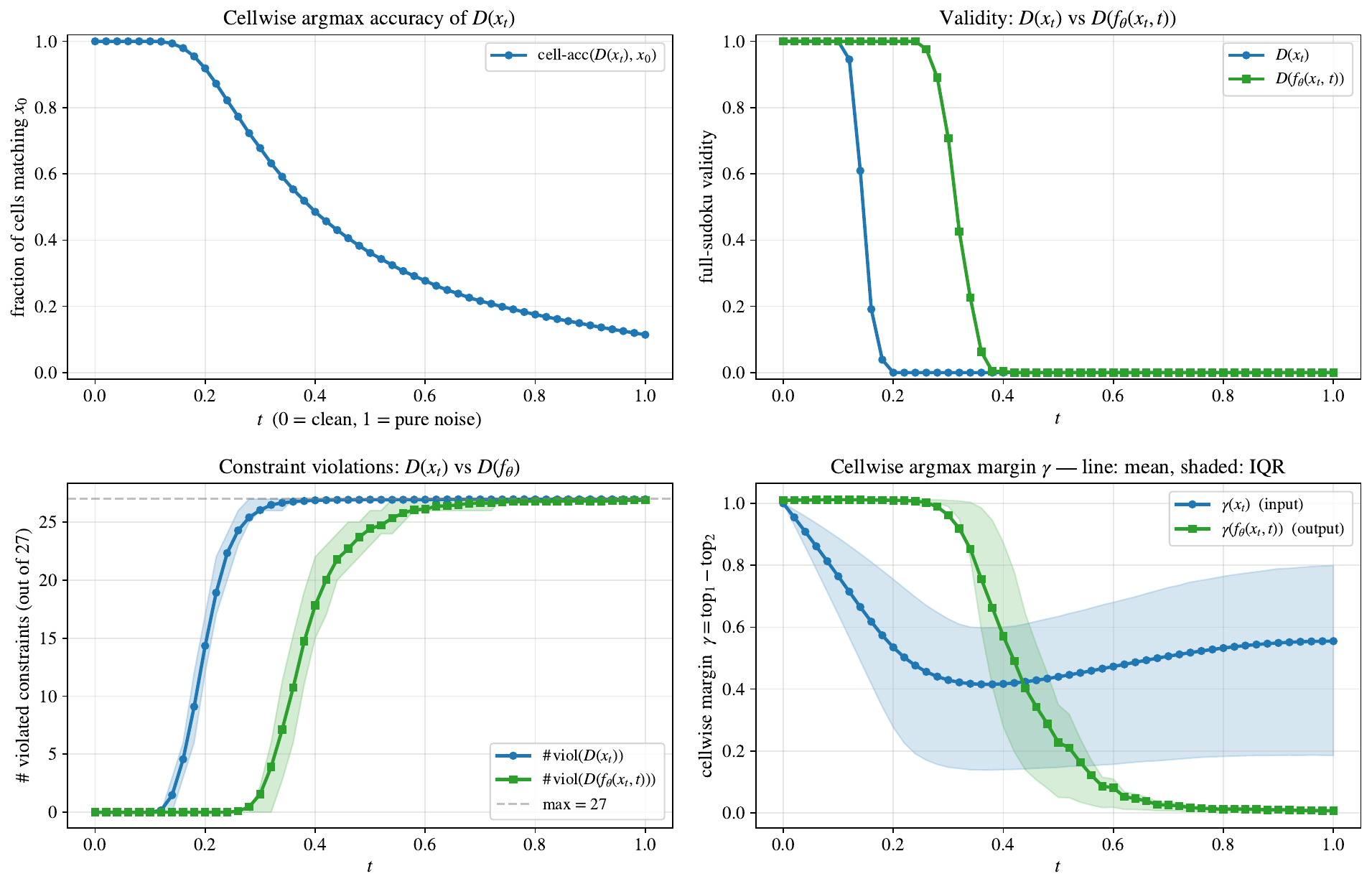}
\caption{\textbf{Forward-noise diagnostics for the baseline model.} We evaluate \(N=64\) Sudoku puzzles with \(n_{\rm rep}=4\) noise draws. Panels show: cellwise argmax accuracy of \(D(\xt)\), full-grid validity of \(D(\xt)\) and \(D(f_\theta(\xt,t))\), number of violated Sudoku constraints, and argmax margin \(\gamma=\mathrm{top}_1-\mathrm{top}_2\). Lines are means; shaded bands show IQR across puzzles and noise draws.}
\label{fig:fig_forward_noise_diagnostics_baseline}
\end{figure}

For one-hot encoded tasks, the decoded output
\[
\xt=c_t+\nu_t\noise
\]
is stable when the Gaussian center \(c_t\) has a sufficiently large cellwise argmax margin relative to the Gaussian noise scale \(\nu_t\). 

Figure~\ref{fig:fig_forward_noise_diagnostics_baseline} shows this effect for forward-noised one-hot Sudoku solutions
\[
\xt=\alphaf{t}\xdata+\betaf{t}\noise.
\]
Cellwise accuracy of the raw decoder \(D(\xt)\) decreases gradually with \(t\) (top left), but full-grid validity collapses much earlier (top right): a single flipped cell is enough to invalidate the decoded Sudoku. The trained denoiser extends the useful noise range. Its decoded output \(D(f_\theta(\xt,t))\) (top right) remains valid for larger \(t\), has fewer constraint violations (bottom left), and maintains a larger argmax margin than the raw noisy input (bottom right).

For any Gaussian sampler the following is true:

\begin{lemma}[Argmax stability]
\label{lem:argmax_stability}
Let \(c\in\mathbb R^{m\times d}\) satisfy
\[
\operatorname{gap}(c)\ge\gamma>0,
\]
and let
\[
\noise\sim\mathcal N(0,\nu^2I_{md}).
\]
Then
\[
\mathbb P(D(c+\noise)\neq D(c))
\le
m(d-1)\Phi\left(-\frac{\gamma}{\sqrt2\,\nu}\right),
\]
where \(\Phi\) is the standard Gaussian CDF.
\end{lemma}

\begin{proof}
Fix a cell \(i\), let \(j=D_i(c)\), and consider a competitor \(a\neq j\). Since
\[
\operatorname{gap}(c)\ge\gamma,
\]
we have
\[
c_{i,j}-c_{i,a}\ge\gamma.
\]
The competitor overtakes \(j\) after adding noise only if
\[
c_{i,a}+\noise_{i,a}\ge c_{i,j}+\noise_{i,j},
\]
or equivalently
\[
\noise_{i,a}-\noise_{i,j}\ge c_{i,j}-c_{i,a}\ge\gamma.
\]
Since
\[
\noise_{i,a}-\noise_{i,j}\sim\mathcal N(0,2\nu^2),
\]
this event has probability at most
\[
\Phi\left(-\frac{\gamma}{\sqrt2\,\nu}\right).
\]
A union bound over all \(m(d-1)\) competitors gives the claim.
\end{proof}

The lemma formalizes the idea that a continuous point with large cellwise margin decodes stably under small Gaussian perturbations.

\subsection{DDPM equals Tweedie plus state memory}
\label{app:subsection:tweedie_ddpm}

\begin{proposition}[DDPM and Tweedie-reprojection relation]
\label{prop:ddpm_anchored_tweedie}
At step \(t\), the DDPM ancestral Gaussian kernel has center
\[
c_t^{\mathrm{DDPM}}
=
c_t^{\mathrm{Tw}}
+
B_t\bigl(\xt-\alphaf{t}\xhatf{0,t}\bigr),
\qquad
c_t^{\mathrm{Tw}}
=
\alphaf{\ttodataf{t}}\xhatf{0,t},
\]
where
\[
B_t
=
\frac{
\sqrt{\diffusioncoeff{t}}\left(1-\alphaf{\ttodataf{t}}^2\right)
}{
1-\alphaf{t}^2
},
\qquad
\alphaf{t}
=
\sqrt{\diffusioncoeff{t}}\alphaf{\ttodataf{t}}.
\]
Thus, relative to Tweedie, DDPM adds a state-anchor residual. The associated
noise scales are
\[
\nu_t^{\mathrm{Tw}}=\betaf{\ttodataf{t}},
\qquad
\nu_t^{\mathrm{DDPM}}=\tau_t,
\qquad
\tau_t^2
=
\frac{
\left(1-\diffusioncoeff{t}\right)
\left(1-\alphaf{\ttodataf{t}}^2\right)
}{
1-\alphaf{t}^2
}.
\]
\end{proposition}

\begin{proof}
Following Appendix \ref{app:ddpm_derivation}, the DDPM ancestral mean can be written as
\[
\mu_t^{\mathrm{DDPM}}
=
A_t\xhatf{0,t}+B_t\xt,
\]
with
\[
A_t
=
\alphaf{\ttodataf{t}}
\frac{1-\diffusioncoeff{t}}{1-\alphaf{t}^2},
\qquad
B_t
=
\frac{
\sqrt{\diffusioncoeff{t}}
\left(1-\alphaf{\ttodataf{t}}^2\right)
}{
1-\alphaf{t}^2
}.
\]
Using
\[
\alphaf{t}
=
\sqrt{\diffusioncoeff{t}}\alphaf{\ttodataf{t}},
\]
we have
\[
A_t+B_t\alphaf{t}
=
\alphaf{\ttodataf{t}}
\frac{
1-\diffusioncoeff{t}
+\diffusioncoeff{t}\left(1-\alphaf{\ttodataf{t}}^2\right)
}{
1-\alphaf{t}^2
}
=
\alphaf{\ttodataf{t}}
\frac{
1-\diffusioncoeff{t}\alphaf{\ttodataf{t}}^2
}{
1-\alphaf{t}^2
}
=
\alphaf{\ttodataf{t}}.
\]
Therefore
\[
\begin{aligned}
\mu_t^{\mathrm{DDPM}}
&=
A_t\xhatf{0,t}+B_t\xt \\
&=
\left(A_t+B_t\alphaf{t}\right)\xhatf{0,t}
+
B_t\left(\xt-\alphaf{t}\xhatf{0,t}\right) \\
&=
\alphaf{\ttodataf{t}}\xhatf{0,t}
+
B_t\left(\xt-\alphaf{t}\xhatf{0,t}\right) \\
&=
c_t^{\mathrm{Tw}}
+
B_t\left(\xt-\alphaf{t}\xhatf{0,t}\right).
\end{aligned}
\]
The Tweedie kernel uses noise scale \(\betaf{\ttodataf{t}}\), while the DDPM ancestral
kernel uses posterior noise scale
\[
\tau_t^2
=
\frac{
\left(1-\diffusioncoeff{t}\right)
\left(1-\alphaf{\ttodataf{t}}^2\right)
}{
1-\alphaf{t}^2
}.
\]
This proves both the center decomposition and the variance difference.
\end{proof}

Proposition~\ref{prop:ddpm_anchored_tweedie} shows that DDPM differs from Tweedie
in two ways. First, it uses a different Gaussian noise scale. Second, its center contains the residual
\(
r_t=\xt-\alphaf{t}\xhatf{0,t}.
\)
If \(\xt\) carries corrupted cellwise preferences, this term is harmful and may even destabilize the trajectory. The next proposition formalizes this failure mode.

\begin{proposition}[State-memory corruption]
\label{prop:state_memory_corruption}
Let
\[
r_t=\xt-\alphaf{t}\xhatf{0,t}.
\]
Assume the denoiser prediction decodes to a valid grid \(v\in\mathcal V\):
\[
D(\xhatf{0,t})=v.
\]
For a cell \(i\) and competitor \(a\neq v_i\), define the denoiser margin
\[
M^\theta_{i,a}
=
\xhatf{0,t,i,v_i}-\xhatf{0,t,i,a},
\]
and the residual anti-margin
\[
M^r_{i,a}
=
r_{t,i,a}-r_{t,i,v_i}.
\]
If for some \(i,a\),
\[
B_tM^r_{i,a}
-
\alphaf{\ttodataf{t}}M^\theta_{i,a}
\ge
\gamma_{\mathrm{bad}}>0,
\]
then the DDPM center prefers the wrong symbol \(a\) over the valid symbol \(v_i\)
in cell \(i\) with margin at least \(\gamma_{\mathrm{bad}}\):
\[
c^{\mathrm{DDPM}}_{t,i,a}
-
c^{\mathrm{DDPM}}_{t,i,v_i}
\ge
\gamma_{\mathrm{bad}}.
\]
If every valid completion has cell \(i\) equal to \(v_i\), then
\[
D(c_t^{\mathrm{DDPM}})\notin\mathcal V.
\]
Moreover,
\[
\mathbb P\left(
D(\xf{\ttodataf{t}}^{\mathrm{DDPM}})\in\mathcal V
\mid \xt
\right)
\le
\Phi\left(
-\frac{\gamma_{\mathrm{bad}}}{\sqrt2\,\tau_t}
\right).
\]
\end{proposition}

\begin{proof}
The DDPM center is
\[
c_t^{\mathrm{DDPM}}
=
\alphaf{\ttodataf{t}}\xhatf{0,t}+B_tr_t.
\]
Therefore
\[
\begin{aligned}
c^{\mathrm{DDPM}}_{t,i,a}
-
c^{\mathrm{DDPM}}_{t,i,v_i}
&=
\alphaf{\ttodataf{t}}
\left(
\xhatf{0,t,i,a}-\xhatf{0,t,i,v_i}
\right)
+
B_t\left(
r_{t,i,a}-r_{t,i,v_i}
\right)\\
&=
-\alphaf{\ttodataf{t}}M^\theta_{i,a}
+
B_tM^r_{i,a}.
\end{aligned}
\]
By assumption, this is at least \(\gamma_{\mathrm{bad}}\).

If all valid completions have symbol \(v_i\) in cell \(i\), then any decoded grid
with a different symbol in cell \(i\) is invalid. For the noisy DDPM step to
become valid, \(v_i\) must overtake \(a\). This requires
\[
\tau_t\noise_{i,v_i}-\tau_t\noise_{i,a}
\ge
\gamma_{\mathrm{bad}}.
\]
The left-hand side is Gaussian with variance \(2\tau_t^2\), giving
\[
\mathbb P(\text{repair})
\le
\Phi\left(
-\frac{\gamma_{\mathrm{bad}}}{\sqrt2\,\tau_t}
\right).
\]
\end{proof}

The proposition suggests that DDPM kernel can be centered on the wrong decoded symbol whenever the state residual exceeds the denoiser's correction:
\[
B_t
\left(r_{t,i,a}-r_{t,i,v_i}\right)
>
\alphaf{\ttodataf{t}}
\left(
\xhatf{0,t,i,v_i}
-
\xhatf{0,t,i,a}
\right).
\]
In that case, sampling more precisely around the DDPM center does not help: the sampler is concentrated around the wrong decoded state. We next visualize this
effect along actual reverse trajectories.

\begin{figure}[t]
\centering
\includegraphics[width=\textwidth]{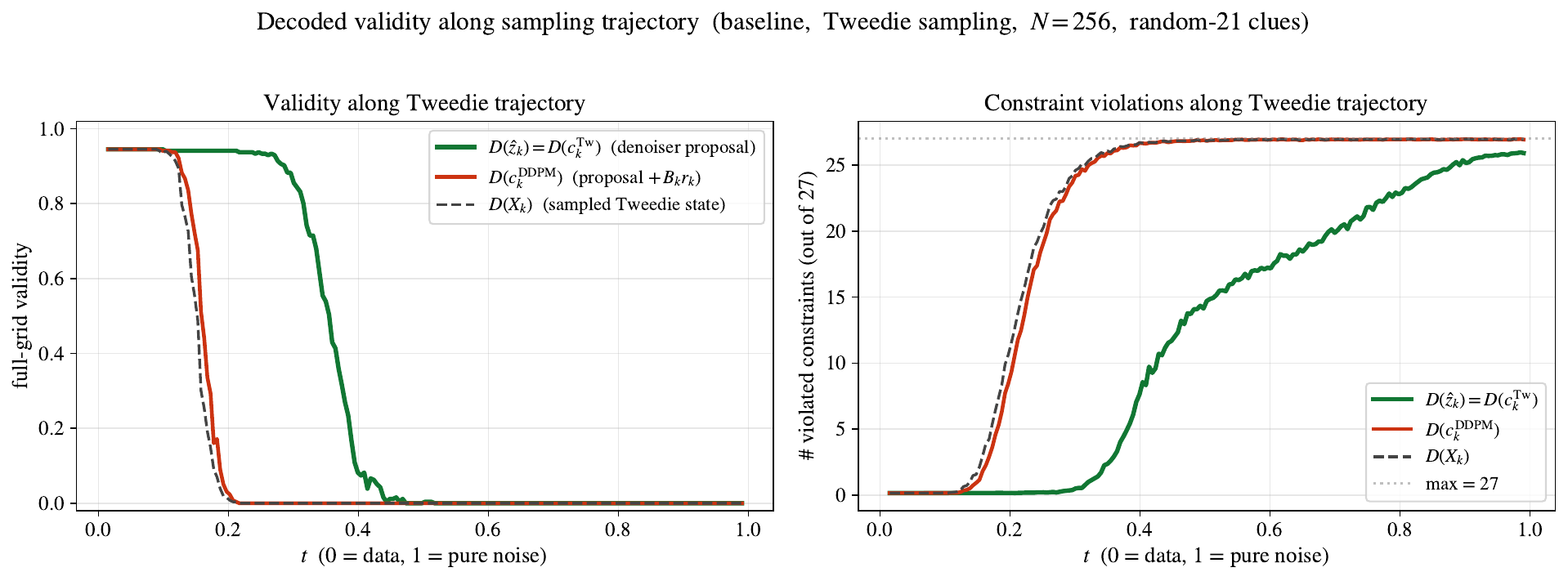}
\caption{\textbf{Decoded validity along a Tweedie sampling trajectory.}
For each reverse step, we decode the denoiser proposal \(\xhatf{0,t}\), the Tweedie
center \(c_t^{\mathrm{Tw}}=\alphaf{\ttodataf{t}}\xhatf{0,t}\), and the counterfactual DDPM
center \(c_t^{\mathrm{DDPM}}=c_t^{\mathrm{Tw}}+B_t r_t\). Left: full-grid Sudoku
validity. Right: number of violated constraints. The Tweedie center shares the
proposal's argmax, while the DDPM residual anchor delays validity.}
\label{fig:fig_latch_window_baseline_tw_traj}
\end{figure}

\begin{figure}[t]
\centering
\includegraphics[width=\textwidth]{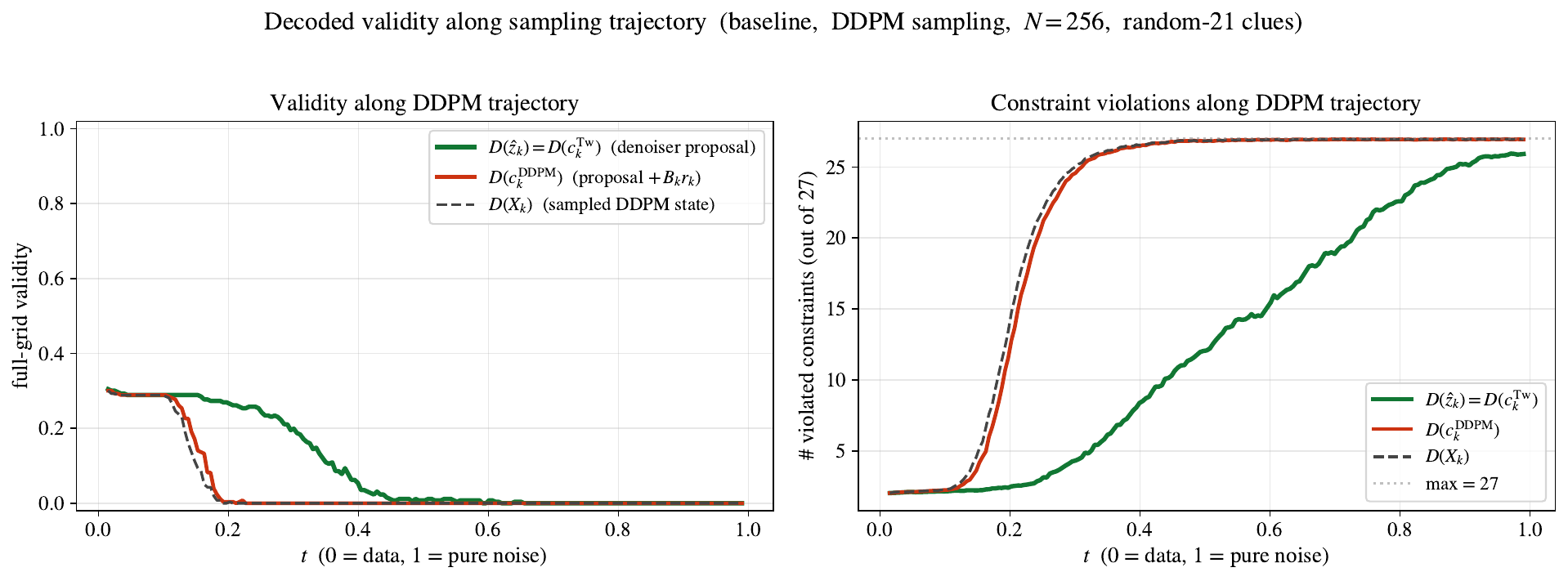}
\caption{\textbf{Decoded validity along a DDPM sampling trajectory.}
Same diagnostics as Fig.~\ref{fig:fig_latch_window_baseline_tw_traj}, but the
trajectory itself is generated by DDPM ancestral sampling. The number of violated
constraints decreases along the trajectory, showing that the sampler moves toward Sudoku structure. However, full-grid validity remains low because a small number of
persistent errors is enough to invalidate the decoded grid. This illustrates the state-anchor failure mode: DDPM can approach a near-valid solution while still
failing to make the needed cellwise corrections required for exact validity.}
\label{fig:fig_latch_window_baseline_ddpm_traj}
\end{figure}

Figures~\ref{fig:fig_latch_window_baseline_tw_traj}-\ref{fig:fig_latch_window_baseline_ddpm_traj}
show what happens during sampling. Along the Tweedie trajectory, the model begins to predict valid Sudoku grids before the end of the reverse process. The important point is that Tweedie does not anchor the next center to the current decoded grid.
Instead, it recenters directly at
\(
c_t^{\mathrm{Tw}}
=
\alphaf{\ttodataf{t}}\xhatf{0,t}.
\)
Since multiplication by the positive scalar \(\alphaf{\ttodataf{t}}\) does not change the argmax, the Tweedie center has the same decoded grid as the model prediction.
Therefore, if the current sample is near-valid but has a few wrong cells, and the denoiser predicts a valid completion, Tweedie can move directly to that valid
completion.

DDPM behaves differently. Its center also contains the residual term \(B_t r_t\),
which keeps part of the current state. This can be helpful when the current state is
already correct, but it can be harmful when the current state is near-valid with a
few persistent mistakes. In that case, the DDPM update averages the denoiser
correction with the current wrong cell preferences, so the sampler can reduce the
number of constraint violations while still failing to make the final argmax changes
needed for full validity.

Figure~\ref{fig:analysis_noise} shows that adding more exploration
does not fix the baseline DDPM sampler. Both samplers degrade when the noise is too
small or too large, but the best DDPM setting still remains far below the best
Tweedie setting. 

To check whether the residual term helps or hurts in practice, we compare the
decoded Tweedie and DDPM centers at each reverse step. Let \(E_+\) be the event that
the Tweedie center is valid but the DDPM center is invalid, and let \(E_-\) be the
opposite event. If the residual anchor were often useful for Sudoku validity, then
\(E_-\) should occur frequently. Instead, Figure~\ref{fig:analysis_eplus_eminus} shows that
\(E_+\) dominates.

\subsection{Sampler-induced states degrade denoiser proposals}
\label{sec:sampler_induced_proposal_mismatch}

The previous section showed that the DDPM anchor can corrupt a valid denoiser proposal. Figure~\ref{fig:fig_latch_window_baseline_ddpm_traj} shows a second effect: along DDPM trajectories, the baseline model often fails to produce valid
proposals at all. We now ask why this failure mode persists along full DDPM trajectories.

Let \(q_t\) denote the forward training distribution
\[
\xt=\alphaf{t}\xdata+\betaf{t}\noise,
\qquad
\xdata\in\mathcal V,\quad \noise\sim\mathcal N(0,I).
\]

\begin{lemma}[Forward noising of discrete grids]
\label{lem:forward_valid_tubes}
Let \(\mathcal V\subset\mathbb R^{m d}\) be the set of valid one-hot clean grids. Here \(B(a,r)=\{y\in\mathbb R^{m d}:\|y-a\|_2\le r\}\) denotes the closed Euclidean ball
of radius \(r\) centered at \(a\).
For
\(u>0\), define the typical forward region
\[
\mathcal R_t(u)
=
\bigcup_{v\in\mathcal V}
B(\alphaf{t}v,\betaf{t}u).
\]
If \(\xt\sim q_t\), then
\[
q_t(\mathcal R_t(u))
\ge
1-\mathbb P(\|G\|_2>u),
\qquad
G\sim\mathcal N(0,I_{md}).
\]
Moreover, for any one-hot grid \(z\), the ball \(B(\alphaf{t}z,\betaf{t}u)\) is
disjoint from the forward ball around a valid grid \(v\) whenever
\[
\|z-v\|_2>\frac{2u}{\mathrm{SNR}_t},
\qquad
\mathrm{SNR}_t=\frac{\alphaf{t}}{\betaf{t}}.
\]
For one-hot grids, this corresponds to the Hamming-depth condition
\[
d_{\mathrm{Ham}}(z,v)
>
\frac{2u^2}{\mathrm{SNR}_t^2}.
\]
\end{lemma}

\begin{proof}
Since
\[
\xt-\alphaf{t}\xdata=\betaf{t}\noise,
\]
we have
\[
\xt\in B(\alphaf{t}\xdata,\betaf{t}u)
\]
whenever
\[
\|\noise\|_2\le u.
\]
This gives the first claim.

For the second claim, the distance between the two centers is
\[
\|\alphaf{t}z-\alphaf{t}v\|_2
=
\alphaf{t}\|z-v\|_2.
\]
Each ball has radius \(\betaf{t}u\). The two balls are disjoint if
\[
\alphaf{t}\|z-v\|_2>2\betaf{t}u,
\]
which is equivalent to
\[
\|z-v\|_2>\frac{2u}{\mathrm{SNR}_t}.
\]
For one-hot grids,
\[
\|z-v\|_2^2=2d_{\mathrm{Ham}}(z,v),
\]
giving the final condition.
\end{proof}

This explains the exposure gap. Under standard training, the model is trained on
noisy versions of valid grids, so certain invalid regions receive little training
mass. During sampling, however, the model's own predictions can move trajectories
toward noisy versions of invalid or near-valid states. Self-correction training
therefore adds supervision on precisely these sampler-induced inputs.

\subsection{Proof and nearby-proposal extension}
\label{app:local_proposal_explained}

{
\paragraph{Proof of Proposition~\ref{prop:local_self_correction_matching}.}
For a fixed continuous proposal \(\bar x_0\), the exact DDPM posterior kernels are the reverse conditionals associated with the Gaussian marginals \(q_t(\cdot;\bar x_0)\). Hence they map \(q_t(\cdot;\bar x_0)\) to \(q_s(\cdot;\bar x_0)\) for every \(s<t\), which proves the claimed marginal over the local reverse window. The first KL divergence is then zero by construction of \(X_t^{\mathrm{SC}}\); the second follows from the standard KL formula for Gaussians with covariance \(\betaf{t}^2I\) and means \(\alphaf{t}x_0\) and \(\alphaf{t}\bar x_0\).

\paragraph{Extension to nearby proposals.}
Proposition~\ref{prop:local_self_correction_matching} assumes that the proposal of the denoiser remains unchanged which is an idealization. Consider instead reverse times \(t_0>t_1>\cdots>t_K=t>0\), and assume that the proposal at each step remains within radius \(\rho_k\) of \(\bar x_0^\star\) almost surely under the reverse process:
\[
\|\bar x_0^{(k)}-\bar x_0^\star\|_2\leq\rho_k.
\]
Let \(A_k\) be the coefficient multiplying the clean proposal in the DDPM posterior mean at step \(k\), and let \(\widetilde\beta_k I\) be the corresponding posterior covariance
(Eqs.~\eqref{eq:posterior_mean}--\eqref{eq:posterior_var}).
Conditioned on the same current state, replacing \(\bar x_0^\star\) by \(\bar x_0^{(k)}\) changes only the posterior mean, by \(A_k(\bar x_0^{(k)}-\bar x_0^\star)\). The resulting one-step KL divergence is therefore at most
\[
\frac{A_k^2\rho_k^2}{2\widetilde\beta_k}.
\]
We consider a reverse window with positive posterior variances \(\widetilde\beta_k>0\), then the KL chain rule and data processing inequality give:
\begin{equation}
D_{\mathrm{KL}}\!\left(
    \mathcal L(Y_t)
    \,\middle\|\,
    q_t(\cdot;\bar x_0^\star)
\right)
\leq
D_{\mathrm{KL}}\!\left(
    \mathcal L(Y_{t_0})
    \,\middle\|\,
    q_{t_0}(\cdot;\bar x_0^\star)
\right)
+
\sum_{k=1}^{K}
\frac{A_k^2\rho_k^2}{2\widetilde\beta_k}.
\label{eq:nearby_proposal_bound}
\end{equation}
Likewise, at a fixed time \(t\), if the proposal \(\bar x_0\) used to construct the
self-correction input also satisfies
\(\|\bar x_0-\bar x_0^\star\|_2\leq\rho\), then
\begin{equation}
D_{\mathrm{KL}}\!\left(
    \mathcal L(X_t^{\mathrm{SC}})
    \,\middle\|\,
    q_t(\cdot;\bar x_0^\star)
\right)
\leq
\frac{\alphaf{t}^2\rho^2}{2\betaf{t}^2}.
\label{eq:nearby_sc_input_bound}
\end{equation}

When the right-hand sides of Eqs.~\eqref{eq:nearby_proposal_bound} and~\eqref{eq:nearby_sc_input_bound} are small, both the sampler states and the self-correction inputs are close to \(q_t(\cdot;\bar x_0^\star)\). This explains why re-noising the model's proposal can produce training inputs similar to those encountered during this part of sampling, while retaining \(x_0\) as the target.

{
\paragraph{Empirical comparison with sampler states.}
We also compare the self-correction inputs directly with states encountered
during DDPM sampling. At each noise level, we measure the frequency of the
event \(E_+\), where the clean proposal gives a valid Tweedie center but the
corresponding DDPM center is invalid. We compare forward-noised training
inputs \(X_t^{\mathrm{std}}\), self-correction inputs
\(X_t^{\mathrm{SC}}\), and actual DDPM states \(Y_t\).
The respective event frequencies are
\[
(0.319,\;0.116,\;0.030)
\quad\text{on Sudoku-Extreme,}
\]
and
\[
(0.247,\;0.067,\;0.036)
\quad\text{on 21-clue Sudoku.}
\]
Thus, on this diagnostic, the gap between self-correction inputs and DDPM
states is \(3.4\)--\(6.8\times\) smaller than for standard forward-noised
inputs. This does not establish equality of the full distributions, but
supports the interpretation that re-noising the model's own proposal better
reflects the states relevant to its sampling errors.
}

\subsection{Self-correction as train-inference mismatch correction}
\label{sec:self_recovery_theory}

The previous section argued that DDPM can enter model-induced states on which
the baseline denoiser no longer proposes valid grids reliably. Self-correction
addresses this mismatch by changing the supervised input distribution. Instead
of training only on forward-noised valid grids, it also trains on noisy versions
of the model's own intermediate predictions, while keeping the original valid
grid as the regression target.

\begin{proposition}[Self-correction target]
\label{prop:self_correction_population_target}
Fix the stopped-gradient proposal generator \(f_{\bar\theta}\). Draw a valid
sample \(X_0\), form a standard noisy input
\[
Y_{t_1}^{\mathrm{std}}
=
\alphaf{t_1}X_0+\betaf{t_1}\epsilon_1,
\]
and define the stopped-gradient self-correction proposal
\[
\bar X_0
=
\operatorname{sg}\!\left(f_{\bar\theta}(Y_{t_1}^{\mathrm{std}},t_1)\right).
\]
The self-correction input at time \(t\) is
\[
Y_t^{\mathrm{sc}}
=
\alphaf{t}\bar X_0+\betaf{t}\epsilon',
\]
while the standard denoising input is
\[
Y_t^{\mathrm{std}}
=
\alphaf{t}X_0+\betaf{t}\epsilon .
\]
Both losses use the original valid sample \(X_0\) as target:
\[
\mathcal L(f)
=
\mathbb E\!\left[
\|f(Y_t^{\mathrm{sc}},t)-X_0\|^2
\right]
+
\lambda_{\mathrm{std}}
\mathbb E\!\left[
\|f(Y_t^{\mathrm{std}},t)-X_0\|^2
\right].
\]

Let \(p_{\mathrm{sc}}(y,t)\) and \(p_{\mathrm{std}}(y,t)\) be the densities of \(Y_t^{\mathrm{sc}}\) and \(Y_t^{\mathrm{std}}\) and their sampled times, respectively, and define
\[
m_{\mathrm{sc}}(y,t)
=
\mathbb E[X_0\mid Y_t^{\mathrm{sc}}=y],
\qquad
m_{\mathrm{std}}(y,t)
=
\mathbb E[X_0\mid Y_t^{\mathrm{std}}=y].
\]
Then, at any \((y,t)\) such that \[ p_{\mathrm{sc}}(y,t)+\lambda_{\mathrm{std}}p_{\mathrm{std}}(y,t)>0,
\]
the minimizer satisfies
\[
f^\star(y,t)
=
\frac{
p_{\mathrm{sc}}(y,t)m_{\mathrm{sc}}(y,t)
+
\lambda_{\mathrm{std}}p_{\mathrm{std}}(y,t)m_{\mathrm{std}}(y,t)
}{
p_{\mathrm{sc}}(y,t)+\lambda_{\mathrm{std}}p_{\mathrm{std}}(y,t)
}.
\]
In particular, whenever
\[
p_{\mathrm{sc}}(y,t)\gg
\lambda_{\mathrm{std}}p_{\mathrm{std}}(y,t),
\]
the learned target is dominated by the self-correction regression target
\[
m_{\mathrm{sc}}(y,t)
=
\mathbb E[X_0\mid Y_t^{\mathrm{sc}}=y].
\]
\end{proposition}

\begin{proof}
Fix \((y,t)\) and write \(a=f(y,t)\). The terms of the objective that
depend on \(a\) are
\[
p_{\mathrm{sc}}(y,t)
\mathbb E[\|a-X_0\|^2\mid Y_t^{\mathrm{sc}}=y]
+
\lambda_{\mathrm{std}}p_{\mathrm{std}}(y,t)
\mathbb E[\|a-X_0\|^2\mid Y_t^{\mathrm{std}}=y].
\]
Using
\[
\mathbb E[\|a-X_0\|^2\mid Y=y]
=
\|a-\mathbb E[X_0\mid Y=y]\|^2+\mathrm{const}(y),
\]
this is, up to constants independent of \(a\),
\[
p_{\mathrm{sc}}(y,t)\|a-m_{\mathrm{sc}}(y,t)\|^2
+
\lambda_{\mathrm{std}}p_{\mathrm{std}}(y,t)
\|a-m_{\mathrm{std}}(y,t)\|^2.
\]
Setting the gradient with respect to \(a\) to zero gives
\[
p_{\mathrm{sc}}(y,t)(a-m_{\mathrm{sc}}(y,t))
+
\lambda_{\mathrm{std}}p_{\mathrm{std}}(y,t)
(a-m_{\mathrm{std}}(y,t))
=
0.
\]
Solving for \(a\) gives the claimed expression.
\end{proof}

Standard training learns the Bayes denoiser on forward-noised valid grids. Self-correction adds training mass around model-induced predictions \(\xtildef{0}\), which may be invalid or partially wrong, and trains the model to
map noisy versions of those predictions back to the original valid grid. Thus self-correction targets the exposure gap encountered by closed-loop samplers such as DDPM.

\section{Self-Correction Loss Ablations on Sudoku}
\label{app:loss_ablations}

We ablate two design dimensions of the self-correction loss \(\mathcal{L}_{\mathrm{SC}}\) on Sudoku: (i) the regularizer loss weight \(\lambda_{\mathrm{simple}}\); and (ii) alternative formulations of the self-correction training step, including different input constructions, target constructions, and multi-step rollouts. We additionally compare to input perturbations from DDPM-IP \cite{ning2023input} and self-conditioning from Analog Bits \cite{chen2022analog}.

We evaluate on Sudoku puzzles with 21 clues using the full checkpoint grid $ \{20\mathrm{k}, 60\mathrm{k}, 100\mathrm{k}, 200\mathrm{k}, 400\mathrm{k}, 800\mathrm{k}, 1.2\mathrm{M}, 1.6\mathrm{M}, 2.0\mathrm{M}\}.$ For certain runs we stopped the training earlier if loss was unstable. For the loss-weight sweep, we report both samplers: Tweedie reprojection (Tw) and DDPM.

\subsection{Simple loss weight}

\label{app:slw-weight}

\begin{figure}[t]
\centering
\includegraphics[
    width=\textwidth,
    trim={0cm 0cm 0 1.8cm},
    clip
]{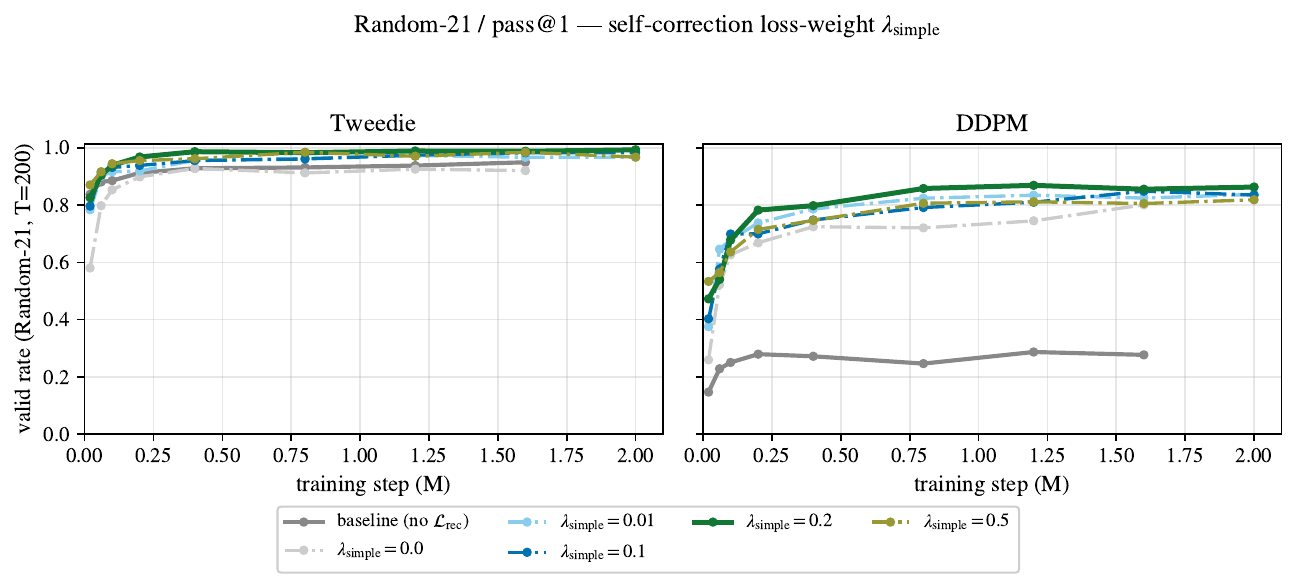}
\caption{\textbf{Self-correction loss-weight sensitivity on Sudoku puzzles with 21 clues.} We sweep \(\lambda_{\mathrm{simple}} \in \{0,0.01,0.1,0.2,0.5\}\) and report checkpoint trajectories under the standard \(T=200\), pass@1 evaluation.}
\label{fig:slw-sensitivity}
\end{figure}

Figure~\ref{fig:slw-sensitivity} shows the sweep over \(\lambda \in \{0, 0.01, 0.1, 0.2, 0.5\}\). The baseline corresponds to standard training without the self-correction loss. Tweedie reprojection already performs strongly across settings, with valid rates in a relatively narrow range. The main effect of the self-correction loss is on DDPM. As discussed in Section~\ref{app:subsection:tweedie_ddpm}, DDPM preserves information from the current state \(x_t\) directly. This is appropriate for continuous denoising, but in discrete constraint problems it can preserve an early incorrect commitment. Self-correction training mitigates this failure mode, improving DDPM from roughly \(29\%\) validity to approximately \(85\%\) on Sudoku puzzles with 21 clues.

Overall, the method is not sensitive to the exact value of \(\lambda\). In the main experiments we use \(\lambda_{\mathrm{simple}}=0.1\), which was chosen as the initial default and lies in the stable high-performing region of the sweep.

\subsection{Loss-Formulation Variants}
\label{app:loss-variants}

We next compare alternative ways of constructing the self-correction training step. 

\begin{description}
\item[\textbf{Baseline.}] No self-correction loss. This is standard single-pass denoising training (as in \ref{app:slw-weight}).

\item[\textbf{Self-correction.}] Our headline recipe described in Algorithm~\ref{alg:self_correction}. The model first predicts \(\hat x_0\), then receives a corrupted version of this previous prediction and is trained to recover the original clean target \(x_0\).

\item[\textbf{Input perturbation.}]
A DDPM-IP-style input regularization baseline from \cite{ning2023input}. The input perturbation baseline adds extra Gaussian noise to the input:
\[
\tilde x_t = x_t + \gamma \epsilon'',
\qquad
\gamma=0.1,
\qquad
\epsilon'' \sim \mathcal{N}(0,I),
\]
and trains on
\[
\left\| f_\theta(\tilde x_t,t) - x_0 \right\|^2.
\]
This encourages robustness to local perturbations of \(x_t\), but it does not specifically train the model to correct structured errors arising from its own
previous predictions.

\item[\textbf{DDPM-step input.}]
Instead of constructing the second input by forward-noising \(\hat x_0\), we construct it using one DDPM ancestral step from \(x_{t_1}\). Concretely, after
computing
\(
\hat x_0 = f_\theta(x_{t_1}, t_1),
\)
we set
\[
x_{t_2}
=
\mu_{\mathrm{DDPM}}(x_{t_1}, \hat x_0; t_1 \!\to\! t_2)
+
\tau_{t_1,t_2}\xi,
\qquad
\xi \sim \mathcal N(0,I),
\]
where \(\mu_{\mathrm{DDPM}}\) is the ancestral DDPM mean and
\(\tau_{t_1,t_2}\) is the corresponding posterior noise scale. This makes the
self-correction input distribution closer to the test-time DDPM trajectory.  The loss itself is unchanged from the standard self-correction loss $\xhatf{\timedata}'' = f_\theta(x_{t_2}, t_2)$, $\Vert \xhatf{\timedata}'' - x_0\Vert^2$ only the input distribution is modified.

\item[\textbf{DDPM-mean target.}] Instead of supervising the second prediction directly with \(\Vert \xhatf{\timedata}' - x_0\Vert^2\), we supervise the resulting DDPM mean:
  \[
  \left\| \mu_{\mathrm{DDPM}}(x_{t_2}, \xhatf{\timedata}'; t_2 \to t_3) - \alpha(t_3) x_0 \right\|^2.
  \] This asks the model to make predictions whose downstream DDPM update lands on the clean manifold.

\item[\textbf{DDPM-step input + DDPM-mean target.}] Combines the previous two modifications: the self-correction input is generated by a DDPM step, and the loss supervises the downstream DDPM mean.

\item[\textbf{Random-\(N\) denoise.}] Performs a random number of no-gradient self-correction rollouts before the final gradient pass. With parameter \(n\), the number of rollouts is sampled from \(\{1,\ldots,n\}\). This exposes the model to deeper unrolled trajectories.

\item[\textbf{Self-conditioning without self-correction loss.}] Following \cite{chen2022analog}, the model receives its previous prediction as an additional input channel, \(f_\theta(\mathrm{cat}([x_t, \xhatf{\timedata}])_\mathrm{dim=-1}, t)\), no self-correction loss.
\end{description}

\begin{figure}[t]
\centering
\includegraphics[
    width=\textwidth,
    trim={0cm 0cm 0 1.8cm},
    clip
]{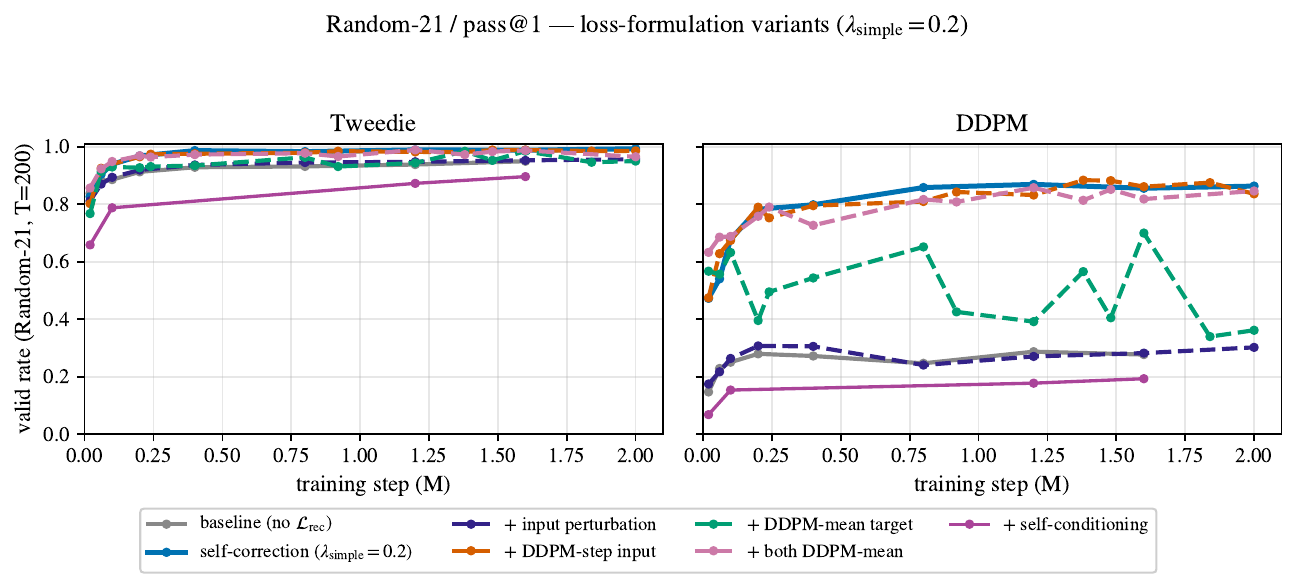}
\caption{\textbf{Loss ablations on Sudoku puzzles with 21 clues.} Self-correction and DDPM-step input give the strongest and most stable improvements, while input perturbation and self-conditioning are substantially weaker, especially for DDPM.}
\label{fig:lossvariants_ablation}
\end{figure}

Figure~\ref{fig:lossvariants_ablation} compares alternative ways of constructing the recovery signal on Sudoku puzzles with 21 clues. Input perturbation gives little improvement over the baseline, suggesting that generic Gaussian noise does not reproduce the structured errors created by closed-loop sampling. Self-conditioning is also weaker, especially for DDPM, indicating that simply providing an additional memory channel is not enough. The strongest variants are those that expose the model to its own intermediate predictions, either through the simple self-correction loss or DDPM-aware variants. This supports our main interpretation: the issue is not only robustness to noise or lack of conditioning, but a mismatch between the forward-noised training inputs and the model-induced states visited during sampling.

\paragraph{Random-symbol corruption.}
We also test whether self-correction helps only by exposing the model to invalid discrete inputs. As a simpler baseline, we randomly select either $20\%$ or $40\%$ of Sudoku cells, replacing each selected cell with a random symbol. We then train the model to recover the original valid grid. This improves DDPM from $34.5\%$ to $64.0\%$ with $20\%$ corrupted cells, showing that invalid-state exposure is useful. However, self-correction improves DDPM further to $78.2\%$, suggesting that model-induced recovery states are better matched to the reverse-sampling distribution than arbitrary symbol corruptions. Tweedie reprojection is already near saturation in this setting, so these augmentations mainly affect DDPM.

{
\paragraph{Training on multi-step DDPM trajectories.}
We additionally tested whether training on states from an actual trajectory improves over the one-step self-correction loss. Starting from a converged denoiser, we maintained a pool of 16-step DDPM trajectories and advanced each trajectory by two differentiable reverse steps per update using truncated backpropagation through time. On 21-clue Sudoku, the best rollout-trained model reaches \(0.78\) DDPM validity, compared with \(0.31\) for standard training and \(0.87\) for one-step self-correction. Thus, training on trajectory states helped, but did not outperform one-step self-correction.
}

\begin{table}[t]
\centering
\caption{\textbf{Ablation on invalid-state exposure for Sudoku conditional generation. }
We evaluate 21-given-cell Sudoku with 500 sampling steps on 1000 puzzles and 3 seeds (checkpoint at 0.2M steps). Values are valid rates in percent, mean $\pm$ std.}
\label{tab:invalid_state_ablation}
\small
\setlength{\tabcolsep}{7pt}
\renewcommand{\arraystretch}{1.12}
\begin{tabular}{lcc}
\toprule
Training variant & Tweedie reprojection & DDPM \\
\midrule
Baseline
& $99.23 \pm 0.35$
& $34.47 \pm 1.75$ \\
Random-symbol corruption, $20\%$ cells
& $99.60 \pm 0.10$
& $64.00 \pm 2.33$ \\
Random-symbol corruption, $40\%$ cells
& $99.07 \pm 0.31$
& $56.60 \pm 1.23$ \\
Self-correction
& $99.07 \pm 0.31$
& $78.20 \pm 0.98$ \\
\bottomrule
\end{tabular}
\end{table}

\subsection{Rectified-flow experiment on one-hot Sudoku}
\label{app:rectified_flow_sudoku}

In addition to evaluating the released SRM checkpoint, we train a
rectified-flow model on one-hot Sudoku with random conditioning masks.
Table~\ref{tab:rectified_flow_sudoku} compares sampling with and without
self-correction training. The same qualitative pattern appears:
Tweedie reprojection improves over Euler without retraining, while
self-correction substantially improves Euler.

\begin{table}[t]
\centering
\caption{\textbf{Sudoku validity with rectified flow.}
Entries report baseline \(\rightarrow\) self-correction for the
additional rectified-flow experiment.}
\label{tab:rectified_flow_sudoku}
\small
\begin{tabular}{lccc}
\toprule
Regime & Euler & EM decay & Tweedie \\
\midrule
21 clues
& \(0.155\to0.630\)
& \(0.427\to0.926\)
& \(0.926\to0.988\) \\
Medium
& \(0.759\to0.909\)
& \(0.864\to0.960\)
& \(0.983\to0.993\) \\
Hard
& \(0.121\to0.735\)
& \(0.417\to0.963\)
& \(0.911\to0.995\) \\
\bottomrule
\end{tabular}
\end{table}

\subsection{Consistency-model experiment}
\label{app:consistency_models}

We additionally evaluate a consistency model on conditional one-hot Sudoku-Extreme. We train the model from scratch using an adaptation of improved consistency training~\cite{song2024improved}. It uses the same VP-cosine path and \(0.82\)M-parameter Transformer architecture as our continuous baseline, with hidden dimension \(128\) and depth \(4\). Training uses the dataset's given-clue masks, a progressively increasing number of noise discretization levels, a stopped-gradient target without exponential moving averaging, and a noise-weighted pseudo-Huber loss. We train for \(2\) million optimizer steps with a learning rate of \(10^{-4}\) and warmup. This consistency model achieves approximately \(3\%\) pass@1 validity with \(5\) sampling steps and \(16\%\) with \(1{,}000\) steps. For context, at \(1{,}000\) steps, our baseline diffusion model achieves \(18\%\) with EM decay and \(26\%\) with Tweedie reprojection. 

\begin{figure}[t]
    \centering
    \includegraphics[
        width=0.65\linewidth,
        trim={0.2cm 0.2cm 0.2cm 0.84cm},
        clip
    ]{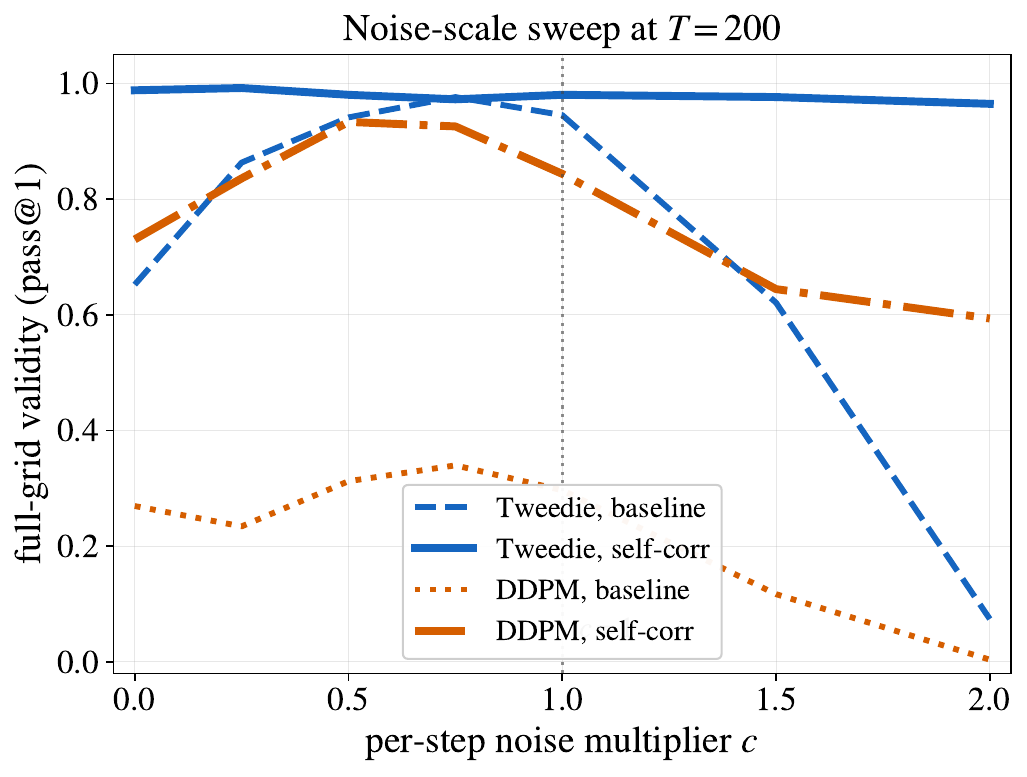}
    \caption{\textbf{Effect of DDPM sampling variance on Sudoku.} Final validity when multiplying the per-step sampling variance. Changing the variance alone does not close the baseline
    DDPM--Tweedie gap.}
    \label{fig:analysis_noise}
\end{figure}

\section{Sampling step-count ablation}

As an additional sanity check, we vary the number of reverse sampling steps \(T\) for Sudoku puzzles with 21 clues.
This tests whether the DDPM gap is simply due to insufficient discretization or too
few opportunities to correct errors.

As Table~\ref{tab:step_count_sweep} shows, increasing \(T\) drives Tweedie sampling to near-perfect validity, but baseline DDPM
remains around \(0.29\)--\(0.31\). Thus the baseline DDPM failure is not fixed by
using more reverse steps. 

\begin{table}[t]
\centering
\caption{\textbf{Sampling step-count ablation.}
Full-grid Sudoku validity for baseline and self-correction models under Tweedie
and DDPM sampling as the number of sampling steps \(T\) varies.
Values are mean \(\pm\) standard error over runs (different seeds for random masks generation).}
\label{tab:step_count_sweep}
\setlength{\tabcolsep}{5pt}
\begin{adjustbox}{max width=\textwidth}
\begin{tabular}{ccccc}
\toprule
& \multicolumn{2}{c}{Tweedie} & \multicolumn{2}{c}{DDPM} \\
\cmidrule(lr){2-3}\cmidrule(lr){4-5}
\(T\) & Baseline & Self-correction & Baseline & Self-correction \\
\midrule
25   & \(0.488 \pm 0.014\) & \(0.355 \pm 0.022\) & \(0.212 \pm 0.036\) & \(0.441 \pm 0.034\) \\
50   & \(0.737 \pm 0.018\) & \(0.887 \pm 0.027\) & \(0.272 \pm 0.013\) & \(0.745 \pm 0.030\) \\
100  & \(0.853 \pm 0.011\) & \(0.965 \pm 0.007\) & \(0.249 \pm 0.020\) & \(0.848 \pm 0.014\) \\
200  & \(0.952 \pm 0.019\) & \(0.978 \pm 0.006\) & \(0.296 \pm 0.033\) & \(0.852 \pm 0.017\) \\
500  & \(0.987 \pm 0.005\) & \(0.992 \pm 0.007\) & \(0.292 \pm 0.062\) & \(0.861 \pm 0.025\) \\
1000 & \(0.999 \pm 0.002\) & \(0.999 \pm 0.002\) & \(0.311 \pm 0.020\) & \(0.854 \pm 0.024\) \\
2000 & \(1.000 \pm 0.000\) & \(1.000 \pm 0.000\) & \(0.298 \pm 0.022\) & \(0.852 \pm 0.000\) \\
5000 & \(1.000 \pm 0.000\) & \(1.000 \pm 0.000\) & \(0.292 \pm 0.033\) & \(0.863 \pm 0.036\) \\
\bottomrule
\end{tabular}
\end{adjustbox}
\end{table}

\section{Beyond one-hot encodings}
\label{app:other_representations}

To verify that our findings are not specific to one-hot representations and argmax decoding, we vary both the representation and the decoder while using models of comparable size. We evaluate uint4 analog codes with threshold decoding (adapted from uint8 Analog Bits~\cite{chen2022analog}), fixed random embeddings with nearest-neighbour decoding, and mini MNIST-Sudoku with a learned CNN decoder (different from the dataset used by SRM \cite{wewer2025spatial}). For mini MNIST-Sudoku, we use \(8\times 8\) MNIST images, whereas SRM uses \(28\times 28\) images; our model has \(0.8\)M parameters compared with approximately \(130\)M for SRM. Each image is high-dimensional and continuous, and each cell is decoded by a learned classifier, following SRM. Table \ref{tab:representation_comparison} contains relevant within-representation comparisons.

\begin{table*}[t]
\centering
\caption{Comparison across Sudoku representations. Results report baseline $\rightarrow$ self-corrected validity (one-seed results).}
\label{tab:representation_comparison}
\begin{tabular}{l r c c c}
\hline
Experiment
& Parameters
& DDPM
& EM-decay
& Tweedie \\
\hline
One-hot Sudoku
& 824,073
& $0.31 \rightarrow 0.87$
& $0.84 \rightarrow 0.98$
& $0.94 \rightarrow 0.98$ \\

Analog-bit Sudoku
& 822,788
& $0.19 \rightarrow 0.79$
& $0.34 \rightarrow 0.87$
& $0.89 \rightarrow 0.98$ \\

Random-embedding Sudoku
& 824,073
& $0.26 \rightarrow 0.76$
& $0.43 \rightarrow 0.90$
& $0.89 \rightarrow 0.93$ \\

mini MNIST--Sudoku
& 838,208
& $0.01 \rightarrow 0.39$
& $0.01 \rightarrow 0.44$
& $0.42 \rightarrow 0.78$ \\

\hline
\end{tabular}
\end{table*}

We observe consistent improvements from DDPM to Tweedie for both the standard Sudoku checkpoint and the self-correction-trained model, while DDPM itself also benefits from self-correction training.

\section{Other metrics}
\label{app:other_metrics}

\subsection{Soft metrics.}

The main results use exact validity, which is a strict binary metric: a solution with a single violated constraint is counted as invalid. We therefore report two additional diagnostics.  

\emph{Distance to one-hot} (Table~\ref{tab:commitment}) is the per-puzzle RMSE between the final continuous output and its nearest one-hot encoding, computed over predicted cells. Lower values indicate that the final continuous state is closer to an exact discrete representation, but do not imply that the decoded configuration is valid. We report this metric for all settings except Sudoku-Extreme pass@10, which reuses the pass@1 boards. 

We also report the \emph{constraint-violation rate} (Table~\ref{tab:constraint-violations}): the fraction of predicted cells involved in a violated task constraint, using row/column/box constraints for Sudoku, row/column constraints for Latin squares, and attacking-queen constraints for \(N\)-queens. Overall, constraint-violation rate shows largely the same trends as exact validity.

\begin{table}[t]
\centering
\caption{\textbf{Distance to one-hot representations across samplers and tasks.}
Per-puzzle RMSE between the sampler output $x$ and its nearest one-hot codeword $\mathrm{onehot}(\arg\max x)$, over predicted (non-clue) cells, $\times 10^{2}$ (lower means closer to one-hot). Mean $\pm$ std over training seeds. Baselines use no self-correction loss; self-correction uses $\lambda_{\mathrm{simple}}=0.1$. EM uses a constant noise scale along the trajectory, so its final prediction contains noise in our implementation.}
\label{tab:commitment}
\setlength{\tabcolsep}{3.8pt}
\renewcommand{\arraystretch}{1.15}
\begin{adjustbox}{max width=\textwidth}
\begin{tabular}{lccccccccccc}
\toprule
\multirow{2}{*}{\makecell[c]{Sampler}}
& \multicolumn{3}{c}{Sudoku}
& \multicolumn{1}{c}{Sudoku-Extreme}
& \multicolumn{2}{c}{N-Queens}
& \multicolumn{2}{c}{Latin}
& \multicolumn{2}{c}{GC} \\
\cmidrule(lr){2-4}\cmidrule(lr){5-5}\cmidrule(lr){6-7}\cmidrule(lr){8-9}\cmidrule(lr){10-11}
& \makecell{21 clues} & \makecell{Medium} & \makecell{Hard}
& \makecell{pass@1}
& \makecell{random} & \makecell{gen}
& \makecell{random} & \makecell{gen}
& \makecell{$N=12$} & \makecell{$N=18$} \\
\midrule
\multicolumn{11}{c}{\textbf{Baseline} \textnormal{(no self-correction loss)}} \\
\addlinespace[1pt]
DDPM
& \pmv{1.5}{0.2} & \pmv{0.5}{0.1} & \pmv{1.6}{0.1} & \pmv{2.3}{1.0} & \pmv{0.5}{0.0} & \pmv{0.5}{0.0} & \pmv{1.2}{0.1} & \pmv{1.0}{0.1} & \pmv{1.1}{0.1} & \pmv{1.1}{0.1} \\

Euler
& \pmv{2.1}{0.3} & \pmv{0.6}{0.1} & \pmv{2.4}{0.1} & \pmv{3.3}{0.5} & \pmv{0.5}{0.0} & \pmv{0.5}{0.0} & \pmv{1.8}{0.1} & \pmv{1.6}{0.0} & \pmv{1.4}{0.1} & \pmv{1.3}{0.1} \\

EM
& \pmv{9.7}{0.0} & \pmv{8.9}{0.0} & \pmv{10.2}{0.1} & \pmv{6.3}{0.1} & \pmv{8.7}{0.0} & \pmv{8.7}{0.0} & \pmv{9.9}{0.1} & \pmv{10.2}{0.1} & \pmv{9.7}{0.2} & \pmv{9.5}{0.1} \\

\addlinespace[2pt]
\cdashline{1-11}
\addlinespace[1pt]

EM decay
& \pmv{0.4}{0.0} 
& \best{\pmv{0.3}{0.0}}
& \pmv{0.7}{0.1} 
& \pmv{1.6}{0.5} 
& \pmv{0.5}{0.0} 
& \best{\pmv{0.5}{0.1}}
& \pmv{0.4}{0.0} 
& \pmv{0.4}{0.0} 
& \pmv{0.6}{0.3} 
& \best{\pmv{0.6}{0.3}} \\

\makecell[l]{Tweedie\\ reprojection\\(ours)}
& \best{\pmv{0.3}{0.0}}
& \best{\pmv{0.3}{0.0}}
& \best{\pmv{0.4}{0.0}}
& \best{\pmv{1.2}{0.4}}
& \best{\pmv{0.4}{0.1}}
& \best{\pmv{0.5}{0.0}}
& \best{\pmv{0.3}{0.0}}
& \best{\pmv{0.3}{0.0}}
& \best{\pmv{0.5}{0.2}}
& \best{\pmv{0.6}{0.2}} \\

\midrule
\multicolumn{11}{c}{\textbf{Self-correction loss} \textnormal{(with $\lambda_{\mathrm{simple}}=0.1$)}} \\
\addlinespace[1pt]
DDPM
& \pmv{0.9}{0.1} & \pmv{0.4}{0.1} & \pmv{1.6}{0.3} & \pmv{3.0}{0.3} & \pmv{1.2}{0.1} & \pmv{4.3}{0.1} & \pmv{0.7}{0.1} & \pmv{1.6}{0.3} & \pmv{0.5}{0.4} & \pmv{0.6}{0.2} \\

Euler
& \pmv{1.8}{0.1} & \pmv{0.7}{0.1} & \pmv{3.2}{0.1} & \pmv{4.1}{0.5} & \pmv{1.8}{0.2} & \pmv{6.1}{0.3} & \pmv{1.4}{0.2} & \pmv{5.4}{0.3} & \pmv{0.5}{0.2} & \pmv{0.6}{0.2} \\

EM
& \pmv{9.7}{0.1} 
& \pmv{9.0}{0.0} 
& \pmv{11.4}{0.3} 
& \pmv{5.1}{0.4} 
& \pmv{9.5}{0.1} 
& \pmv{13.8}{0.4} 
& \pmv{9.3}{0.1} 
& \pmv{13.2}{1.4} 
& \pmv{1.1}{0.1} 
& \pmv{9.0}{0.1} \\

\addlinespace[2pt]
\cdashline{1-11}
\addlinespace[1pt]

EM decay
& \best{\pmv{0.3}{0.1}}
& \best{\pmv{0.3}{0.0}}
& \best{\pmv{0.4}{0.1}}
& \pmv{1.9}{0.2} 
& \pmv{0.7}{0.1} 
& \best{\pmv{2.0}{0.4}}
& \best{\pmv{0.4}{0.1}}
& \best{\pmv{0.5}{0.1}}
& \best{\pmv{0.5}{0.3}}
& \best{\pmv{0.4}{0.2}} \\

\makecell[l]{Tweedie\\ reprojection\\(ours)}
& \best{\pmv{0.3}{0.1}}
& \pmv{0.4}{0.1} 
& \pmv{0.7}{0.2} 
& \best{\pmv{1.1}{0.2}}
& \best{\pmv{0.6}{0.0}}
& \pmv{2.9}{0.2} 
& \best{\pmv{0.4}{0.1}}
& \pmv{0.8}{0.3} 
& \pmv{0.7}{0.3} 
& \pmv{0.7}{0.3} \\

\bottomrule
\end{tabular}
\end{adjustbox}
\end{table}

\begin{table}[t]
\centering
\caption{\textbf{Constraint violations across samplers and tasks.}
Percentage of denoised cells that violate a task constraint (lower is better):
row, column, or box constraints for Sudoku; row or column constraints for Latin
squares; and attacking-queen constraints for N-Queens. Entries are mean
$\pm$ standard deviation across training seeds. Baselines use no
self-correction loss; self-correction uses
$\lambda_{\mathrm{simple}}=0.1$.}
\label{tab:constraint-violations}
\setlength{\tabcolsep}{3.8pt}
\renewcommand{\arraystretch}{1.15}

\begin{adjustbox}{max width=\textwidth}
\begin{tabular}{lcccccccc}
\toprule
\multirow{2}{*}{\makecell[c]{Sampler}}
& \multicolumn{3}{c}{Sudoku}
& \multicolumn{1}{c}{Sudoku-Extreme}
& \multicolumn{2}{c}{N-Queens}
& \multicolumn{2}{c}{Latin} \\
\cmidrule(lr){2-4}
\cmidrule(lr){5-5}
\cmidrule(lr){6-7}
\cmidrule(lr){8-9}
& \makecell{21 clues}
& \makecell{Medium}
& \makecell{Hard}
& \makecell{pass@1}
& \makecell{random}
& \makecell{gen}
& \makecell{random}
& \makecell{gen}\\
\midrule

\multicolumn{9}{c}{
  \textbf{Baseline}
  \textnormal{(no self-correction loss)}
} \\
\addlinespace[1pt]

DDPM
& \pmv{4.5}{0.1}
& \pmv{1.3}{0.2}
& \pmv{3.7}{0.1}
& \pmv{8.7}{0.4}
& \pmv{13.1}{0.0}
& \pmv{24.6}{0.7}
& \pmv{2.8}{0.6}
& \pmv{1.1}{0.4} \\

Euler
& \pmv{6.0}{0.2}
& \pmv{1.8}{0.1}
& \pmv{5.3}{0.0}
& \pmv{9.6}{0.9}
& \pmv{17.6}{0.8}
& \pmv{29.1}{2.8}
& \pmv{3.9}{0.1}
& \pmv{2.1}{0.3} \\

EM
& \pmv{5.8}{0.0}
& \pmv{1.8}{0.0}
& \pmv{4.6}{0.0}
& \pmv{10.0}{0.3}
& \pmv{17.3}{1.0}
& \pmv{27.4}{2.4}
& \pmv{3.4}{0.3}
& \pmv{1.9}{0.5} \\

\addlinespace[2pt]
\cdashline{1-9}
\addlinespace[1pt]

EM decay
& \pmv{0.9}{0.0}
& \pmv{0.3}{0.0}
& \pmv{1.0}{0.1}
& \pmv{4.7}{0.1}
& \pmv{7.1}{0.3}
& \pmv{18.9}{3.0}
& \pmv{0.3}{0.1}
& \pmv{0.1}{0.1} \\

\makecell[l]{Tweedie reprojection\\(ours)}
& \best{\pmv{0.3}{0.1}}
& \best{\pmv{0.1}{0.0}}
& \best{\pmv{0.4}{0.0}}
& \best{\pmv{4.0}{0.1}}
& \best{\pmv{4.7}{1.0}}
& \best{\pmv{13.6}{0.3}}
& \best{\pmv{0.1}{0.1}}
& \best{\pmv{0.0}{0.0}} \\

\midrule

\multicolumn{9}{c}{
  \textbf{Self-correction loss}
  \textnormal{(with $\lambda_{\mathrm{simple}}=0.1$)}
} \\
\addlinespace[1pt]

DDPM
& \pmv{0.8}{0.1}
& \pmv{0.3}{0.0}
& \pmv{1.3}{0.3}
& \pmv{6.5}{0.3}
& \pmv{6.5}{0.1}
& \pmv{12.5}{0.2}
& \pmv{0.2}{0.1}
& \pmv{0.9}{0.4} \\

Euler
& \pmv{1.6}{0.2}
& \pmv{0.5}{0.0}
& \pmv{3.8}{0.3}
& \pmv{8.1}{0.7}
& \pmv{9.6}{1.0}
& \pmv{17.5}{1.9}
& \pmv{0.7}{0.3}
& \pmv{5.2}{0.6} \\

EM
& \pmv{1.4}{0.2}
& \pmv{0.5}{0.0}
& \pmv{2.5}{0.5}
& \pmv{7.7}{0.3}
& \pmv{8.7}{1.1}
& \pmv{17.1}{0.1}
& \pmv{0.4}{0.2}
& \pmv{3.2}{1.2} \\

\addlinespace[2pt]
\cdashline{1-9}
\addlinespace[1pt]

EM decay
& \best{\pmv{0.1}{0.0}}
& \best{\pmv{0.0}{0.0}}
& \best{\pmv{0.1}{0.0}}
& \pmv{3.5}{0.1}
& \pmv{3.1}{0.2}
& \best{\pmv{7.5}{0.8}}
& \best{\pmv{0.0}{0.0}}
& \best{\pmv{0.0}{0.0}} \\

\makecell[l]{Tweedie reprojection\\(ours)}
& \best{\pmv{0.1}{0.1}}
& \best{\pmv{0.0}{0.0}}
& \pmv{0.4}{0.1}
& \best{\pmv{2.7}{0.2}}
& \best{\pmv{2.5}{0.2}}
& \pmv{9.3}{0.5}
& \best{\pmv{0.0}{0.0}}
& {\pmv{0.3}{0.3}} \\

\bottomrule
\end{tabular}
\end{adjustbox}
\end{table}

\subsection{Uniqueness and coverage.}

High validity alone does not rule out repeatedly generating the same
solutions. We therefore measure \textit{valid uniqueness}: the number of distinct
valid decoded boards divided by the number of valid generated boards.
We compute this ratio separately for each training seed and then average
across seeds. Table~\ref{tab:valid_uniqueness} reports the results.
Duplicates are uncommon in these evaluations. For conditional tasks,
however, different samples have different masks and clues, so uniqueness
across the batch does not establish diversity for a fixed conditioning
input.

\begin{table}[t]
\centering
\caption{\textbf{Uniqueness among valid generated boards.}
Entries are percentages, shown as baseline
\(\rightarrow\) self-correction and averaged across training seeds.
Standard deviations are omitted. Random denotes random conditioning;
generation denotes unconditional sampling.}
\label{tab:valid_uniqueness}
\small
\begin{tabular}{lcccc}
\toprule
Sampler & Sudoku 21 clues & Sudoku Medium & Sudoku Hard
& Sudoku-Extreme \\
\midrule
DDPM
& \(100.0\to100.0\) & \(100.0\to99.9\)
& \(100.0\to100.0\) & \(100.0\to99.8\) \\
Euler
& \(100.0\to100.0\) & \(99.9\to99.9\)
& \(100.0\to100.0\) & \(100.0\to100.0\) \\
EM
& \(100.0\to100.0\) & \(100.0\to99.9\)
& \(100.0\to100.0\) & \(100.0\to99.9\) \\
EM decay
& \(100.0\to100.0\) & \(99.8\to99.9\)
& \(100.0\to100.0\) & \(99.7\to99.7\) \\
Tweedie
& \(100.0\to100.0\) & \(99.9\to99.9\)
& \(100.0\to100.0\) & \(99.9\to99.7\) \\
\midrule
Sampler & \(N\)-Queens random & \(N\)-Queens generation
& Latin random & Latin generation \\
\midrule
DDPM
& \(97.0\to96.5\) & \(100.0\to99.8\)
& \(100.0\to100.0\) & \(100.0\to100.0\) \\
Euler
& \(98.2\to97.0\) & \(100.0\to100.0\)
& \(100.0\to100.0\) & \(100.0\to100.0\) \\
EM
& \(100.0\to100.0\) & \(100.0\to100.0\)
& \(100.0\to100.0\) & \(100.0\to100.0\) \\
EM decay
& \(100.0\to100.0\) & \(100.0\to100.0\)
& \(100.0\to100.0\) & \(100.0\to100.0\) \\
Tweedie
& \(96.0\to95.7\) & \(99.8\to99.5\)
& \(100.0\to100.0\) & \(100.0\to100.0\) \\
\bottomrule
\end{tabular}
\end{table}

For \(N=14\), the complete solution space contains \(365{,}596\) boards. The training set contains \(168{,}000\) distinct boards, leaving \(197{,}596\) solutions outside the training set. For each configuration, we generate \(500{,}000\) unconditional samples and count the distinct valid boards. We report total coverage as well as coverage within and outside the training set, using the size of each set as the corresponding denominator.

Self-correction increases total coverage from \(20.0\%\) to \(44.8\%\) with EM decay and from \(24.6\%\) to \(33.8\%\) with Tweedie. Coverage is similar within and outside the training set. Thus, the validity gains are accompanied by broader coverage of both training and unseen solutions, rather than repeated generation of a small set of training boards.

\begin{table}[t]
\centering
\caption{\textbf{Coverage of the complete \(N=14\) Queens solution space.}
Each configuration uses \(500{,}000\) unconditional samples.
The two count columns report distinct valid generated boards.
Coverage percentages use denominators \(365{,}596\), \(168{,}000\),
and \(197{,}596\) for the full, training, and outside-training sets,
respectively.}
\label{tab:nqueens_coverage}
\small
\setlength{\tabcolsep}{4pt}
\begin{tabular}{llrrrrr}
\toprule
Sampler & Training
& \makecell{Total\\coverage (\%)}
& \makecell{Distinct\\in training}
& \makecell{Distinct\\outside training}
& \makecell{Training\\coverage (\%)}
& \makecell{Outside\\coverage (\%)} \\
\midrule
EM decay & Baseline
& 20.0 & 33{,}523 & 39{,}629 & 20.0 & 20.1 \\
EM decay & Self-correction
& 44.8 & 75{,}360 & 88{,}436 & 44.9 & 44.8 \\
Tweedie & Baseline
& 24.6 & 41{,}274 & 48{,}572 & 24.6 & 24.6 \\
Tweedie & Self-correction
& 33.8 & 57{,}362 & 66{,}063 & 34.1 & 33.4 \\
\bottomrule
\end{tabular}
\end{table}

\subsection{Multiple completions for the same Sudoku clues.}

We also test whether the model can generate different valid completions for a fixed conditioning input. We select \(50\) Sudoku puzzles with \(32\)--\(38\) revealed cells and enumerate all valid completions by backtracking. Each puzzle admits between \(2\) and \(12\) solutions, with \(297\) solutions in total. For each puzzle, we run Tweedie reprojection with \(1{,}000\) independent noise realizations under baseline and self-correction training. Both configurations recover \(139\) of the \(297\) solutions, or \(46.8\%\). Thus, training against individual clean targets does not restrict sampling to a single completion for each set of clues. However, neither configuration recovers all possible completions, and self-correction does not improve the aggregate coverage in this experiment.

\section{Trajectories under different samplers}
\label{app:visualization}

{
\paragraph{When do decoded predictions stop changing?}

We define the commitment time \(\tau^\star\) as the earliest recorded reverse-progress value after which the decoded prediction remains equal to its final decoded output, with \(\tau=0\) at noise and \(\tau=1\) at data. In the pooled results of Table~\ref{tab:commitment_time}, final-valid predictions stabilize earlier than final-invalid ones. Self-correction is associated with later stabilization in both groups: for final-invalid trajectories, the mean commitment time increases from \(0.92\) to \(0.98\) under DDPM and from \(0.89\) to \(0.99\) under Tweedie. This is consistent with self-correction allowing the model to revise its predictions for longer, although these revisions do not necessarily lead to a valid solution.

\begin{table}[t]
\centering
\caption{\textbf{Commitment time by final validity.}
We averaged results across nine tasks with \(256\) samples each (separately for final-valid and final-invalid outputs).}
\label{tab:commitment_time}
\small
\begin{tabular}{llcrcr}
\toprule
Sampler & Training
& \(\tau^\star_{\mathrm{valid}}\)
& \(n_{\mathrm{valid}}\)
& \(\tau^\star_{\mathrm{invalid}}\)
& \(n_{\mathrm{invalid}}\) \\
\midrule
DDPM & Baseline
& 0.72 & 1229 & 0.92 & 1075 \\
DDPM & Self-correction
& 0.88 & 1755 & 0.98 & 549 \\
\midrule
Tweedie & Baseline
& 0.69 & 1909 & 0.89 & 395 \\
Tweedie & Self-correction
& 0.83 & 2019 & 0.99 & 285 \\
\bottomrule
\end{tabular}
\end{table}

\paragraph{When does Tweedie reprojection fail?}
\label{app:tweedie_failures}
We distinguish failures according to whether a valid clean proposal appears during sampling. A failure is denoiser-related when the final Tweedie output is invalid and the decoded clean proposal was never valid at any recorded reverse step. In a diagnostic batch of \(256\) unconditional \(N\)-Queens trajectories, all failures are of this type. On Sudoku-Extreme, \(38\%\) of Tweedie failures are denoiser-related. In the remaining \(62\%\), a valid proposal appears at least once, but the trajectory still ends in an invalid sample. Poor proposal quality is therefore one limitation of Tweedie reprojection, but finding a valid proposal during sampling does not by itself guarantee a valid final output.

\paragraph{Visualizations}

We show example sampling trajectories for Latin squares (Figure~\ref{fig:app:vis_latin}), \(N\)-queens
(Figure~\ref{fig:app:vis_nqueens}), graph connectivity (Figure~\ref{fig:app:gc}), and Sudoku (Figure~\ref{fig:app:sudoku}). Within each task, all sampler rows use the same trained checkpoint; different tasks use different models.
For MNIST-Sudoku, we also show trajectories from the released SRM checkpoint
\cite{wewer2025spatial} under Euler sampling (equivalently, deterministic DDIM on the linear path;
bottom) and Tweedie reprojection (top),
visualizing both \(x_t\) (Figure~\ref{fig:app:sudoku_mnist_xt}) and the corresponding clean predictions \(\hat x_0\) (Figure~\ref{fig:app:sudoku_mnist_x0}).
In both cases, individual cells resemble recognizable MNIST digits, which we
refer to as local correctness. However, the DDIM trajectory does not produce a
globally valid Sudoku grid; for example, the central \(3\times3\) block contains
the digit \(1\) twice. In the same example, Tweedie reprojection ends in a
valid solution.
}

\begin{figure}[ht]
    \centering
    \includegraphics[
    width=\textwidth,
    trim={0cm 0cm 0 3.1cm},
    clip
    ]{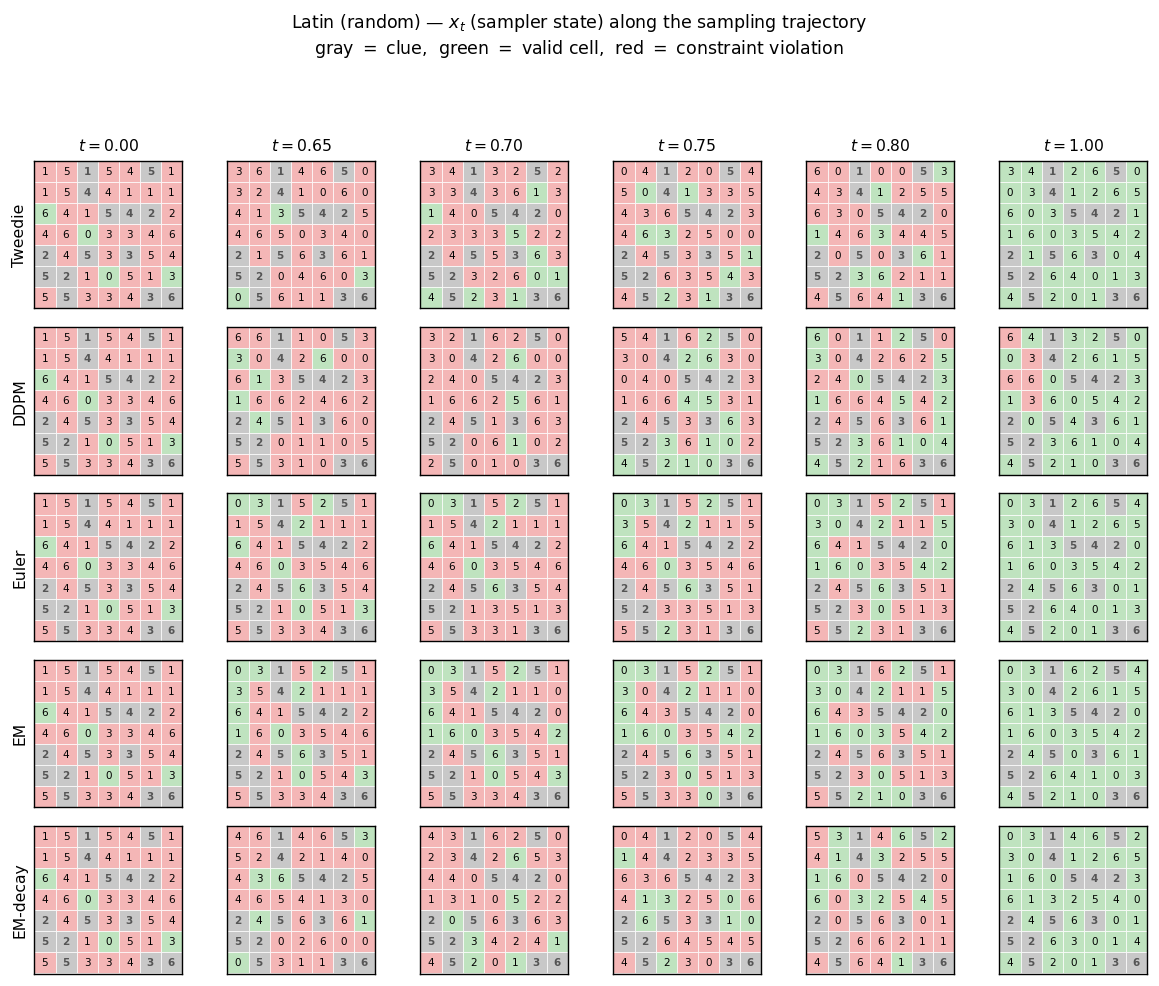}
\caption{Latin squares trajectory}
    \label{fig:app:vis_latin}
\end{figure}

\begin{figure}[ht]
    \centering
    \includegraphics[
    width=\textwidth,
    trim={0cm 0cm 0 3.1cm},
    clip
    ]{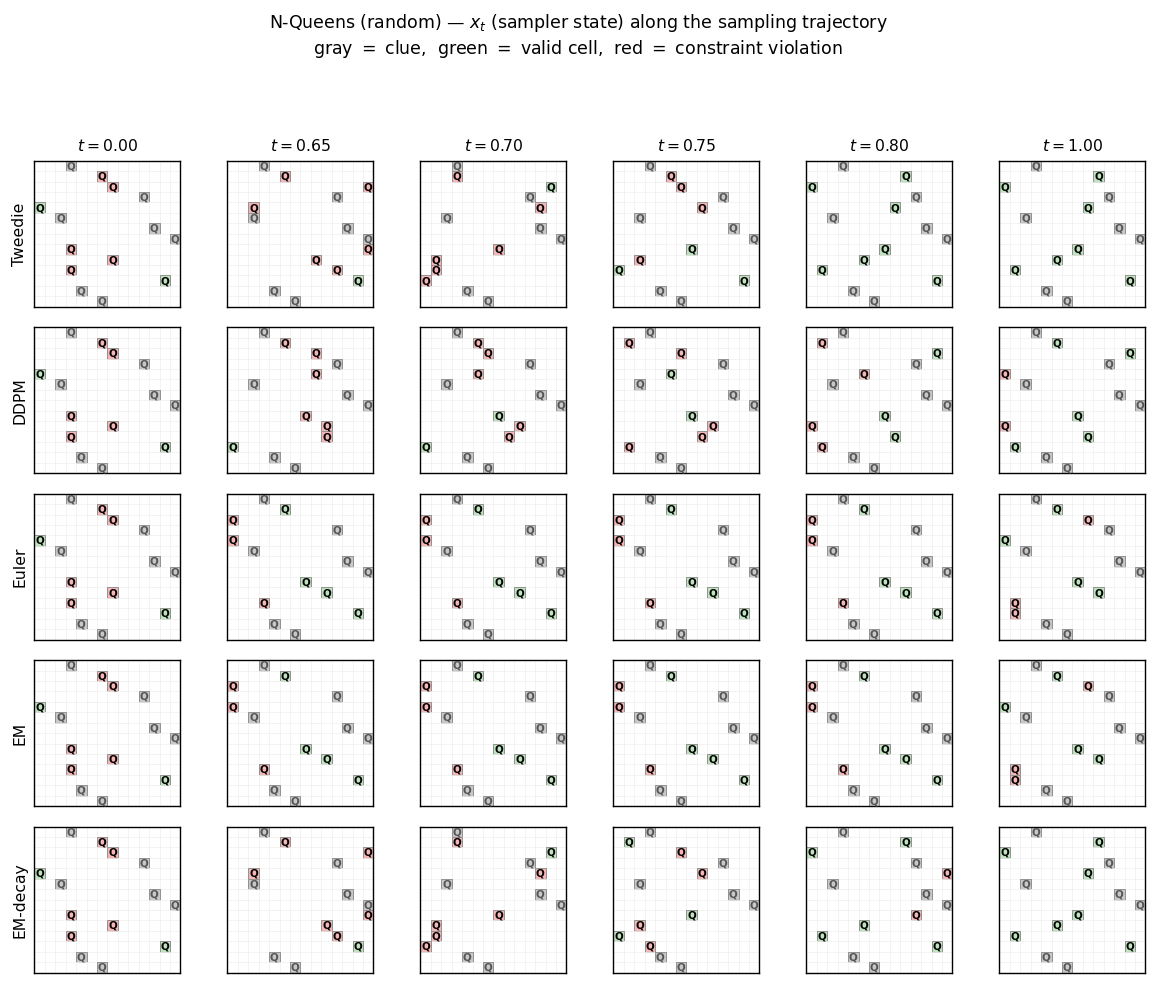}
\caption{N-Queens trajectory}
    \label{fig:app:vis_nqueens}
\end{figure}

\begin{figure}[ht]
    \centering
    \includegraphics[
    width=\textwidth,
    trim={0cm 0cm 0 2.9cm},
    clip
    ]{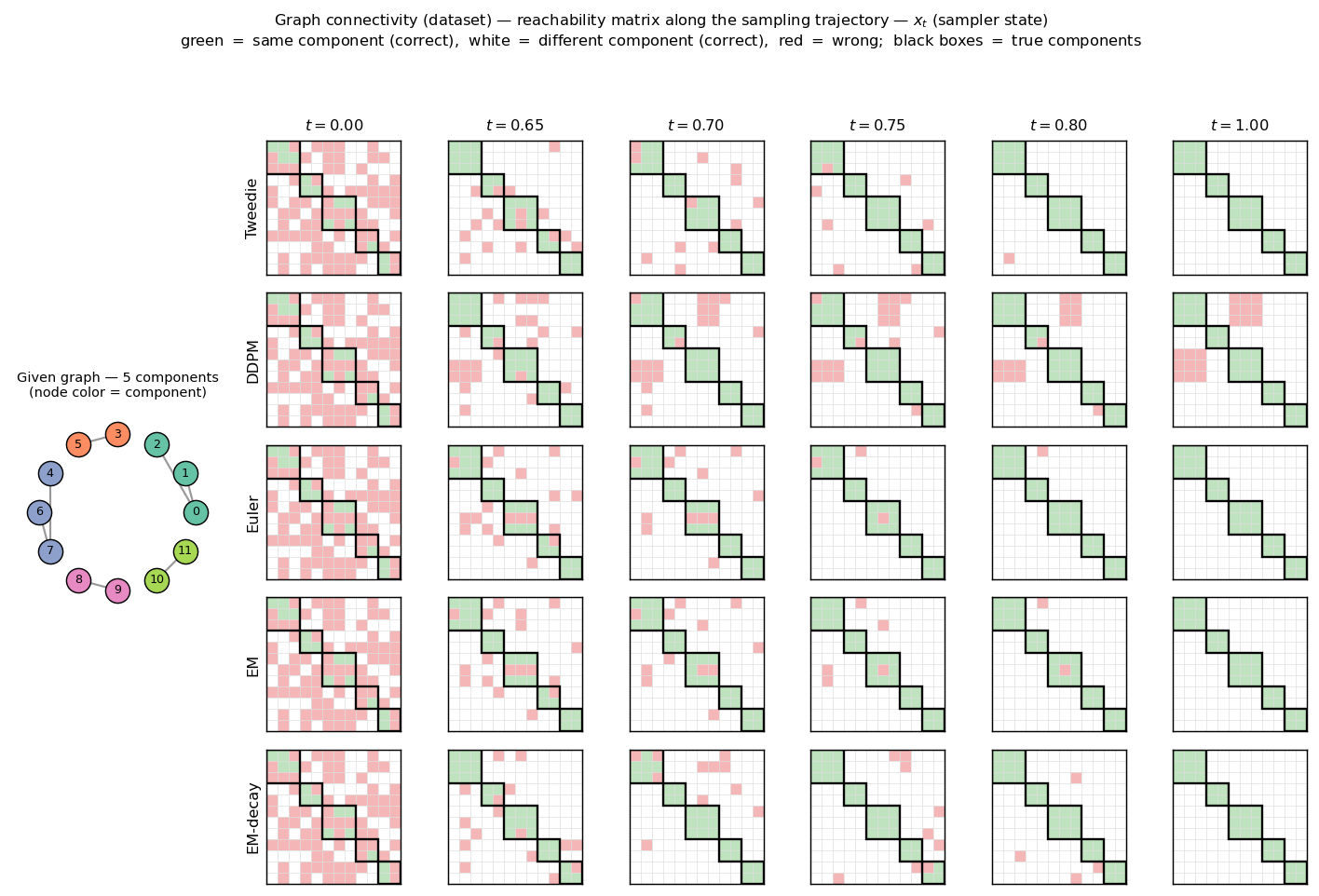}
\caption{Graph connectivity trajectory}
    \label{fig:app:gc}
\end{figure}

\begin{figure}[ht]
    \centering
    \includegraphics[
    width=\textwidth,
    trim={0cm 0cm 0 3.1cm},
    clip
    ]{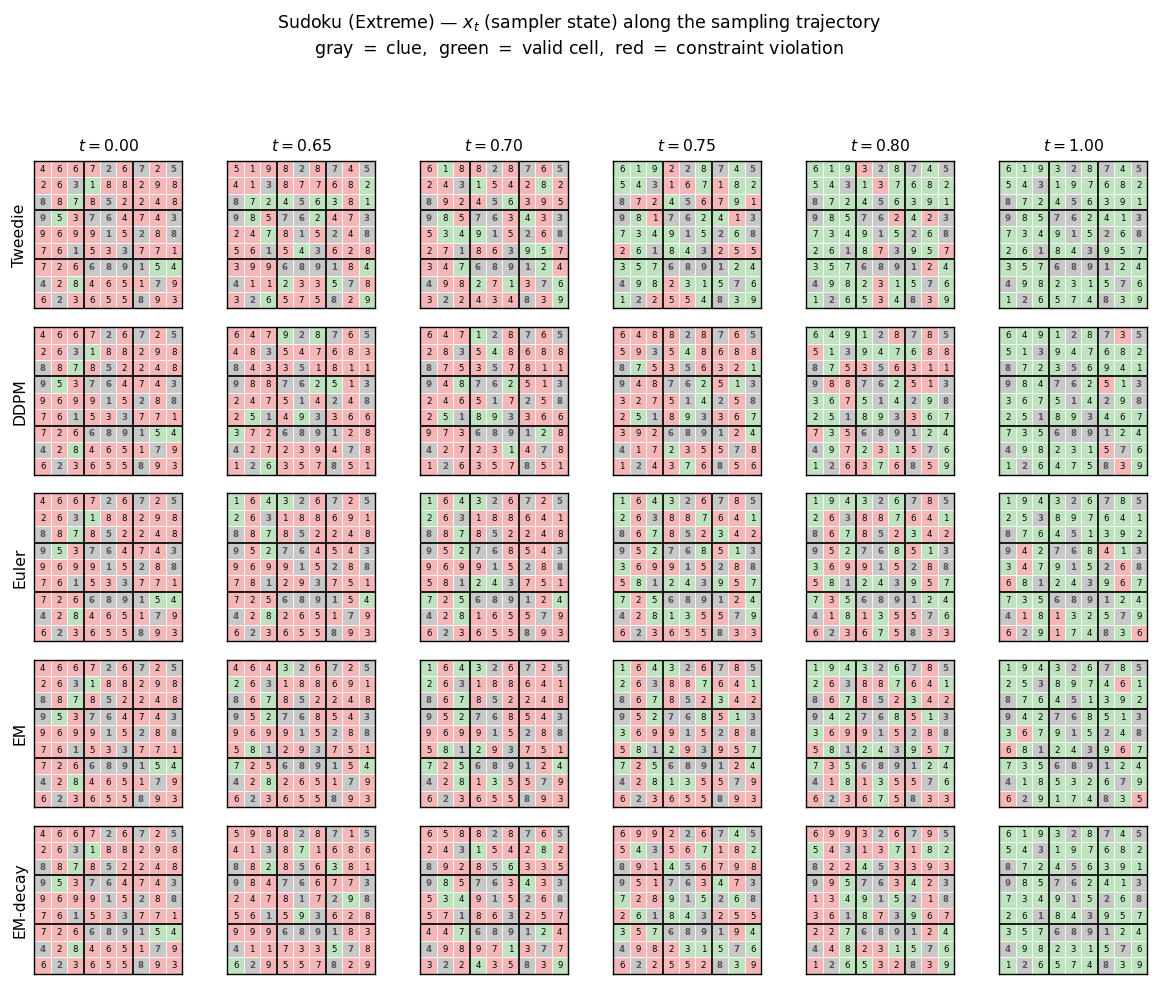}
\caption{Sudoku trajectory}
    \label{fig:app:sudoku}
\end{figure}

\begin{figure}[ht]
    \centering
    \includegraphics[
    width=\textwidth,
    trim={0cm 0cm 0 0cm},
    clip
    ]{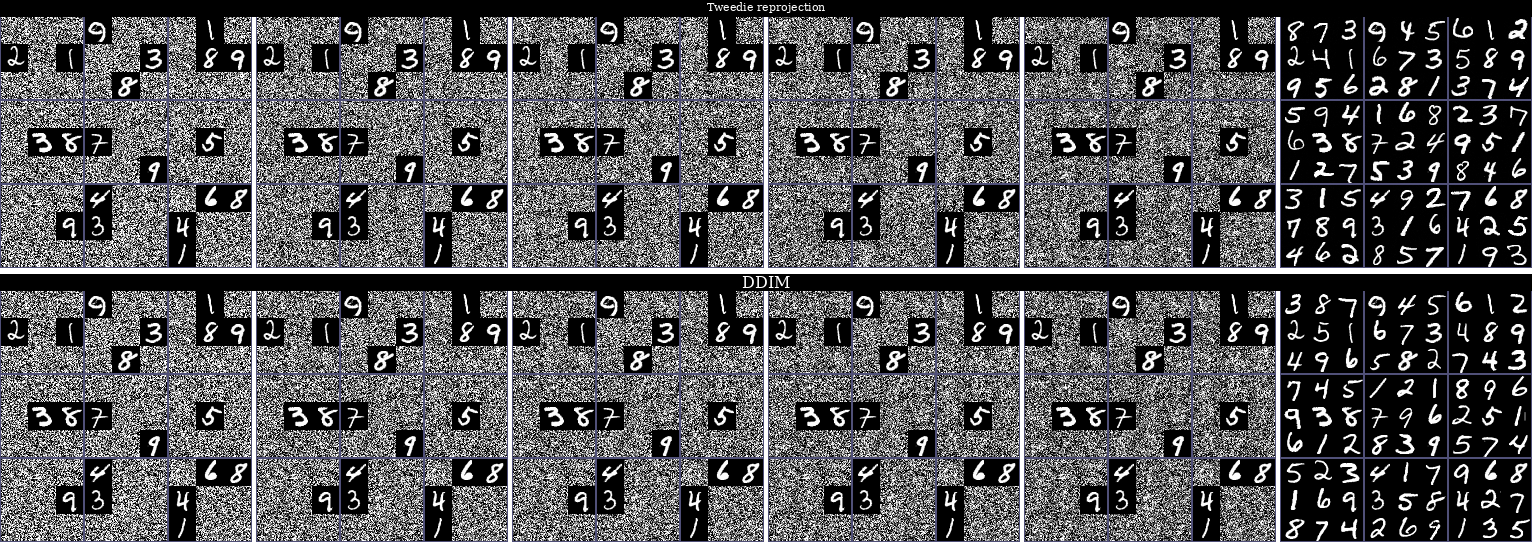}
\caption{Sudoku-MNIST trajectory ($x_t$)}
    \label{fig:app:sudoku_mnist_xt}
\end{figure}

\begin{figure}[ht]
    \centering
    \includegraphics[
    width=\textwidth,
    trim={0cm 0cm 0 0cm},
    clip
    ]{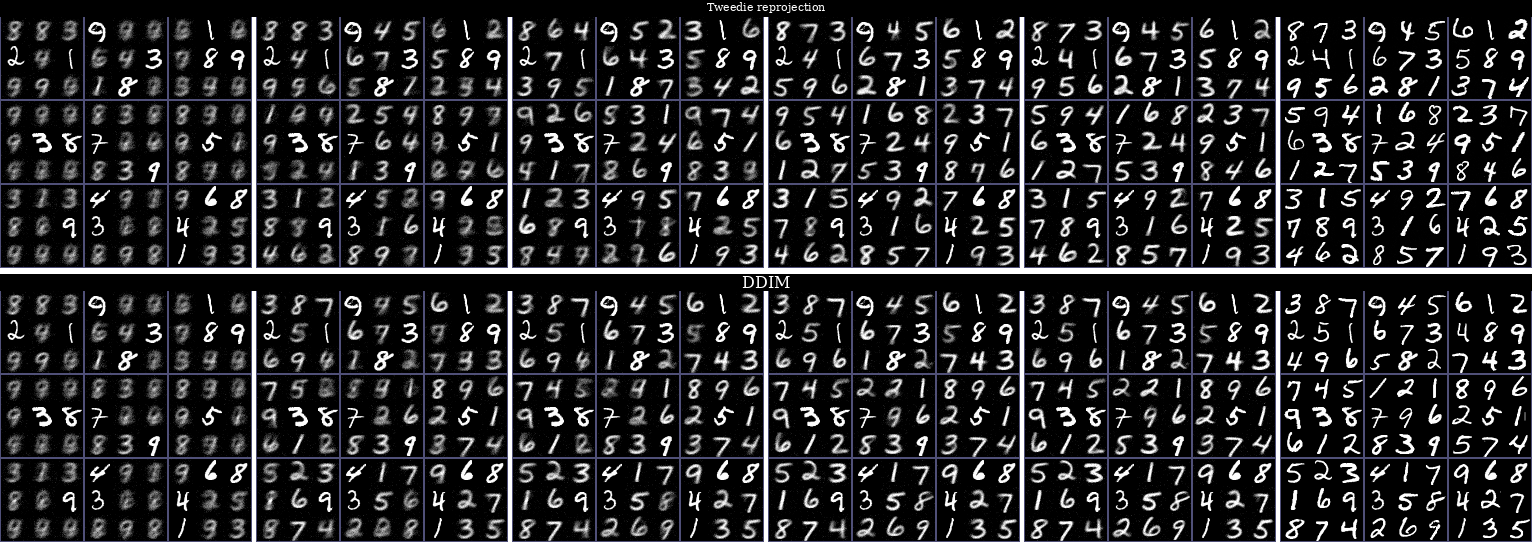}
\caption{Sudoku-MNIST trajectory ($\hat{x}_0$)}
    \label{fig:app:sudoku_mnist_x0}
\end{figure}


\end{document}